\documentclass[10pt, letterpaper, twoside]{article} 

\usepackage[utf8]{inputenc} % allow utf-8 input
\usepackage[T1]{fontenc}    % use 8-bit T1 fonts
\usepackage[hidelinks]{hyperref}       % hyperlinks
\usepackage{url}            % simple URL typesetting
\usepackage{booktabs}       % professional-quality tables
\usepackage{amsfonts}       % blackboard math symbols
\usepackage{nicefrac}       % compact symbols for 1/2, etc.
\usepackage{microtype}      % microtypography
\usepackage{xcolor}         % colors

\usepackage{amsmath}
\usepackage{mathtools}
\usepackage{algorithm}
\usepackage{algpseudocode}
\usepackage{amsthm}
\usepackage{enumitem}
\usepackage{graphicx}
\usepackage{subcaption}

\usepackage{tikz}
\usetikzlibrary{arrows.meta,fit}

\theoremstyle{plain}
\newtheorem{theorem}{Theorem}[section]
\newtheorem{lemma}{Lemma}[section]

\newtheorem{assumption}{Assumption}[section]

\theoremstyle{remark}
\newtheorem{remark}{Remark}

\newtheorem{example}{Example}[section]

\newlength{\myindent}
\newcommand{\myPhase}[1]{%
  \State \textbf{Phase #1:}%
}
\newcommand{\InPhase}[1]{%
  \State \hspace*{\myindent}{#1}%
}

\newcommand{\ForInPhase}[1]{%
  \State \hspace*{\myindent}\textbf{for #1 do}%
}

\newcommand{\InFor}[1]{%
  \State \hspace*{2\myindent}#1%
}

\newcommand{\EndForPhase}{%
  \State \hspace*{\myindent} \textbf{end For} \quad \textbf{end Phase}%
}

\newcommand{\IfInFor}[1]{%
  \State \hspace*{2\myindent}\textbf{if #1}%
}
\newcommand{\ElseIfInFor}[1]{%
  \State \hspace*{2\myindent}\textbf{else if #1}%
}
\newcommand{\InIfInFor}[1]{%
  \State \hspace*{3\myindent}#1%
}
\newcommand{\EndIfEndForEndPhase}{%
  \State \hspace*{\myindent}\textbf{end if} \quad \textbf{end for} \quad \textbf{end Phase}%
}

\newcommand{\EndPhase}{%
  \State \textbf{end Phase}%
}

\usepackage{placeins}
\usepackage{enumitem}
\setlist[itemize]{itemsep=0pt,parsep=2pt,topsep=2pt}
\setlist[enumerate]{itemsep=0pt,parsep=2pt,topsep=2pt}

\usepackage{amsmath,amsthm}
\usepackage{amssymb,mathrsfs}
\allowdisplaybreaks

\usepackage[numbers,sort&compress]{natbib}

\theoremstyle{plain}
\newtheorem{proposition}{Proposition}[section]
\newtheorem{cor}{Corollary}[section]

\numberwithin{equation}{section}
\numberwithin{remark}{section}

\theoremstyle{definition}
\newtheorem{myalgorithm}{Algorithm}

\def\A{\mathcal{A}}

\def\tA{{\tilde \A}}
\def\B{\mathcal{B}}
\def\tB{{\tilde \B}}
\def\C{\mathcal{C}}
\def\tC{{\tilde \C}}
\def\F{\mathcal{F}}

\def\tG{{\tilde \Gamma}}
\def\J{\mathcal{J}}
\def\K{\mathcal{K}}
\def\M{\mathcal{M}}
\def\E{\mathbb{E}}
\def\P{\mathcal{P}}
\def\R{\mathbb{R}}
\def\S{\mathcal{S}}
\def\tS{{\tilde \S}}
\def\T{\mathcal{T}}
\def\Z{\mathcal{Z}}
\def\cE{\mathcal{E}}
\def\stop{\text{\rm stop}}
\def\argmax{\mathop{\arg\max}}
\def\ud{\underline{\delta}}
\def\bd{\bar \delta}
\def\1{\mathbf{1}}
\def\Rec{\mathcal{R}}

\def\Sr{\S^{\mathrm{r}}}
\def\Se{\S^{\mathrm{e}}}
\def\Sb{\S^{\mathrm{b}}}
\def\tSr{\tS^{\mathrm{r}}}
\def\tSe{\tS^{\mathrm{e}}}
\def\tSb{\tS^{\mathrm{b}}}

\def\Cr{\C^{\mathrm{r}}}
\def\Ce{\C^{\mathrm{e}}}
\def\Cb{\C^{\mathrm{b}}}
\def\tCr{\tC^{\mathrm{r}}}
\def\tCe{\tC^{\mathrm{e}}}
\def\tCb{\tC^{\mathrm{b}}}

\newcommand{\Clim}{\mathop{C\text{-}\!\lim}}

\DeclareMathAlphabet{\mathbbb}{U}{bbold}{m}{n}
\newcommand{\ind}{\mathbbb{1}}

\newcommand\blfootnote[1]{%
  \begingroup
  \renewcommand\thefootnote{}\footnote{#1}%
  \addtocounter{footnote}{-1}%
  \endgroup
}

\begin{document}

\title{Average-Reward Reinforcement Learning for Multichain MDPs: A Hierarchical Decomposition Approach\textsuperscript{$\dagger$}}
%\footnote{This research was supported in part by Amii and the Natural Sciences and Engineering Research Council of Canada (NSERC) under grants RGPIN-2024-04939 and DGECR-2024-00312.}}

\markboth{\small\itshape Yu and Heidt}{\small\itshape Average-Reward RL for Multichain MDPs}

\author{
Huizhen Yu\textsuperscript{*}
\qquad
Isaiah Heidt\textsuperscript{*}
\\[0.6ex]
Department of Computing Science\\
University of Alberta, Canada
}

\date{} 
\maketitle
\vspace{-2em}

\blfootnote{\textsuperscript{$\dagger$}H.Y.\ dedicates this paper to the memory of Professor Dimitri Bertsekas, her mentor and friend.}
\blfootnote{\textsuperscript{*}Both authors contributed equally. 
Emails: \texttt{janey.hzyu@gmail.com} (H.Y.) and
\texttt{iheidt@ualberta.ca} (I.H.).}

\begin{abstract}
We study learning optimal policies in average-reward multichain Markov decision processes (MDPs), where the optimal gain may depend on the initial state and recurrence structures vary across policies, creating challenges for reinforcement learning (RL) methods. We propose an asynchronous value-iteration-based RL algorithm that requires no model knowledge beyond the MDP's transition graph and leverages Bather's decomposition to hierarchically partition the state space into communicating subsystems and transient states. This decomposition induces a recasting of the global decision problem into structured subproblems, which our algorithm exploits. We show that the algorithm converges to the optimal gain and produces gain-optimal policies after finite time. Building on this base algorithm, we develop two further algorithms: one approximately solves the multichain average optimality equations to obtain near gain-optimal policies, and another targets near bias-optimality by approximating the optimal bias function and solving an induced average-reward multichain MDP using the base algorithm. We provide almost-sure convergence guarantees for all three algorithms and empirically compare their tradeoffs, showing that the latter two also consistently improve transient performance relative to the base algorithm. To our knowledge, these are the first essentially model-free average-reward RL algorithms for general multichain MDPs without reductions to discounted problems.
\end{abstract}
 
\hspace*{1.9cm} hierarchical decomposition; asynchronous stochastic value iteration;\\ 
\hspace*{1.9cm} gain and bias optimality; stability and convergence

\section{Introduction}

Reinforcement learning (RL) methods for Markov decision processes (MDPs) under the average-reward criterion are important for sequential decision problems where optimizing long-run performance is essential. In this paper, we study learning optimal policies in finite-space multichain MDPs and propose new asynchronous RL methods that do not require model knowledge beyond the MDP's transition graph---the directed graph indicating which states can be reached from each state under each action.
 
A multichain MDP represents the most general type of MDP structure, arising when a stationary policy can induce more than one recurrent class: the process may settle into different regions of the state space, and some may be better than others. This is not only a matter of the decisions we make but a question of the fundamental structure of the problem we are solving. In a communicating or weakly communicating MDP, mistakes can always be undone: outside of transient behavior, every state remains reachable from every other, so even a poor decision can be corrected by a policy that steers the process back out. Multichain structure that is not weakly communicating breaks this guarantee: the state space can fragment into regions that once exited, cannot be revisited, or that once entered, cannot be exited, so a policy may commit the process to a suboptimal long-run regime irrevocably, with no path back. Such irreversibility is pervasive in real-world stochastic systems: an agent interacting with its environment may make decisions that permanently alter the options available to it in the future. Descending a steep slope may leave it unable to climb back or taking a risky route may cause irreversible damage to its components. In this sense, essentially any real-world system with effectively irrevocable decisions represents a multichain problem, showing the importance of algorithms specifically designed to tackle them.

The classical theory of average-reward multichain MDPs is well understood. Let $\mathcal{M} = (\mathcal{S}, \mathcal{A}, p, r)$ be a finite-state, finite-action MDP with state space $\S$, action constraint set $\A = \{\A(s)\}_{s \in \S}$, transition function $p$, and reward function $r$.  A classical approach to computing \emph{gain-optimal} (or \emph{average-reward optimal}) policies is via the \emph{average-reward optimality equations} (AOE), which characterize optimal solutions via a pair of nested equations, given below in vector notation (cf.\ \cite[Chap.\ 9]{Put94}):
\begin{equation} \label{eq-aoe}
   g = \max_{\pi \in \Pi_{\A}} \{ P_\pi g \}, \qquad   g + h = \max_{\pi \in \Pi_{\A_*}} \{ r_\pi + P_\pi h \}.
\end{equation}
Here, we solve for $(g, h) \in \R^{|\S|} \times \R^{|\S|}$; $\Pi_\A$ denotes the set of nonrandomized stationary policies satisfying the action constraint $\A$; and $\Pi_{\A_*}$ is defined analogously for an action constraint $\A_*  \subset \A$ induced by the first optimization problem. Specifically, $\A_* = \{\A_*(s)\}_{s \in \S}$, where $\A_*(s) \subset \A(s)$ consists of actions that attain the maximum in the first equation for state $s$. The AOE \eqref{eq-aoe} has infinitely many solutions, all of the form $(g^*, h)$, where $g^*$ is the optimal gain vector of the MDP. (For the solution structure of the $h$-component, see \cite{ScF78}.) For any AOE solution, any $\pi^* \in \Pi_{\A_*}$ that attains the maximum in the second equation of \eqref{eq-aoe} for all states is gain-optimal.

This nested optimization problem is difficult to solve with existing model-free RL methods. Much recent work reduces the average-reward problem to discounted problems via the vanishing discount factor approach; the resulting discounted problems can then be addressed by model-free RL or model-based approaches that leverage the classical MDP toolkit (see, e.g., \cite{GrP23,JBGZ24,ZuC24,ZuC25} and references therein). While \cite{FPL18} studies non-communicating MDPs without this reduction, it focuses on a single communicating class and is model-based. 

In contrast, rather than approximating the average-reward problem through a sequence of discounted problems, we tackle the multichain problem directly by exploiting the structure of the original problem itself. Our approach is based on Bather's unique hierarchical decomposition (UHD) of the state space~\cite{Bather3}, which we view as fundamental to solving multichain MDPs, independent of the particular algorithm built on top of it: it identifies which regions of the state space are capable of sustaining long-run behavior and are safe to explore with existing model-free methods. This decomposition partitions the state space into communicating subsystems (also introduced by Bather~\cite{Bather2}) and transient states, inducing a structured reformulation of the global decision problem into two types of subproblems (Schweitzer~\cite{Schweitzer1984}, Ohno~\cite{Ohno}): optimal decision problems within each subsystem, and a global decision problem formulated as an aggregated optimal stopping problem (OSP).

These subproblems have features that are amenable to model-free RL. The aggregated OSP is a transient total-reward problem, while each subsystem is communicating, so its associated AOE reduces to a single equation (the second equation in \eqref{eq-aoe} with $\A_* = \A$ for the subsystem).

We exploit this structure algorithmically. Our base algorithm is an asynchronous, essentially model-free RL analogue of Ohno's value-iteration algorithm for multichain MDPs. For each subsystem, we solve the average-reward problem using a stochastic relative value iteration method, RVI Q-learning, introduced in Abounadi, Bertsekas, and Borkar \cite{ABB01} for unichain MDPs and recently shown to converge in (weakly) communicating MDPs \cite{WYS24,YWS25a,YWS25b}. For the aggregated OSP, we employ classical Q-learning \cite{WatkinsQ,Tsi94}.

We show that the base algorithm converges almost surely (a.s.\@) to the optimal gain vector $g^*$ and produces gain-optimal policies after finite time (Theorems~\ref{thm:gain_identification1} and~\ref{thm:policy_attains_gstar1}).

A shortcoming of the base algorithm, inherited from Ohno's approach, is that it can exhibit poor transient performance. Moreover, although $g^*$ is obtained asymptotically, the AOE~\eqref{eq-aoe} is not fully solved: we only obtain partial solutions from the subsystem-level AOEs.

We propose two further algorithms that build on the base algorithm and address these issues. They are inspired by the seminal work of \citet{FeS81,FeS84} and \citet{Schweitzer1984} on value iteration methods for multichain MDPs. The first algorithm approximately solves the full AOE to obtain near gain-optimal policies by leveraging the sub-MDP AOE solutions computed by RVI Q-learning. The second algorithm, aimed at obtaining near bias-optimal policies, approximates the optimal bias function via discounted approximations and solves an induced average-reward multichain MDP using the base algorithm. We establish almost-sure convergence guarantees for both algorithms, together with eventual performance guarantees (Theorems~\ref{thm-1}--\ref{thm-2}, \ref{thm-b-conv}--\ref{thm-3-gen}), so that all three algorithms come with such guarantees. To our knowledge, these are the first essentially model-free average-reward RL algorithms for general multichain MDPs that do not rely on reductions to discounted problems.

Beyond these results, an important implication of this work is removing a key barrier to average-reward \emph{options} algorithms, a class of hierarchical RL algorithms with temporal abstractions: the restrictive unichain or communicating assumptions required in~\cite{WNS21b,WYS24,YWS25b} are no longer needed. More broadly, the sub-policies learned within each class of Bather's UHD are closely analogous to options~\cite{SPS99}, suggesting a principled basis for \emph{option discovery} and reuse. Moreover, since the decomposition depends only on transition structure, not rewards or transition probability values, it may also suit \emph{continual learning} settings where preferences evolve or transition probabilities vary over time but the environment's transition structure remains stable.

Another important application of our algorithms is to \emph{infinite-space MDPs}, which may themselves be obtained as discrete-time approximations of continuous-time stochastic control problems through the Markov chain approximation method \cite{KuD01}. Although our algorithms are developed for finite-space MDPs, they can be combined with quantization-based approximation methods, such as those studied in \cite{SLY18,KaY26}, to compute approximate solutions for infinite-space MDPs. Since our algorithms handle general multichain MDPs, applying them does not require strong ergodicity conditions on the original infinite-space MDP model, in contrast to most existing learning-based algorithms for the quantized MDPs.

Finally, although our algorithms assume knowledge of the MDP's transition graph, this assumption can in principle be removed by learning the transition structure through exploration. In Section~\ref{sec:conclusion}, we discuss this limitation and related directions for future work, including online discovery and maintenance of Bather's UHD.

The paper is organized as follows. Section~\ref{sec:prel} provides background on average-reward MDPs and a class of total-reward MDPs, and briefly describes the RL setting and assumptions considered in this paper. Related background material on RVI Q-learning for weakly communicating MDPs and total-reward Q-learning used by our algorithms is provided in Appendices~\ref{append-rvi-ql} and~\ref{append-ql-ssp}, respectively. Section~\ref{sec:baseline-alg} presents the base algorithm, Section~\ref{sec:alg2} the algorithm for approximate AOE solutions, and Section~\ref{sec:alg3} the algorithm aimed at obtaining near bias-optimal policies. These sections also state the main convergence results and establish the performance guarantees, with proof outlines for the convergence of the base algorithm included in Section~\ref{sec:baseline-alg}. Section~\ref{sec:main-results} presents preliminary experimental results, with additional experiments and details provided in Appendix~\ref{sebsec:app-experiments}. Detailed convergence proofs are given in Sections~\ref{sec:conv-proofs-base-alg} and~\ref{sec-conv-proofs}. Section~\ref{sec:conclusion} concludes the paper with a discussion of limitations and future work.

% Section 2: Preliminaries
\section{Preliminaries} \label{sec:prel}

We first recall some basic concepts for average-reward MDPs and introduce notation used throughout. We then briefly review a class of total-reward MDPs that we use in developing the average-reward RL algorithms proposed in this paper.

\subsection{Average-reward MDPs: basic definitions and notation}

We follow standard MDP and RL notation (see, e.g., \cite{Put94,SuB18}). For an MDP $\M=(\S,\A,p,r)$, we use $\pi$ or $\mu$ to denote a policy. We focus on stationary policies and, in most cases, nonrandomized stationary policies, since the latter class contains optimal policies for the various expected reward criteria considered in this paper \cite{Put94}. This class is denoted by $\Pi_\A := \{ \pi \mid \pi(s) \in \A(s), \forall \, s \in \S\}$. More generally, for any $\B \subset \A$ such that $\B(s) \not= \varnothing$ for all $s \in \S$, $\Pi_\B$ denotes the subset of nonrandomized stationary policies satisfying the action constraints specified by $\B$. (Here and in what follows, the set-inclusion notation $\subset$ is used to mean $\subseteq$.)

A stationary policy $\pi$ induces a Markov chain on $\S$ with transition probability matrix $P_\pi$ and expected one-stage reward vector $r_\pi$. The set of recurrent states (resp.\ transient states) of this Markov chain is denoted by $\Rec_\pi$ (resp.\ $\T_\pi$). A set $D \subset \S$ is \emph{closed} under $\pi$ if, starting from any $s \in D$, the chain remains in $D$ with probability one. 

The \emph{gain} (or average reward) of a stationary policy $\pi$ from state $s$ is 
\[
g_\pi(s):=\lim_{N\to\infty}\frac{1}{N}\E^\pi_s\left[\sum_{t=0}^{N-1}r(s_t,a_t)\right],
\]
where $\E^\pi_s$ denotes expectation under $\pi$ with $s_0=s$. The \emph{optimal gain} is $g^*(s)=\sup_\pi g_\pi(s)$. We say that $\pi$ is \emph{gain-optimal for state $s$} if $g_\pi(s)=g^*(s)$, and \emph{gain-optimal} if this holds for all $s \in \S$. 

The optimal gain function $g^*(\cdot)$ is constant on each recurrent class induced by a stationary policy, but need not be constant over $\S$ in the multichain setting. As noted earlier, $g^*$ is the unique solution to the $g$-component of the multichain AOE~\eqref{eq-aoe}. Its two nested equations take the following componentwise form: for all $s \in \S$,
\begin{align*}
   g^*(s) & = \max_{a \in \A(s)} \Big\{ \sum_{s' \in \S}  p(s' \mid s, a) g^*(s') \Big\}, \\
   g^*(s) + h(s) & = \max_{a \in \A_*(s)} \Big\{ r(s,a) + \sum_{s' \in \S}  p(s' \mid s, a) h(s') \Big\},
\end{align*}
where $\A_*(s) =  \argmax_{a \in \A(s)} \{ \sum_{s' \in \S}  p(s' \mid s, a) g^*(s') \}$.   
When $g^*$ is constant, independent of the initial state $s$, the first equation becomes vacuous, $\A_* = \A$, and the scalar $g^*$ is uniquely determined by the second equation.

A general structural condition ensuring a constant $g^*$ is that $\M$ is \emph{communicating}: for every $s, s' \in \S$, there exists a stationary policy under which $s$ is reachable from $s'$ with positive probability. More generally, a subset $C \subset \S$ is called a \emph{closed communicating class} if every pair of states in $C$ communicate in this sense and, starting from any state in $C$, the system remains within $C$ regardless of the policy; i.e., $C$ is closed under every stationary policy. On each such $C$, $g^*(\cdot)$ is constant. A communicating MDP has $\S$ as its single closed communicating class. Slightly more generally, an MDP is \emph{weakly communicating} if $\S$ consists of a single closed communicating class together with a possibly empty set of states that are transient under every stationary policy. In such an MDP, $g^*$ is also constant. This state communication structure forms the foundation for Bather's UHD for general, multichain MDPs (see Section~\ref{sec:baseline-alg}).

The AOE~\eqref{eq-aoe} admits infinitely many solutions of the form $(g^*, h)$. The solution structure for $h$ depends on the recurrence structure under stationary optimal policies \cite{ScF78}. In particular, $h$ need not be unique up to an additive constant, even in the communicating setting. A particular solution related to the bias-optimality criterion is the optimal bias function $h^*$. We introduce the bias of a policy and discuss bias optimality later (see Section~\ref{sec:alg3}), when these concepts are needed.
   
\medskip   
\noindent {\bf The RL setting and additional notation.} In RL, we deal with random state transitions and rewards generated according to $\M$ at feasible state-action pairs $(s, a) \in \A$. Throughout, the reward distribution associated with each feasible pair $(s, a)$ is assumed to have mean $r(s, a)$ and \emph{finite variance}. The transition and reward data may be obtained from a simulator or through interactions with the environment, by following one or multiple trajectories, possibly generated by multiple agents. In a multichain MDP, a single trajectory need not provide data from all relevant parts of the state space, so multiple trajectories may in general be needed. Our requirements on how such data are generated and used are formulated based on general asynchronous stochastic approximation theory; specific conditions will be stated later for the respective algorithms. 

We use $\ind\{\ldots\}$ for indicator functions, $\1_B$ for set indicators, and $\1$ for all-ones vectors of various dimensions.

\subsection{Total-reward MDPs of SSP type} \label{sec-background-ssp}

In general, solving a (undiscounted) total-reward MDP is more demanding than solving an average-reward MDP, as it amounts to solving the problem with respect to (w.r.t.\@) the bias-optimality criterion, which is stronger than gain-optimality \cite[Sec.\ 10.4]{Put94}. However, the average-reward framework need not be invoked for transient total-reward problems or for certain non-transient problems with termination under every policy or under optimal policies. We discuss a family of such problems, which has a well-established theory of optimality and solution methods, including Q-learning in the RL setting. We make use of this class of problems in both the formulation and analysis of our average-reward RL algorithms for multichain MDPs. To provide the background, we review here its model assumptions and optimality properties; for references, see \citet{BeT91} and \cite[Sec.\ 2.2]{BeT96}. Prior Q-learning results for these problems are reviewed in Appendix~\ref{append-ql-ssp}.
 
Consider an MDP $\hat \M := \big\{\hat \S \cup \{\Delta\}, \hat \A, \hat p, \hat r \big\}$ where $\Delta$ is a special termination state: it is absorbing and has zero reward. We call $\hat \S$ the \emph{effective state space} of $\hat \M$, and let $\hat \A$ denote the set of feasible state-action pairs for states in $\hat \S$. We call a policy \emph{proper} if $\Delta$ is reached with probability one from every initial state $s \in \hat \S$, and \emph{improper} otherwise. 
For any policy $\pi$, define its total-reward function by
$V_{\pi}(s) := \liminf_{N \to \infty} \E_s^\pi \left[\sum_{t=0}^N \hat r(s_t,a_t)\right]$, $s\in \hat \S$. Define the optimal value function of $\hat \M$ by $V^*(s) : = \sup_\pi V_{\pi}(s)$, $s \in \hat \S$. 

\begin{assumption}\label{cond:ssp} 
Consider $\hat \M$ and its nonrandomized stationary policies.
\begin{itemize}
\item[\rm (i)] There exists at least one proper policy. 
\item[\rm (ii)] Every improper policy $\pi$ satisfies $V_{\pi}(s) = - \infty$ for at least one state $s \in \hat \S$.
\end{itemize}
\end{assumption}

We call $\hat \M$ a \emph{total-reward MDP of SSP type} if $\hat \M$ satisfies Assumption~\ref{cond:ssp}. SSP stands for stochastic shortest path. In \cite{BeT91,BeT96}, which study total-cost MDPs, such MDPs are referred to as SSP problems.

\begin{theorem}[{\cite{BeT91}, \cite[Prop.\ 2.1]{BeT96}}] \label{thm:ssp-opt}
Let $\hat \M$ be a total-reward MDP of SSP type. Then the optimal value function $V^*$ is real-valued and is the unique real-valued function satisfying the optimality equation
\[  \textstyle{ V^*(s) = \max_{a \in \A(s)} \big\{ \hat r(s,a) + \sum_{s' \in \hat \S} \hat p(s' \mid s, a) V^*(s') \big\},} \quad s \in \hat \S.  \]
Furthermore, any stationary policy attaining the maximum in this equation for every state $s$ is proper and optimal.
\end{theorem}

The optimality equation can equivalently be expressed in terms of the optimal state-action value function of $\hat \M$ (or optimal action-value function for short), defined by 
\[ 
\textstyle{ Q^*(s,a) : = \hat r(s,a) + \sum_{s' \in \hat \S} \hat p(s' \mid s, a) V^*(s')}, \quad (s,a) \in \hat \A.
\] 
Throughout the paper, we call a nonrandonmized stationary policy $\pi$ greedy w.r.t.\ a Q-function $Q$ if $\pi(s) \in \argmax_{a} Q(s,a)$ over the feasible actions at every state $s$.
Theorem~\ref{thm:ssp-opt} then implies that any stationary policy greedy w.r.t.\ $Q^*$ is proper and optimal for $\hat \M$. Under the standard stepsize and asynchrony conditions of \cite{Tsi94}, classical Q-learning applied to an SSP-type total-reward MDP converges a.s.\ to $Q^*$ \cite{Tsi94,YuB13} (cf.\ Appendix~\ref{append-ql-ssp}).

For average-reward RL in multichain MDPs, our base algorithm, after decomposing the problem, makes use of Q-learning for a transient total-reward MDP, which is an SSP-type MDP with all policies proper. A subsequent RL algorithm we propose involves Q-learning in a more general SSP-type total-reward MDP, which may have improper policies.

% Section 3: Base algorithm
\section{Base algorithm}
\label{sec:baseline-alg}

\begin{algorithm}[t]
\caption{Method for obtaining optimal gain and gain-optimal policy (Conceptual Description)}
\label{alg:mf-multichain}
\begin{algorithmic}[1]
\myPhase{1} Apply Bather's UHD to $\mathcal M$
    \InPhase{Obtain communicating classes $\{C_{\ell,k}\}$ with restricted action sets $\mathcal A_{\ell,k}$ and transient sets $\{\mathcal T_\ell\}$}
    \InPhase{Define communicating sub-MDPs $\mathcal M_{\ell,k}=(C_{\ell,k},\mathcal A_{\ell,k},p,r)$}
\EndPhase
\For{$t = 0,1,2,\dots$} \Comment{Phases 2--4 run concurrently}
\myPhase{2} In parallel, run RVI Q-learning on each $\mathcal M_{\ell,k}$
\InPhase{Obtain estimates $r^{(t)}_{\ell,k}$ and $\pi^{(t)}_{\ell,k}$ for aggregated OSP $\bar{\mathcal{M}}^{(t)} = (\bar{\mathcal S}, \bar{\mathcal A}, \bar p, \bar r^{(t)})$}
\EndPhase
\myPhase{3} Run total-reward Q-learning on the OSP induced by $\bar{\mathcal{M}}^{(t)}$
\InPhase{Obtain iterates $\bar{Q}^{(t)}$ and policy $\bar{\pi}^{(t)}$ specifying \emph{stop} or \emph{exit} through $(s^{*},a^*)$ for each $\bar{s}_{\ell,k}$}
\EndPhase
\myPhase{4} Construct policy $\pi^{(t)}$ for original MDP $\mathcal M$
\InPhase{set $\pi^{(t)}(s)\gets\bar{\pi}^{(t)}(s)$ for all $s\in\mathcal T$}
\ForInPhase{all $(\ell,k)$ in parallel}
\IfInFor{$\bar\pi^{(t)}(\bar s_{\ell,k}) = a_{\ell,k}^{\mathrm{stop}}$ then}
\InIfInFor{set $\pi^{(t)}(s)\gets \pi^{(t)}_{\ell,k}(s)$ for all $s\in C_{\ell,k}$}
\ElseIfInFor{$\bar\pi^{(t)}(\bar s_{\ell,k}) = (s^{*},a^*)$ then}
\InIfInFor{$\pi^{\mathrm{reach}}_{\ell,k}\gets \textsc{ReachPolicy}(\mathcal{M}_{\ell,k},s^{*})$}
\InIfInFor{Set $\pi^{(t)}(s)\gets \pi^{\mathrm{reach}}_{\ell,k}(s)$ for $s\in C_{\ell,k}\setminus\{s^{*}\}$}
\InIfInFor{Set $\pi^{(t)}(s^{*})\gets a^*$}
\EndIfEndForEndPhase
\EndFor
\end{algorithmic}
\end{algorithm}

We first present a largely model-free algorithm for computing the optimal gain vector and gain-optimal policies in a multichain MDP. The method builds on Bather's UHD~\cite{Bather3}, and subsequent value-iteration approaches of Schweitzer~\cite{Schweitzer1984} and Ohno~\cite{Ohno}.

The algorithm proceeds in four phases. A summary is provided in Algorithm~\ref{alg:mf-multichain}. Phase~1 decomposes the MDP into communicating subsystems. Phase~2 estimates the optimal gain within each subsystem via RVI Q-learning. Phase~3 constructs an aggregated OSP to determine whether each class should be retained or made transient. Phase~4 lifts the aggregated decisions to a policy on the original MDP. Phases 2–4 operate concurrently.

The algorithm assumes knowledge of the \emph{transition graph}:
for each state--action pair, the set of possible successor
states, but not their transition probabilities.
This structural information is used to construct the
decomposition in Phase~1; the learning procedures in
Phases~2--4 do not
require any model knowledge.
We discuss the practical relevance of this assumption
and possible approaches to relaxing it in
Section~\ref{sec:conclusion}.

\subsection{The four phases of the algorithm}

\subsubsection{Phase 1: Bather's UHD}
\label{subsec:phase1}

Bather's UHD~\cite{Bather3,Schweitzer1984} is a purely structural decomposition, depending only on which transitions are possible, not on their underlying probabilities. As such, it can be computed from the transition graph introduced above. We apply this decomposition to partition $\mathcal{S}$ into communicating subsystems and transient states.

The decomposition proceeds hierarchically, level by level. At level $0$, we identify all closed communicating classes of $\mathcal{S}$ under the full action constraint set $\mathcal{A}$, forming $\mathcal{C}_0$, which is always non-empty~\cite{Bather3}. These classes can be identified, for example, using the
algorithm of Fox and Landi~\cite{FoxLandi}. The transient set $\mathcal T_0$ consists of the states
outside $\mathcal C_0$ from which $\mathcal C_0$ is reached
almost surely under every policy. The remaining states 
form $\mathcal{S}_1 = \mathcal{S} \setminus (\mathcal{C}_0 \cup \mathcal{T}_0)$, 
on which the procedure repeats with a restricted action set that excludes 
transitions out of $\mathcal{S}_1$, yielding $\mathcal{C}_1$, $\mathcal{T}_1$, 
and so on. The process terminates in finitely many levels 
$\ell = 0, 1, \ldots, L$, producing a partition 
\[
\mathcal{S} = \mathcal{C}_0 \cup \mathcal{T}_0 \cup \mathcal{C}_1 \cup \mathcal{T}_1 
\cup \cdots \cup \mathcal{C}_L \cup \mathcal{T}_L,
\]
where $\mathcal{C}_\ell = \{{C}_{\ell,k} : k = 1, \ldots, n(\ell)\}$ is a 
collection of disjoint communicating classes and $\mathcal{T}_\ell$ is a set of 
transient states at level $\ell$. Let $\mathcal{T} = \bigcup_\ell \mathcal{T}_\ell$ 
denote the set of all transient states.

For each class ${C}_{\ell,k}$, define the class-preserving action set
$\mathcal{A}_{\ell,k}(s) \subseteq \mathcal{A}(s), s \in {C}_{\ell,k}$,
consisting of actions whose transitions remain within ${C}_{\ell,k}$. The corresponding exit actions $\mathcal{E}_{\ell,k}(s) := \mathcal{A}(s) \setminus \mathcal{A}_{\ell,k}(s)$ are those whose transitions can reach at least one state at a strictly lower level of the decomposition. We define the exit state--action pairs by $\mathcal{E}_{\ell,k} \coloneqq \{(s,a) : s \in C_{\ell,k},\ a \in \mathcal{E}_{\ell,k}(s)\}$.

Under $\mathcal{A}_{\ell,k}$, 
each ${C}_{\ell,k}$ forms a communicating MDP, so the optimal gain within 
the class is constant across its states. Each class induces a sub-MDP 
\[
\mathcal{M}_{\ell,k} = ({C}_{\ell,k},\, \mathcal{A}_{\ell,k},\, p,\, r).
\]
The decomposition satisfies two key properties, which ensure the UHD uniquely identifies all regions capable of sustaining long-run behavior.

\begin{lemma}[Properties of Bather's UHD \cite{Bather3,Schweitzer1984}]
\label{lem:uhd} \hfill \\
{\rm (i)} Every recurrent class induced by any stationary policy is contained in a unique ${C}_{\ell,k}$.\\
{\rm (ii)} For each ${C}_{\ell,k}$, there exists a stationary policy under which ${C}_{\ell,k}$ is recurrent.
\end{lemma}
This decomposition reduces the global problem to choosing, for each 
${C}_{\ell,k}$, whether to remain within it or exit toward other regions, 
forming the basis of the aggregated decision problem in Phases 2--4.

\subsubsection{Phase 2: RVI Q-learning}

Phase~2 runs RVI Q-learning~\cite{ABB01,YWS25b} on each
communicating sub-MDP $\mathcal M_{\ell,k}$ in parallel,
using the asynchronous scheme described in
Appendix~\ref{append-rvi-ql}. For each learner, choose a stepsize sequence
$\{\alpha_{\ell,k}^{(n)}\}_{n\geq 0}$ satisfying
Assumption~\ref{cond-ss} and a reference function $f_{\ell,k}$
satisfying Assumption~\ref{cond-f}. Let $Y_{\ell,k}^{(t)}$ denote the set of class-preserving
state--action pairs selected for update at iteration $t$,
with each learner's update schedule satisfying
Assumption~\ref{cond-us}.

For each $(s,a)\in Y_{\ell,k}^{(t)}$, observe a successor and reward $(S_{t+1}^{sa},R_{t+1}^{sa})$ generated by taking action $a$ in state $s$, and set
\[
 Q_{\ell,k}^{(t+1)}(s,a)
 =Q_{\ell,k}^{(t)}(s,a)
 +\alpha_{\ell,k}^{(\nu(t,(s,a)))}
 \Bigl[
 R_{t+1}^{sa}-f_{\ell, k}(Q_{\ell,k}^{(t)})
 +\max_{a'\in\mathcal A_{\ell,k}(S_{t+1}^{sa})}
 Q_{\ell,k}^{(t)}(S_{t+1}^{sa},a')
 -Q_{\ell,k}^{(t)}(s,a)
 \Bigr].
\]
Here, $\nu(t,(s,a))$ counts the updates of $(s,a)$
prior to iteration $t$ in this learner. Components not selected for update retain their values.
This is the classwise version of~\eqref{eq: q-ite}.

The reference function $f_{\ell, k}$ supplies the gain
estimates, which satisfy
\[
 r_{\ell,k}^{(t)}
 :=f_{\ell, k}(Q_{\ell,k}^{(t)})
 \to r_{\ell,k}^*
 \qquad\text{a.s.},
\]
where $r_{\ell,k}^*$ is the optimal gain of
$\mathcal M_{\ell,k}$.
Moreover, almost surely, for all sufficiently large $t$,
every policy $\pi_{\ell,k}^{(t)}$ greedy w.r.t.
$Q_{\ell,k}^{(t)}$ is gain-optimal for $\mathcal M_{\ell,k}$.
These guarantees follow from~\cite[Thm.~3.1]{YWS25b}
and are formalized for Phase~2 in
Lemma~\ref{lem:phase2_rvi_convergence} of
Section~\ref{sec:conv-proofs-base-alg}.

The associated relative-value iterates
$v_{\ell,k}^{(t)}(s)
:=\max_{a\in\mathcal A_{\ell,k}(s)}Q_{\ell,k}^{(t)}(s,a)$
also converge a.s. to a compact subset of relative-value
vectors that, together with $r_{\ell,k}^*$, solve the
AOE for $\mathcal M_{\ell,k}$; see
Appendix~\ref{append-rvi-ql}.
Convergence to a single relative-value vector is not
required for the base algorithm.

Classes at level zero cannot be exited, so their
restricted optimal gains coincide with their optimal
gains in the original MDP.
For higher-level classes, the restricted gains describe
what can be achieved by remaining within the class.
Phase~3 incorporates the possibility of leaving these
classes to determine the optimal gains throughout
the entire MDP.

\subsubsection{Phase 3: aggregated OSP}
\label{subsec:phase3}

Phase~3 uses the sub-MDP gain estimates from Phase~2
to determine whether to remain within each class or exit it, and which actions to take at transient states. These decisions are formulated as an OSP over an
aggregated MDP, where stopping represents remaining
within a class and following an optimal policy for the restricted sub-MDP.
We use the aggregated OSP formulation implicit in
Schweitzer~\cite{Schweitzer1984} and made explicit by
Ohno~\cite{Ohno}, solving it here through model-free learning.

We first specify the limiting aggregated MDP
$\bar{\mathcal M}
=(\bar{\mathcal S},\bar{\mathcal A},\bar p,\bar r)$
using the exact sub-MDP gains $r_{\ell,k}^*$. During learning, the gains are replaced by the current
estimates $r_{\ell,k}^{(t)}$ from Phase~2, allowing the two phases to operate concurrently. The transition probabilities $\bar p$ need not be
computed or stored, since aggregated successors can
be obtained from transition observations in the
original MDP, as explained in the learning scheme below.

\paragraph{States and actions.}Let
$\mathcal C=\{C_{\ell,k}:\ell=0,\ldots,L,\ k=1,\ldots,n(\ell)\}$.
Represent each class $C_{\ell,k}$ by a single state $\bar s_{\ell,k}$, retain all transient states, and introduce an absorbing terminal state $\Delta$. The aggregated state space is
\[
 \bar{\mathcal{S}}
 =\{\bar s_{\ell,k}:C_{\ell,k}\in\mathcal C\}
   \cup \mathcal T\cup\{\Delta\}.
\]
Define the aggregation map
$\phi:\mathcal S\to\bar{\mathcal{S}}\setminus\{\Delta\}$ by
\[
 \phi(s)=
 \begin{cases}
  \bar s_{\ell,k},&s\in C_{\ell,k};\\
  s,&s\in \mathcal T.
 \end{cases}
\]
At each class state $\bar s_{\ell,k}$, the available actions consist
of a stopping action $a_{\ell,k}^{\mathrm{stop}}$ and the exit
state--action pairs in $\mathcal E_{\ell,k}$. Thus,
\[
 \bar {\mathcal{A}}(\bar s)=
 \begin{cases}
  \{a_{\ell,k}^{\mathrm{stop}}\}\cup \mathcal E_{\ell,k},
       &\bar s=\bar s_{\ell,k};\\
  \mathcal A(s),&\bar s=s\in \mathcal T;\\
  \{a_\Delta\},&\bar s=\Delta,
 \end{cases}
\]
where $a_\Delta$ is a dummy terminal action.
An exit pair $(s,a)\in \mathcal E_{\ell,k}$ specifies both the underlying state at which to act and the original action to execute. At level zero, $\mathcal E_{0,k}=\varnothing$, so stopping is the only
available action.

\paragraph{Transitions and rewards.}
For an exit action $\bar{a} = (s,a)\in \mathcal E_{\ell,k}$, define
\begin{equation}
 \bar p(\bar y\mid\bar s_{\ell,k},\bar{a})
 =
 \sum_{y\in \mathcal S:\,\phi(y)=\bar y}p(y\mid s,a),
 \qquad \bar y\in\bar{\mathcal{S}} \setminus\{\Delta\}.
 \label{eq:base-aggregated-exit}
\end{equation}
For $s\in \mathcal T$ and $a\in \mathcal A(s)$, define similarly
\begin{equation}
 \bar p(\bar y\mid s,a)
 =
 \sum_{y\in \mathcal S:\,\phi(y)=\bar y}p(y\mid s,a),
 \qquad \bar y\in\bar{\mathcal{S}}\setminus\{\Delta\}.
 \label{eq:base-aggregated-transient}
\end{equation}
For stopping and terminal actions, set
\[
 \bar p(\Delta\mid\bar s_{\ell,k},
                    a_{\ell,k}^{\mathrm{stop}})=1,
 \qquad
 \bar p(\Delta\mid\Delta,a_\Delta)=1.
\]
The rewards are
\begin{equation}
 \bar r(\bar s,\bar a)=
 \begin{cases}
  r_{\ell,k}^*,
    &\bar s=\bar s_{\ell,k},\
     \bar a=a_{\ell,k}^{\mathrm{stop}};\\
  0,&\text{otherwise}.
 \end{cases}
 \label{eq:base-aggregated-reward}
\end{equation}
The stopping reward represents the optimal gain attainable by
remaining in the corresponding class. Since all other rewards
are zero, the total reward of the aggregated process is precisely
the reward received upon stopping. Maximizing expected
total reward therefore amounts to maximizing the expected gain
of the selected class. Under this formulation~\cite{Schweitzer1984,Ohno}, the optimal OSP values coincide with the original optimal gains through
the aggregation map $\phi$
(Lemma~\ref{lem:agg-upper-bound}).

The aggregated OSP is a transient total-reward problem
of the type described in Section~\ref{sec-background-ssp}:
every stationary policy is proper
(Lemma~\ref{lem:osp_properness}).
We include proofs of these properties for completeness
in Section~\ref{sec:conv-proofs-base-alg}.

\paragraph{Learning and gain estimates.}
During learning, Phase~3 uses the aggregated OSP
$\bar{\mathcal M}^{(t)}
=(\bar{\mathcal S},\bar{\mathcal A},\bar p,\bar r^{(t)})$,
where $\bar r^{(t)}$ is obtained
from~\eqref{eq:base-aggregated-reward} by replacing
$r_{\ell,k}^*$ with $r_{\ell,k}^{(t)}$.
We apply the asynchronous total-reward Q-learning
scheme of Appendix~\ref{append-ql-ssp}, using transition
observations from the original MDP and the current
sub-MDP gain estimates.
The aggregated transition probabilities need not be
computed or stored; instead, we can use transition data generated from the original MDP.

Recall that a retained transient state $\bar s=s$
has access to its original actions $a\in\mathcal A(s)$,
whereas a class state $\bar s_{\ell,k}$ has access to exit
actions $\bar a=(s,a)\in\mathcal E_{\ell,k}$ and
a stopping action $a_{\ell,k}^{\mathrm{stop}}$.
Thus, each non-stopping aggregated state--action pair
specifies an original pair $(s,a)$ from which a
transition can be observed.

At iteration $t$, select a nonempty set $\bar Y_t$
of nonterminal aggregated state--action pairs.
For each $(\bar s,\bar a)\in\bar Y_t$, construct a
successor and reward
$(\bar S_{t+1}^{\bar s\bar a},
  \bar R_{t+1}^{\bar s\bar a})$ as follows:
\begin{itemize}
\item
\emph{Non-stopping actions.}
For the corresponding original pair $(s,a)$,
observe a successor $s'$ generated by taking
action $a$ in state $s$ of the original MDP, and set
\[
\bar S_{t+1}^{\bar s\bar a}=\phi(s'),
\qquad
\bar R_{t+1}^{\bar s\bar a}=0.
\]

\item
\emph{Stopping actions.}
For a stopping action $\bar a
= a_{\ell,k}^{\mathrm{stop}}$,
no transition observation is needed. Set
\[
\bar S_{t+1}^{\bar s\bar a}=\Delta,
\qquad
\bar R_{t+1}^{\bar s\bar a}=r_{\ell,k}^{(t)}.
\]
\end{itemize}
Using these successors and rewards, update
\[
 \bar Q^{(t+1)}(\bar s,\bar a)
 =\bar Q^{(t)}(\bar s,\bar a)
 +\bar\alpha_t(\bar s,\bar a)
 \Bigl[
 \bar R_{t+1}^{\bar s\bar a}
 +\ind\{\bar S_{t+1}^{\bar s\bar a}\ne\Delta\}
 \cdot\!\max_{\bar a'\in\bar{\mathcal A}(\bar S_{t+1}^{\bar s\bar a})}
 \bar Q^{(t)}(\bar S_{t+1}^{\bar s\bar a},\bar a')
 -\bar Q^{(t)}(\bar s,\bar a)
 \Bigr].
\]
Here $\bar\alpha_t(\bar s,\bar a)$ is the stepsize
for the selected component.
Components not selected for update retain their values,
and $\bar Q^{(t)}(\Delta,a_\Delta)=0$ for every $t$.

Under the sampling scheme above and the asynchrony and
stepsize conditions in Assumption~\ref{cond-ssp-ql},
Lemma~\ref{lem: phase3_conv} establishes
\[
 \bar Q^{(t)}\to\bar Q^*
 \qquad\text{a.s.},
\]
where $\bar Q^*$ is the optimal action-value function of the
limiting aggregated OSP.
The lemma combines properness of all stationary policies
with the Phase~2 convergence guarantees, which ensure
that the errors in the stopping rewards vanish almost surely.

Define the gain estimates on the original state space by
\begin{equation}
 g^{(t)}(s)
 :=
 \max_{\bar a\in\bar{\mathcal{A}}(\phi(s))}
       \bar Q^{(t)}(\phi(s),\bar a),
 \qquad s\in \mathcal S.
 \label{eq:base-gain-estimate}
\end{equation}
By Lemmas~\ref{lem: solve_agg} and~\ref{lem:agg-upper-bound}, the corresponding
maxima of $\bar Q^*$ equal $g^*(s)$.
Consequently, $g^{(t)}\to g^*$ a.s.

Choose an aggregated policy greedily:
\[
 \bar\pi^{(t)}(\bar s)
 \in
 \mathop{\rm arg\,max}_{\bar a\in\bar{\mathcal{A}}(\bar s)}
       \bar Q^{(t)}(\bar s,\bar a),
 \qquad \bar s\ne\Delta.
\] 
The aggregated Q-values incorporate the Phase~2
sub-MDP gain estimates $r_{\ell,k}^{(t)}$ through
the stopping rewards.
The resulting greedy policy specifies whether to
remain in each class $C_{\ell,k}$ (stop) or exit
through some $(s^*,a^*)\in\mathcal E_{\ell,k}$,
as well as which actions to take at retained
transient states.
Phase~4 lifts these aggregated decisions to a policy
on the original state space.

\subsubsection{Phase 4: deriving a policy for the original MDP}
\label{subsec:phase4}

Phase 4 lifts the aggregated policy $\bar\pi^{(t)}$ to a stationary
policy $\pi^{(t)}$ on $\mathcal M$.
Since transient states and their action sets are retained in the
aggregated MDP, their decisions can be used directly:
\[
 \pi^{(t)}(s)=\bar\pi^{(t)}(s),
 \qquad s\in\mathcal T.
\]
For a class $C_{\ell,k}$, if
$\bar\pi^{(t)}(\bar s_{\ell,k})=a_{\ell,k}^{\mathrm{stop}}$,
the aggregated policy indicates it is currently best to remain within the class, receiving reward $r_{\ell, k}^{(t)}$.
We implement this decision using the current sub-MDP policy
$\pi_{\ell,k}^{(t)}$ from Phase 2, setting
$\pi^{(t)}(s)=\pi_{\ell,k}^{(t)}(s)$ for all $s\in C_{\ell,k}$.

If instead
$\bar\pi^{(t)}(\bar s_{\ell,k})=(s^*,a^*)$,
the aggregated policy indicates it is currently best to exit the class through $(s^*, a^*)$. Implementing this decision requires first reaching $s^*$ within the class and then executing $a^*$. Let $\textsc{ReachPolicy}(\mathcal M_{\ell,k},s^*)$ return a
stationary nonrandomized policy $\pi_{\ell,k}^{\mathrm{reach}}$ that uses
only class-preserving actions and reaches $s^*$ a.s. from
every state in $C_{\ell,k}$.
We refer to this requirement as the \emph{reaching property}.
Such a policy exists because $\mathcal M_{\ell,k}$ is
communicating. Set
\[
 \pi^{(t)}(s)=
 \begin{cases}
  \pi_{\ell,k}^{\mathrm{reach}}(s),
      &s\in C_{\ell,k}\setminus\{s^*\};\\
  a^*,&s=s^*.
 \end{cases}
\]
Thus, the process follows the reaching policy until it reaches
$s^*$ and then executes $a^*$.

The base algorithm requires only the reaching property;
the following example describes two possible implementations.

\begin{example}[Implementing $\textsc{ReachPolicy}$]
\label{ex:base-reach-policy}

We construct the reaching policy by solving an SSP problem on the restricted subsystem $\mathcal M_{\ell,k}$.
Make $s^*$ absorbing and terminal, and assign reward $-1$
to every transition except those entering $s^*$, which receive reward $0$. Since $\mathcal M_{\ell,k}$ is communicating, a proper policy exists. Moreover, every improper stationary policy incurs a total reward of $-\infty$. The SSP convergence results in Appendix~\ref{append-ql-ssp} therefore apply under the stepsize and sampling conditions stated there:
total-reward Q-learning converges a.s., and its greedy policies eventually satisfy the reaching property while minimizing expected hitting time.

If stationary randomized policies are allowed in the
lifting construction, another option is to select
uniformly among the class-preserving actions at each
state. Communication ensures that this policy also satisfies the reaching property,
although its expected hitting time may be substantially larger.
We use the SSP construction in our implementation.

The transition graph available in Phase 1 can be used to verify whether a computed policy satisfies the reaching property. Thus, for a selected target, total-reward Q-learning can be continued until its greedy policy is verified to satisfy this property.

Alternatively, eventual validity can be ensured without using the transition graph for verification by maintaining a separate SSP learner for each class--target pair. Each learner retains its iterates and stepsize history, allowing learning to resume whenever its target is selected again. If every learner used infinitely often receives updates satisfying the SSP convergence conditions, the resulting policies eventually satisfy the reaching property a.s., even if the targets continue to change.
\hfill $\diamond$
\end{example}

Fix an iteration $t$ at which all reaching policies
used in the construction of $\pi^{(t)}$ satisfy the reaching property.
For every initial state $s$ and class $C_{\ell,k}$,
Lemma~\ref{lem:equiv-class} gives
\[
\mathbb P_{\pi^{(t)}}\!
\left(
\text{absorbed in }C_{\ell,k}
\mid s_0=s
\right)
=
\mathbb P_{\bar\pi^{(t)}}\!
\left(
\text{stops at }\bar s_{\ell,k}
\mid \bar s_0=\phi(s)
\right).
\]
Thus, the process under $\pi^{(t)}$ eventually remains in each $C_{\ell,k}$ with the same probability that
$\bar\pi^{(t)}$ stops at $\bar s_{\ell,k}$.
Once the Phase~2 sub-MDP policies used in stopping
classes are gain-optimal, the gain of $\pi^{(t)}$
equals the expected stopping reward of
$\bar\pi^{(t)}$ in the limiting OSP.
An optimal aggregated policy then yields a
gain-optimal policy on the original MDP.

\subsection{Convergence results of the base algorithm}
\label{subsec:base-alg-conv}

We now present the main theoretical guarantees of the base algorithm.
The first result establishes convergence of the gain estimates, while
 the second establishes eventual gain optimality of the constructed
policy. Their full proofs and supporting lemmas are given in
Section~\ref{sec:conv-proofs-base-alg}.

\begin{theorem}[Convergence of $g^{(t)} \to g^*$]\label{thm:gain_identification1}
Let $\mathcal M$ be a finite MDP and run Algorithm \ref{alg:mf-multichain}. 
Then the iterates $\bar{Q}^{(t)}$ produced in Phase~3 converge almost surely to $\bar{Q}^*$, the optimal action-value function of the limiting aggregated OSP $\bar{\mathcal{M}}$. Moreover, the limit $\bar{Q}^*$ is related to the optimal gain $g^*$ of $\mathcal{M}$ by 
\[
 g^*(s)=\max_{\bar a\in\bar{\mathcal A}(\phi(s))}
                 \bar Q^*(\phi(s),\bar a),
 \qquad s\in\mathcal S.
\]
Thus, $g^{(t)} \to g^*$ almost surely.
\end{theorem}

\begin{proof}[Proof sketch]
Phase~2 RVI Q-learning yields
$r_{\ell,k}^{(t)}\to r_{\ell,k}^*$ a.s. for every class
(Lemma~\ref{lem:phase2_rvi_convergence}).
Moreover, every stationary policy of the aggregated OSP
is proper (Lemma~\ref{lem:osp_properness}), providing
the SSP structure needed for total-reward Q-learning
convergence.
The stopping rewards used in Phase~3 differ from their
limiting values by vanishing errors.
Combining this properness property with the vanishing
reward errors, Lemma~\ref{lem: phase3_conv} establishes
$\bar Q^{(t)}\to\bar Q^*$ a.s., where $\bar Q^*$ is the
optimal action-value function of the limiting aggregated OSP.

The actionwise maxima of $\bar Q^*$ give the optimal
expected stopping rewards of this OSP
(Lemma~\ref{lem: solve_agg}), which coincide with the
optimal gains of the corresponding states in the original
MDP (Lemma~\ref{lem:agg-upper-bound}).
Since each action set is finite, taking actionwise maxima
in the converging Q-value iterates therefore gives
$g^{(t)}\to g^*$ a.s.
\end{proof}

\begin{theorem}[Finite-time attainment of gain-optimality]\label{thm:policy_attains_gstar1}
In the construction of the policy $\pi^{(t)}$ in Algorithm \ref{alg:mf-multichain} (Phase 4), let $\textsc{ReachPolicy}(\mathcal M_{\ell,k},s^*)$ be any method that produces a stationary nonrandomized policy $\pi_{\ell, k}^{reach}$ on $\mathcal M_{\ell,k}$ such that $s^*$ is reached with probability $1$ from every initial state in $C_{\ell,k}$. Then, almost surely for all sufficiently large $t$, the constructed policy $\pi^{(t)}$ is gain-optimal, i.e., $g_{\pi^{(t)}}(s)=g^*(s) \text{ for all } s\in\mathcal S.$
\end{theorem}

\begin{proof}[Proof sketch]
Phase~2 eventually produces gain-optimal policies within
every class, almost surely
(Lemma~\ref{lem:phase2_rvi_convergence}).
Meanwhile, the aggregated Q-values converge to $\bar Q^*$
(Lemma~\ref{lem: phase3_conv}).
Because there are finitely many state--action pairs,
the optimality gaps of strictly suboptimal aggregated
actions, if any exist, have a positive minimum.
Once the Q-value errors are sufficiently small, these
actions cannot be selected greedily.
Hence, $\bar\pi^{(t)}$ is eventually optimal for the
limiting aggregated OSP.

The lifting construction preserves the distribution over
eventual stopping classes
(Lemma~\ref{lem:equiv-class}).
For all sufficiently large $t$, the policy used within
each stopping class attains its optimal class gain
$r_{\ell,k}^*$, while the rewards accumulated before
settling in that class do not affect the long-run
average reward.
Thus, the gain of the lifted policy equals the optimal
expected stopping reward of the aggregated OSP.
Since this value coincides with $g^*$ on the original
state space (Lemma~\ref{lem:agg-upper-bound}),
$\pi^{(t)}$ is gain-optimal almost surely for all
sufficiently large $t$.
\end{proof}

\subsection{Limitations and motivation for the subsequent algorithms}
\label{subsec:base-alg-limitations}

The base algorithm identifies the optimal gain vector and eventually produces gain-optimal policies almost surely. However, it does not construct a full relative-value vector solving the AOE, and its policies may exhibit poor transient performance. We next discuss these limitations to motivate the two subsequent algorithms.

\paragraph{Partial relative-value solutions.}
Phase~2 produces relative-value solutions for the restricted
sub-MDPs $\mathcal M_{\ell,k}$, where only
class-preserving actions are available.
These sub-MDP solutions do not generally combine
directly into a relative-value vector solving the
full AOE with $g^*$: the global equations must also
account for gain-preserving exit actions and retained
transient states.

Thus, although the base algorithm identifies the
optimal gain vector, constructing an associated
full relative-value solution requires additional work.
Algorithm~\ref{alg:2} addresses this by identifying
sub-MDP solutions that can be extended and learning
an approximate solution of the full AOE.

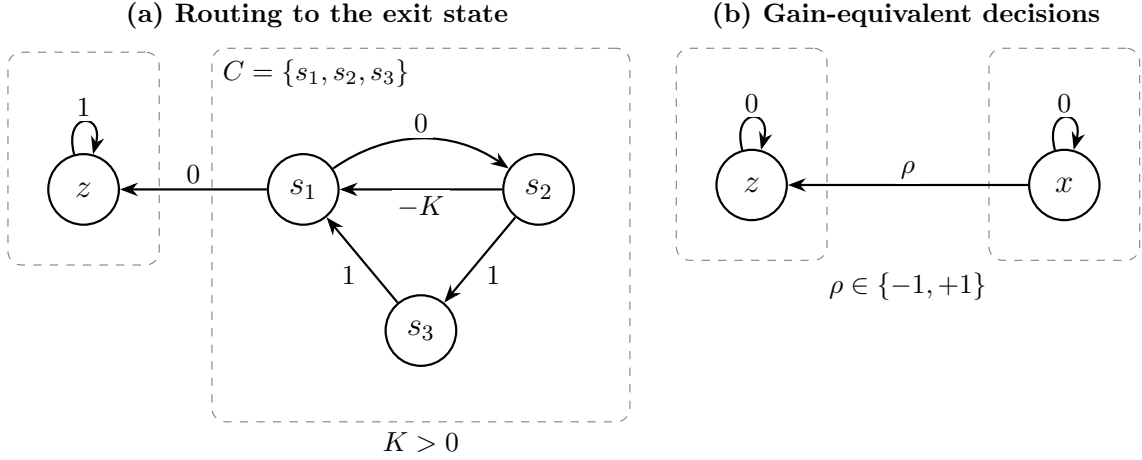
\begin{figure}[t]
\centering

\begin{minipage}[t]{0.55\textwidth}
\centering
\vspace{0pt}
\textbf{(a) Routing to the exit state}\par\medskip
\resizebox{\linewidth}{!}{%
\begin{tikzpicture}[
    state/.style={
        circle, draw, thick,
        minimum size=9mm, inner sep=0pt,
        font=\large
    },
    transition/.style={
        -{Stealth[length=2.3mm]}, thick
    },
    reward/.style={
        fill=white, inner sep=2pt,
        font=\normalsize
    },
    boundary/.style={
        draw=gray, dashed, rounded corners=5pt,
        inner sep=7mm
    }
]
\node[state] (z)  at (0,0) {$z$};
\node[state] (s1) at (2.8,0) {$s_1$};
\node[state] (s2) at (5.8,0) {$s_2$};
\node[state] (s3) at (4.3,-1.8) {$s_3$};

\coordinate (zTop) at (0,1.25);
\node[boundary, fit=(z)(zTop), inner sep=5mm] {};

\coordinate (classTop) at (4.3,1.1);
\node[boundary, fit=(s1)(s2)(s3)(classTop)] (C) {};
\node[anchor=north west, inner sep=4pt]
    at (C.north west) {$C=\{s_1,s_2,s_3\}$};

\draw[transition]
    (s1) to[bend left=35]
    node[reward, above] {$0$} (s2);

\draw[transition]
    (s2) -- node[reward, below] {$-K$} (s1);

\draw[transition]
    (s2) -- node[reward, below right] {$1$} (s3);

\draw[transition]
    (s3) -- node[reward, below left] {$1$} (s1);

\draw[transition]
    (s1) -- node[reward, above] {$0$} (z);

\draw[transition]
    (z) to[loop above, looseness=6]
    node[reward, above] {$1$} (z);

\node[anchor=north] at (C.south) {$K>0$};
\end{tikzpicture}%
}
\end{minipage}
\hfill
\begin{minipage}[t]{0.41\textwidth}
\centering
\vspace{0pt}
\textbf{(b) Gain-equivalent decisions}\par\medskip
\resizebox{\linewidth}{!}{%
\begin{tikzpicture}[
    state/.style={
        circle, draw, thick,
        minimum size=9mm, inner sep=0pt,
        font=\large
    },
    transition/.style={
        -{Stealth[length=2.3mm]}, thick
    },
    reward/.style={
        fill=white, inner sep=2pt,
        font=\normalsize
    },
    boundary/.style={
        draw=gray, dashed, rounded corners=5pt,
        inner sep=5mm
    }
]
\node[state] (z) at (0,0) {$z$};
\node[state] (x) at (4,0) {$x$};

\coordinate (zTop) at (0,1.25);
\coordinate (xTop) at (4,1.25);
\node[boundary, fit=(z)(zTop)] {};
\node[boundary, fit=(x)(xTop)] {};

\draw[transition]
    (z) to[loop above, looseness=6]
    node[reward, above] {$0$} (z);

\draw[transition]
    (x) to[loop above, looseness=6]
    node[reward, above] {$0$} (x);

\draw[transition]
    (x) -- node[reward, above] {$\rho$} (z);

\node at (2,-1.3) {$\rho\in\{-1,+1\}$};
\end{tikzpicture}%
}
\end{minipage}

\caption{Two sources of poor transient performance in the
base algorithm. Each arrow represents a distinct action
with transition probability one, and its label gives the
original one-stage reward. Dashed boundaries indicate
communicating subsystems.}
\label{fig:base-transient-limitations}
\end{figure}

\paragraph{Transient performance.}
Gain optimality alone does not distinguish policies by
the rewards accumulated before their long-run behavior
is established. The base algorithm can therefore produce
gain-optimal policies with poor transient performance.
Figure~\ref{fig:base-transient-limitations} illustrates
two distinct sources of this behavior.

\begin{example}[Routing to a selected exit state]
\label{ex:base-reaching-limitation}
Consider the MDP in
Figure~\ref{fig:base-transient-limitations}(a), with $K>0$.
The optimal gain achievable while remaining in
$C=\{s_1,s_2,s_3\}$ is $2/3$, attained by the cycle
$s_1\to s_2\to s_3\to s_1$.
Exiting to the absorbing state $z$ yields gain $1$,
so Phase~3 selects the exit action at $s_1$.

Suppose $\textsc{ReachPolicy}$ returns a policy minimizing
expected hitting time, as sought by our auxiliary SSP
construction.
Starting from $s_2$, this policy takes the direct
transition to $s_1$, incurring reward $-K$, before exiting
to $z$.
The alternative route $s_2\to s_3\to s_1$ collects
reward $2$ before taking the same exit.
Both lifted policies achieve the optimal gain of $1$, however, the longer route improves the cumulative reward obtained before reaching $z$ by
$K+2$. This transient advantage can be made arbitrarily large
by increasing $K$.
\hfill $\diamond$
\end{example}

The underlying limitation is that the reaching property
constrains eventual arrival at the selected exit state,
but not the rewards accumulated along the way.
A reaching method may account for these rewards, but
the base algorithm's gain-optimality guarantee does not
require it to do so.
Consequently, valid reaching policies can incur large
transient losses, which may accumulate when the process
passes through several exiting classes.

A separate source of poor transient performance comes from
the selection of the aggregated policy.
The OSP is designed to optimize gain, with zero
reward assigned to all non-stopping actions.
Consequently, selecting actions greedily from its
Q-values does not account for differences in transient
reward among tied gain-optimal actions.
This affects stopping and exit decisions at class
states, as well as actions at retained transient states.

\begin{example}[Gain-equivalent stopping and exit decisions]
\label{ex:base-osp-limitation}
Consider Figure~\ref{fig:base-transient-limitations}(b).
Both singleton classes have gain zero.
At $x$, the process can remain indefinitely with reward
zero or exit to $z$ with reward $\rho\in\{-1,+1\}$, receiving reward zero thereafter. Both decisions achieve gain zero.
For transient performance, however, exiting is preferable
when $\rho=1$, whereas remaining is preferable when
$\rho=-1$.

These two reward settings produce the same aggregated
OSP: both stopping rewards and the exit reward are zero.
Its action values therefore provide no basis for choosing
exit in the first setting and stop in the second.
A greedy selection can incur an avoidable penalty
or forgo a positive reward while remaining gain-optimal.
\hfill $\diamond$
\end{example}

In the second example, no internal reaching steps
are needed, so improving $\textsc{ReachPolicy}$
alone cannot resolve the issue.
These simple deterministic examples show that transient losses can arise
both from the routes used to implement exit decisions
and from the aggregated decisions themselves.
In larger MDPs, the same effects can interact along
stochastic trajectories spanning many states and
classes.
This motivates Algorithm~\ref{alg:3}, which seeks
(near) bias-optimal policies to improve transient
performance.

%% Sections on augmented algorithms
\section{Extending sub-MDP solutions to approximate AOE solutions} \label{sec:alg2} 

In this section and the next, when Algorithm~\ref{alg:mf-multichain} is used as a subroutine on an auxiliary MDP, we refer to it as the base algorithm; references to Algorithm~\ref{alg:mf-multichain} itself concern its application to the original MDP.

Given $g^*$, the AOE~\eqref{eq-aoe} reduces to
\begin{equation} \label{eq-aoe2}
   h = \max_{ \mu \in \Pi_{\A_*}} \left\{ r_\mu - g^*  + P_\mu h \right\}.
\end{equation}
Solving the AOE can provide additional information about the structural optimality properties of the MDP \cite{ScF78}, as well as a natural way to select policies with potentially improved transient behavior. We are therefore interested in finding an approximate solution to this equation using the outputs of Algorithm~\ref{alg:mf-multichain}. Several high-level considerations guide our approach. We explain them first, before presenting the algorithm.

Observe first that the exact value of the action constraint set $\A_*$ is unavailable, although the a.s.\ convergence of Algorithm~\ref{alg:mf-multichain} (Theorem~\ref{thm:gain_identification1}) enables approximations of $\A_*$ to arbitrary accuracy as time goes to infinity. Thus, rather than aiming to solve \eqref{eq-aoe2} exactly, we take a perturbation approach inspired by Federgruen and Schweitzer's work \cite{FeS84} on solving nested AOE-like equations via value iteration. Specifically, we \emph{relax both the equality and the action constraints} in~\eqref{eq-aoe2}. 
For some small constants $\epsilon > 0$ and $\delta > 0$, we aim to find $x \in \R^{|\S|}$ satisfying 
\begin{equation}
  \Big\| \, x - \max_{ \mu \in \Pi_{\tA}} \left\{ r_\mu - g^*  + P_\mu x \right\} \Big\|_\infty \leq \epsilon, \notag
\end{equation}
where the action constraint set $\tA$ satisfies
\begin{equation}
 \A_*(s) \subset \tA(s) \subset \A_\delta(s), \qquad \forall \, s \in \S, \notag
\end{equation}
with $\A_\delta(s)$ given by 
\[ \A_\delta(s) : = \Big\{ a \in \A(s) \mid  \sum_{s' \in \S} p( s' \,|\, s, a)  g^*(s') \geq g^*(s) - \delta \Big\}.\] 
Recall that $\A_*(s) = \big\{ a \in \A(s) \mid  \sum_{s' \in \S} p( s' \,|\, s, a)  g^*(s') = g^*(s)\big\}.$
Thus, $\tA$ is an outer-approximation of $\A_*$, with $\tA(s)$ consisting of a subset of actions that are $\delta$-optimal for state $s$ in the first equation of the multichain AOE~\eqref{eq-aoe}. The outer-approximation property is essential: excluding even one action in $\A_*(s)$ could eliminate all optimal policies (cf.\ \cite[Thm.~3.1(e2)]{ScF78}). 

Second, to compute an approximate AOE solution, we make use of the RVI Q-learning iterates on each sub-MDP $\M_{\ell,k}$ in Phase~2 of Algorithm~\ref{alg:mf-multichain}. Since each sub-MDP is communicating, the RVI Q-learning algorithm can be configured with appropriate stepsizes and asynchronous update schedules (\cite[Thm.~3.2]{YWS25b}; cf.\ \cite[Thm.~2.3]{YWS25a} and Appendix~\ref{append-rvi-ql}) so that it converges a.s.\ to an AOE solution for each sub-MDP, rather than merely to a compact subset of such solutions. Convergence to such a compact subset sufficed earlier for Algorithm~\ref{alg:mf-multichain} to produce optimal policies after finite time for each sub-MDP. However, it is inadequate here for computing a global approximate AOE solution with convergence guarantees. Thus, we exploit the stronger convergence guarantee of RVI Q-learning.

Third, however, not every sub-MDP AOE solution is compatible with a global AOE solution \cite{ScF78}. Two questions then arise: 
\begin{itemize}[leftmargin=0.65cm,labelwidth=!]
\item[(i)] Which sub-MDP solutions are usable? 
\item[(ii)] How can the corresponding sub-MDPs be determined from Algorithm~\ref{alg:mf-multichain}'s stochastic outputs? 
\end{itemize}
Based on Section~\ref{sec:alg2-1a} below, the sub-MDPs fall naturally into three categories depending on whether remaining, exiting, or both are gain-optimal. Solutions in the first category can be incorporated into a global AOE solution.

The actual RL algorithm must handle stochastic data while accomplishing the various tasks outlined above simultaneously. We therefore first present an idealized algorithm, assuming that the limiting sub-MDP solutions and other limiting values of Algorithm~\ref{alg:mf-multichain} are available. It serves as a blueprint for Algorithm~\ref{alg:2} in Section~\ref{sec-alg2}.

\subsection{An idealized algorithm via limiting sub-MDP solutions} \label{sec:alg2-1a}

Consider the idealized case in which the limits of the iterates from Algorithm~\ref{alg:mf-multichain} are known. By Theorem~\ref{thm:gain_identification1} and the stronger convergence property of RVI Q-learning discussed earlier, these limits include:
\begin{itemize}[leftmargin=0.5cm,labelwidth=!]
\item For each $C_{\ell,k} \in \C$, the optimal gain $r^*_{\ell,k}$ of the sub-MDP $\M_{\ell,k} = (C_{\ell,k}, \A_{\ell,k}, p, r)$, along with a solution $v^*_{\ell,k}$ to its AOE:
\begin{equation} \label{eq-subMDP-aoe-sol}
   v^*_{\ell,k}(s) = \max_{a \in \A_{\ell,k}(s)} \Big\{ r(s,a) - r^*_{\ell,k} + \sum_{s' \in C_{\ell,k}} p(s' \mid s, a) v^*_{\ell,k}(s') \Big\}, \quad \forall \, s \in C_{\ell,k}.
\end{equation}
We denote by $v_*$ the vector obtained by stacking the vectors $v^*_{\ell,k}$ over all $C_{\ell,k} \in \C$.
\item The optimal action-value function $\bar{Q}^*$ of the aggregated OSP $\bar \M$, henceforth denoted by $Q^o_*$, where the subscript o stands for the OSP.
\item The optimal gain function $g^*$ in the original MDP $\M$. 
\end{itemize}

To simplify notation, we write 
\[ \A_\ell := \bigcup_{k=1}^{n(\ell)} \A_{\ell,k}, \qquad \cE := \bigcup_{\ell,k} \, \cE_{\ell,k}.\] 
Thus, for $s \in C_{\ell,k}$, $\A_\ell(s)$ and $\cE(s)$ consist of the class-preserving and exiting actions, respectively, w.r.t.\ the sub-MDP $\M_{\ell,k}$ on $C_{\ell,k}$.

We first recall several useful relations between the aggregated OSP $\bar \M$ and the original MDP $\M$. For each $C_{\ell,k} \in \C$ and all $s \in C_{\ell,k}$,
\begin{equation} \label{eq-aosp0}
  g^*(s) \geq r_{\ell,k}^*,
\end{equation}
and
\begin{equation} \label{eq-aosp1}
    \sum_{s' \in \S} p(s' \mid s, a) \, g^*(s') = \begin{cases}
     Q^o_*(\bar s_{\ell,k}, (s,a)) & \forall \, a \in \cE(s); \\
     g^*(s) & \forall \, a \in \A_{\ell}(s).
     \end{cases}
\end{equation} 
For transient states $s \in \T$,
\begin{equation} \label{eq-aosp2}
    \sum_{s' \in \S} p(s' \mid s, a) \, g^*(s') = Q^o_*(s,a), \qquad \forall \, a \in \A(s). \qquad \quad \
\end{equation} 
By~\eqref{eq-aosp1} and~\eqref{eq-aosp2}, the action constraint set $\A_*$ can be computed by comparing $Q^o_*$ with $g^*$. Before doing so, we exploit another piece of information from the aggregated OSP: whether it is gain-optimal to remain in a sub-MDP or to exit it.

In particular, consider a three-way partition $\Cr, \Ce, \Cb$ of $\C$, according to whether it is gain-optimal to remain in a sub-MDP, to exit it, or both, respectively. More precisely, for the aggregated OSP $\bar \M$, consider the maximum optimal $Q$-value over the continuation (i.e., non-stopping) actions at each aggregated state $\bar s_{\ell,k}$:
\begin{equation} \label{eq-def-Qo_c}
  Q^o_c (\bar s_{\ell,k}): = \max_{\bar a \not= a_{\ell,k}^\stop} Q^o_*(\bar s_{\ell,k}, \bar a),
\end{equation}  
where we set $Q^o_c (\bar s_{\ell,k}) = - \infty$ if no such actions are available at $\bar s_{\ell,k}$. Recall that the continuation actions correspond to the exit state-action pairs in $\cE_{\ell,k}$ for the sub-MDP $\M_{\ell,k}$, and $Q^o_c (\bar s_{\ell,k})$ is thus the best gain achievable by exiting the sub-MDP. Comparing this value with the optimal gain $r_{\ell,k}^*$ achievable by remaining in $\M_{\ell,k}$ gives the following three-way partition of $\C$.

\medskip
\noindent {\bf Three-Way Partition of $\C$:} 
\begin{align} 
  \Cr & : = \{ C_{\ell, k} \mid r_{\ell,k}^* > Q^o_c (\bar s_{\ell,k}) \}, \label{eq-def-C1} \\
   \Ce & := \{ C_{\ell, k} \mid r_{\ell,k}^* < Q^o_c (\bar s_{\ell,k}) \}, \label{eq-def-C2} \\
   \Cb  &:= \{ C_{\ell, k} \mid r_{\ell,k}^* = Q^o_c (\bar s_{\ell,k}) \}. \label{eq-def-C3}
\end{align}

Note that $\Cr \not=\varnothing$, since it contains every level-$0$ set $C_{0,k}$, for which no exit actions are available. 
Moreover, 
\[ g^*(s) = r_{\ell,k}^* \ \ \ \text{for } s \in C_{\ell,k} \in \Cr \cup \Cb; \qquad g^*(s) > r_{\ell,k}^*\ \ \ \text{for } s \in C_{\ell,k} \in \Ce.\]
Under any stationary optimal policy in the original MDP $\M$, all $C \in \Cr$ are closed, while all $C \in \Ce$ are transient. However, a set $C \in \Cb$ may be closed under some stationary optimal policies and transient under others. 
These properties determine whether the sub-MDP AOE solutions on $C$ can be used to form a global AOE solution, as shown below.

\medskip
\noindent {\bf Compute $\A_*$:}\ \ Using the above partition of $\C$, together with~\eqref{eq-aosp1} and~\eqref{eq-aosp2}, we can express $\A_*(s)$ for each $s \in \S$ as follows:
\begin{equation} \label{eq-ideal-A*}
   \A_*(s) = \begin{cases}
     \A_{\ell}(s) & \text{if} \ s \in C_{\ell,k} \in \Cr; \\
     \A_{\ell}(s) \cup \{ a \in \cE(s) \mid Q^o_*\big(\bar s_{\ell,k}, (s, a)\big) = g^*(s) \} & \text{if} \ s \in C_{\ell,k} \in \Ce \cup \Cb; \\
     \{ a \in \A(s) \mid Q^o_*(s, a) = g^*(s) \} & \text{if} \ s \in \T.
     \end{cases}
\end{equation}
Observe that, by the definitions of $\Cr$, $\Ce$, and $\Cb$, no exiting action can belong to $\A_*(s)$ for any $s \in C_{\ell,k}$ with $C_{\ell,k} \in \Cr$, while for every $C_{\ell,k} \in \Ce \cup \Cb$,
\begin{equation} \label{eq-exit-a-CeCb}
 \A_*(s) \cap \cE(s) \not=\varnothing \quad \text{for some} \  s  \in C_{\ell,k}.
\end{equation} 
In other words, every sub-MDP in $\Ce \cup \Cb$ admits a gain-optimal exit.

\medskip
\noindent {\bf Extend the sub-MDP AOE solutions on $\Cr$ to $\S$:}\ \ Let $\Sr, \Se, \Sb$ be the unions of the sub-MDP state spaces in the corresponding categories of the partition:
\[\Sr := \bigcup_{C \in \Cr} C, \qquad \Se := \bigcup_{C \in \Ce} C, \qquad \Sb := \bigcup_{C \in \Cb} C,\]
and denote $(\Sr)^c :=  \S \setminus \Sr$. Define $v_{*,\mathrm{r}} := v_* |_{\Sr}$, i.e.,
\[ v_{*,\mathrm{r}}(s) = v^*_{\ell,k}(s)  \ \ \ \text{for} \  s \in C_{\ell, k} \in \Cr.\]
For $s \in C_{\ell, k} \in \Cr$, we have $r^*_{\ell,k} = g^*(s)$ and $\A_*(s) =  \A_{\ell}(s)$ by \eqref{eq-ideal-A*}, and any action in $\A_{\ell}(s)$ keeps the system within $C_{\ell,k}$. Therefore, by \eqref{eq-subMDP-aoe-sol}, $v_{*,1}$ satisfies the AOE~\eqref{eq-aoe2} on $\Sr$: 
\begin{equation} \label{eq-partial-aoe-sol}
v_{*,\mathrm{r}}(s) =  \max_{a \in \A_*(s)} \left\{ r(s,a) - g^*(s) + \sum_{s' \in \S} p(s' \mid s, a) v_{*,\mathrm{r}}(s') \right\}, \ \ \  \forall \, s \in \Sr.
\end{equation}
Thus, extending $v_{*,\mathrm{r}}$ to a full AOE solution amounts to solving the remaining equations in the AOE:
\begin{equation} \label{eq-1a}
   y(s)  = \max_{a \in \A_*(s)} \left\{  r(s,a) - g^*(s)  +  \sum_{s' \in \Sr} p(s' \mid s, a) v_{*,\mathrm{r}}(s')  
    + \sum_{s' \in \S \setminus \Sr} p(s' \mid s, a) y(s') \right \}, \ \ \  \forall \, s \in \S \setminus \Sr.
\end{equation}    
As Prop.~\ref{prp-idealized-alg} below shows, this system of equations have a solution $\bar y$. Hence, $(v_{*,\mathrm{r}}, \bar y)$ solves the AOE \eqref{eq-aoe2}. This completes the description of the idealized algorithm.

\begin{proposition} \label{prp-idealized-alg} 
Equation \eqref{eq-1a} admits a solution, and a unique solution if $\Cb = \varnothing$. For any solution $\bar y$, the pair $(v_{*,\mathrm{r}}, \bar y)$ solves the AOE \eqref{eq-aoe2}. 
\end{proposition}

In the remainder of this subsection, we prove Prop.~\ref{prp-idealized-alg}. While this can also be shown using non-constructive analytical arguments (see Remark~\ref{rem:alt-proof-idealized-alg}), we instead use a penalty-based proof: we penalize states in $\Sb = \bigcup_{C \in \Cb} C$, relate the penalized version of \eqref{eq-1a} to the optimality equation of an SSP-type total-reward MDP, and then let the penalty parameter tend to $0$. These arguments underlie the Q-learning step in Algorithm~\ref{alg:2}, introduced next. There, analogous penalties are applied to outer approximations of $\Sb$ to ensure the stability of Q-learning in the undiscounted total-reward setting.

For each $\epsilon > 0$, consider the equation obtained from \eqref{eq-1a} by subtracting an amount of $\epsilon$ from the rewards for states in $\Sb$: for all $s \in \S \setminus \Sr$,
\begin{equation} \label{eq-prf1}
   y(s)  = \max_{a \in \A_*(s)} \Big\{  r(s,a) - g^*(s) - \epsilon \1_{\Sb}(s) +  \sum_{s' \in \Sr} p(s' \mid s, a) v_{*,\mathrm{r}}(s')  
    + \sum_{s' \in \S \setminus \Sr} p(s' \mid s, a) y(s') \Big\}.
\end{equation}   
We show that \eqref{eq-prf1} is the optimality equation of an SSP-type total-reward MDP $\M^\epsilon$ and therefore admits a unique solution $y_\epsilon$ by Theorem~\ref{thm:ssp-opt}. This also establishes the uniqueness assertion in Prop.~\ref{prp-idealized-alg}: if $\Cb=\varnothing$, then $\Sb=\varnothing$, and \eqref{eq-prf1} reduces to \eqref{eq-1a}.

We define the MDP $\M^\epsilon :=  \big\{\hat \S \cup \{\Delta\}, \hat \A, \hat p, \hat r \big\}$ as follows. 
\begin{itemize}[leftmargin=0.6cm,labelwidth=!]
\item Let $\S \setminus \Sr$ be the effective state space $\hat \S$, and aggregate the states in $\Sr$ into a special reward-free termination state $\Delta$. 
\item Let the action constraints be $\hat \A(s) = \A_*(s), s \in \hat \S$. 
\item For each feasible state-action pair $(s,a) \in \hat \A$, define the transition probabilities $\hat p(\cdot \mid s, a)$ by 
\begin{equation} \label{eq-def-Meps-1}
\hat p(s' \mid s, a) = p(s' \mid s, a), \ \ \ s' \in \hat \S, \qquad  \hat p(\Delta \mid s, a) = \sum_{s' \in \Sr} p(s' \mid s, a),
\end{equation}
and define the expected one-stage reward $\hat r(s,a)$ by
\begin{equation} \label{eq-def-Meps-2}
\hat r(s,a) = r(s,a) - g^*(s) - \epsilon \1_{\Sb}(s) + \sum_{s' \in \Sr} p(s' \mid s, a) \, v_{*,\mathrm{r}}(s').
\end{equation}
\end{itemize}
Thus, transitions into $\Sr$ in the original MDP are replaced by transitions to the termination state $\Delta$, while the expected contribution from the corresponding values $v_{*,\mathrm{r}}(s')$ is incorporated into the one-stage reward $\hat r$. Equivalently, the last term in $\hat r(s,a)$ can be viewed as the expected terminal reward associated with a transition into $\Sr$, with $v_{*,\mathrm{r}}(s')$ serving as the terminal reward associated with $s' \in \Sr$.

\begin{lemma} \label{lem-aux-idealized-alg}
The MDP $\M^{\epsilon}$ defined above satisfies Assumption~\ref{cond:ssp} and is therefore of SSP type.
\end{lemma}

\begin{proof}
We verify the two conditions in Assumption~\ref{cond:ssp} for $\M^{\epsilon}$. The existence of a proper policy as required by Assumption~\ref{cond:ssp}(i) follows by combining the communicating property of each $C_{\ell,k} \in \Ce \cup \Cb$ with the existence of an admissible exit action for each such sub-MDP. Specifically, by \eqref{eq-exit-a-CeCb}, for each $C_{\ell,k} \in \Ce \cup \Cb$, there exists a state $s_{\ell,k} \in C_{\ell,k}$ such that $\A_*(s_{\ell,k}) \cap \cE(s_{\ell,k}) \not= \varnothing$. Since the sub-MDP $\M_{\ell,k} = (C_{\ell,k}, \A_{\ell,k}, p, r)$ is communicating, it has a policy $\mu_{\ell,k} \in \Pi_{\A_{\ell,k}}$ under which the state $s_{\ell,k}$ is reached with probability one from every initial state in $C_{\ell,k}$. Define a policy $\mu \in \Pi_{\hat \A}$ for $\M^\epsilon$ as follows:
\begin{itemize}[leftmargin=0.75cm,labelwidth=!]
\item[(i)] on each $C_{\ell,k} \in \Ce \cup \Cb$, let $\mu$ coincide with $\mu_{\ell,k}$ except at the state $s_{\ell,k}$, where $\mu(s_{\ell,k}) \in \A_*(s_{\ell,k}) \cap \cE(s_{\ell,k})$; and 
\item[(ii)] for all $s \in \T$, let $\mu(s) \in \A_*(s)$. 
\end{itemize}
(That $\mu \in \Pi_{\hat \A}$ follows from the definition of $\hat \A$ and the fact that, for all $s \in C_{\ell,k}$, $\mu_{\ell,k}(s) \in \A_{\ell,k}(s) \subset \A_*(s)$.) By the defining properties of $\cE$ and $\T$ in Bather's decomposition, under $\mu$ the system eventually exits the effective state space $\S \setminus \Sr$ with probability one, regardless of the initial state. Thus, $\mu$ is proper, verifying Assumption~\ref{cond:ssp}(i).

For Assumption~\ref{cond:ssp}(ii), we use the property of Bather's decomposition given in Lemma~\ref{lem:uhd}(i) to locate a recurrent class of an improper policy, and show that its gain in $\M^\epsilon$ on this recurrent class is strictly negative. Specifically, let $\mu \in \Pi_{\hat \A}$ be an improper policy for $\M^\epsilon$. Then $\mu$ induces at least one recurrent class $D \subset \hat \S$. For every $s \in D$, $\hat p(\Delta \mid s, \mu(s)) = 0$ and hence, by \eqref{eq-def-Meps-1}, $p(s' \mid s, \mu(s)) = 0$ for all $s' \in \Sr$. Therefore, by \eqref{eq-def-Meps-2}, 
\begin{equation} \label{eq-prf-lem-id-1}
   \hat r(s,\mu(s)) = r(s, \mu(s)) - g^*(s) - \epsilon \1_{\Sb}(s), \qquad s \in D.
\end{equation}
By Lemma~\ref{lem:uhd}(i), $D$ is contained in some $C_{\ell,k} \in \Ce \cup \Cb$. Together with \eqref{eq-prf-lem-id-1}, this implies that the gain $\hat g_\mu$ of $\mu$ in $\M^{\epsilon}$ is strictly negative on $D$. Indeed, if $D \subset C_{\ell, k} \in \Ce$, then $r^*_{\ell,k} <  g^*(s)$ for $ s \in C_{\ell,k}$ and thus 
\[ \hat g_\mu(s) \leq  r^*_{\ell,k} - g^*(s) < 0, \qquad   s \in D.\]
If $D \subset C_{\ell, k} \in \Cb$, then $r^*_{\ell,k} = g^*(s)$ for $s \in C_{\ell,k}$, and therefore
\[ \hat g_\mu(s) \leq r^*_{\ell,k} - g^*(s) - \epsilon = -\epsilon < 0, \qquad   s \in D.\] 
Consequently, by the definition of the gain of a policy, starting from any $s \in D$, the expected cumulative $n$-stage reward under $\mu$ tends to $-\infty$ as $n \to \infty$. Hence $V_\mu(s) = - \infty$ for $s \in D$, verifying Assumption~\ref{cond:ssp}(ii).
\end{proof}

\begin{proof}[Proof of Prop.~\ref{prp-idealized-alg}]
We prove the general statement in the proposition. As discussed earlier, the uniqueness assertion when $\Cb = \varnothing$ follows directly from Lemma~\ref{lem-aux-idealized-alg} and Theorem~\ref{thm:ssp-opt}. 

For each $\epsilon > 0$, $\M^\epsilon$ is a total-reward MDP of SSP type (Lemma~\ref{lem-aux-idealized-alg}), and its optimality equation \eqref{eq-prf1} therefore admits a unique solution $y_\epsilon$ (Theorem~\ref{thm:ssp-opt}). By the definition of $\M^\epsilon$, the family $\{y_\epsilon\}_{\epsilon > 0}$ is nondecreasing as $\epsilon \downarrow 0$. Moreover, the optimal value $y_\epsilon$ is attained by a proper stationary policy in $\Pi_{\hat \A}$ (Theorem~\ref{thm:ssp-opt}). Hence $\{y_\epsilon\}_{\epsilon > 0}$ is uniformly bounded above by $\|v_{*,\mathrm{r}}\|_\infty$ plus the maximum total expected reward attainable in the original MDP before hitting $\Sr$, starting from states outside $\Sr$, under any nonrandomized stationary policy that renders all states outside $\Sr$ transient. Therefore, $y_\epsilon \uparrow \bar y$ as $\epsilon \downarrow 0$ for some $\bar y \in \R^{|\S \setminus \Sr|}$. Since $y_\epsilon$ satisfies \eqref{eq-prf1}, letting $\epsilon \downarrow 0$ and using continuity shows that $\bar y$ satisfies \eqref{eq-1a}.  

Finally, as discussed earlier, combining \eqref{eq-1a} with the AOEs for the sub-MDPs $\M_{\ell,k}$ corresponding to $C_{\ell,k} \in \Cr$ shows that any solution $\bar y$ of \eqref{eq-1a}, together with $v_{*,\mathrm{r}}$, solves the system of equations in \eqref{eq-partial-aoe-sol}--\eqref{eq-1a} and therefore solves the AOE~\eqref{eq-aoe2}.
\end{proof}

\begin{remark} \label{rem:alt-proof-idealized-alg}
A non-constructive alternative for proving the existence of a solution to \eqref{eq-1a} is to regard this equation as a possible AOE for an MDP on $\S \setminus \Sr$ and show that the gain of this MDP is zero. The latter can be established using arguments similar to those in the proof of Lemma~\ref{lem-aux-idealized-alg}.
\end{remark}

\subsection{An RL algorithm} \label{sec-alg2}

We now introduce an RL algorithm for approximately solving the AOE~\eqref{eq-aoe2}. The algorithm mimics the idealized algorithm of Section~\ref{sec:alg2-1a}. Here, however, only the convergent iterates produced by Algorithm~\ref{alg:mf-multichain} are available, not their limits. Specifically, we have the following convergent sequences:
\begin{itemize}[leftmargin=0.5cm,labelwidth=!]
\item For each $C_{\ell,k} \in \C$, Algorithm~\ref{alg:mf-multichain}, with Phase 2 configured as described earlier, produces $(r_{\ell,k}^{(t)}, v_{\ell,k}^{(t)})_{t \geq 0}$ satisfying
$$r_{\ell,k}^{(t)} \to r_{\ell,k}^*, \qquad v_{\ell,k}^{(t)} \to v_{\ell,k}^*, \ \ \ a.s.,$$ 
where $r_{\ell,k}^*$ is the optimal gain and $v_{\ell,k}^*$ a sample-path-dependent AOE solution for the sub-MDP $\M_{\ell,k}$. (As noted earlier, this follows from the convergence results for RVI Q-learning in \cite[Thm.~2.3]{YWS25a} and \cite[Thm.~3.2]{YWS25b}.) Stacking these vectors $v^{(t)}_{\ell,k}$ over all $C_{\ell,k} \in \C$ to define $v_t$, we have
\[ v_t := \big(  v^{(t)}_{\ell,k} \big)_{\ell,k} \ \to \ v_* = \big(  v^{*}_{\ell,k} \big)_{\ell,k}, \ \ \ a.s., \]
where $v_*$ is the corresponding limiting vector in the idealized algorithm.
\item For the aggregated OSP $\bar \M$, by Theorem~\ref{thm:gain_identification1}, Algorithm~\ref{alg:mf-multichain} (Phase 3) produces sequences $Q^o_t:=\bar Q^{(t)}$ and $g_t:=g^{(t)}$, satisfying
\[Q^o_t  \to Q^o_*, \qquad  g_t  \to g^*, \ \ \ a.s.\] 
\item For each $C_{\ell,k} \in \C$ and $t \geq 0$, consider the aggregated state $\bar s_{\ell,k}$ in the OSP and its estimated maximum $Q$-value over the continuation actions,
\[ Q^o_{t,c} (\bar s_{\ell,k}): = \max_{\bar a \not= a_{\ell\!,k}^\stop} Q^o_t(\bar s_{\ell,k}, \bar a),\]
with $Q^o_{t,c} (\bar s_{\ell,k}):= - \infty$ if stopping is the only action available. From the convergence of $\{Q^o_t\}$ it follows that
\[ Q^o_{t,c} \to Q^o_{c}, \ \ \ a.s.,\]
where the limit $Q^o_{c}$ is as in \eqref{eq-def-Qo_c} and determines the three-way partition of $\C$ in the idealized algorithm.
\end{itemize}
We use these sequences to approximate the steps of the idealized algorithm. 

The resulting RL algorithm is given in Algorithm~\ref{alg:2}. In a nutshell, it uses the outputs of Algorithm~\ref{alg:mf-multichain} to concurrently approximate the three main steps of the idealized algorithm:
\begin{itemize}
\item[(i)] Construct action constraint sets $\tA_t$ that \emph{eventually} outer-approximate $\A_*$.
\item[(ii)] Approximate the three-way partition of $\C$ by $\tCr_t, \tCe_t, \tCb_t$. In particular, \emph{eventually}, $\tCr_t \subset \Cr$ and $\tCe_t \subset \Ce$, while $\tCb_t \supset \Cb$, with $\tCb_t$ collecting sub-MDPs for which both remaining and exiting are nearly gain-optimal.
\item[(iii)] \emph{Eventually} construct an approximate global AOE solution that extends the sub-MDP solutions on $\tCr_t$. In particular, total-reward Q-learning is employed to compute the remaining components of the approximate AOE solution, with rewards shifted by $-g_t$ from Algorithm~\ref{alg:mf-multichain} and states associated with $\tCr_t$ treated as terminal, with terminal rewards given by their estimated sub-MDP solutions from Algorithm~\ref{alg:mf-multichain}. States associated with $\tCb_t$ are further penalized by $\epsilon$ to ensure stability of Q-learning~\cite{Tsi94,YuB13}, analogously to the penalty argument used in the proof of Prop.~\ref{prp-idealized-alg}.
\end{itemize}

A technical complication arises in constructing $\tA_t$ and $\tCr_t, \tCe_t, \tCb_t$. Since these sets are determined by comparisons involving convergent random iterates from Algorithm~\ref{alg:mf-multichain}, using a single threshold can cause them to keep changing indefinitely when the limiting value being compared lies exactly on the threshold. To prevent this, Algorithm~\ref{alg:2} uses a pair of thresholds $\ud < \bd$ and separate rules for adding and retaining elements. The separation between the two thresholds ensures that a convergent sequence cannot cross both infinitely often. This two-threshold mechanism ensures finite-time stabilization (i.e., eventual constancy) of these sets (see Lemma~\ref{lem-1} below). It also contributes to the apparently more elaborate form of the set-construction rules below, despite the simplicity of the underlying mechanism.

For Algorithms~\ref{alg:2} and~\ref{alg:3}, we adopt the following convention:

\medskip
\noindent {\bf Information Structure Convention:} All variables are defined on a common probability space, and all updates are non-anticipatory (i.e., depend only on past and current information). Dependencies across algorithms and phases are specified by the information-flow statements preceding each algorithm.
\medskip

We now describe the construction of the relevant sets and the Q-learning step in Algorithm~\ref{alg:2}. Its starting time, parameters, and stepsizes may be chosen adaptively based on past information, consistent with the information-flow structure. 

\setcounter{myalgorithm}{1}
\begin{myalgorithm}[RL algorithm for approximate AOE solutions]
\label{alg:2} \hfill \\
\textbf{Information flow:} Algorithm~\ref{alg:2} $\leftarrow$ Algorithm~\ref{alg:mf-multichain}.\\
\textbf{Initialization:} Choose starting time $\bar t$, parameters $\epsilon > 0$, $\bd > \ud > 0$, and $Q_{\bar t} \in \R^{|\A|}$.\\*[0.3em]
For $t \geq \bar t$, perform the following operations:\\*[0.3em]
\noindent {\bf Set construction:}
\begin{itemize}[leftmargin=0.6cm,labelwidth=!]
\item[\bf 1.] Construct $\tA_t$ to approximate $\A_*$ by comparing $Q^o_t$ with $g_t$.\\
For $s \in C_{\ell,k} \in \C$, all class-preserving actions $\A_\ell(s)$ belong to $\A_*(s)$; hence, only the exiting actions from $\cE(s)$ need to be selected. For $s \in \T$, all actions need to be considered. Define the selected sets $\A^e_{t}(s)$, initialized by $\A^e_{\bar t - 1}(s) = \varnothing$, using the parameters $(\ud, \bd)$ as follows:
\begin{align}
 \forall \, s \in C_{\ell,k} \in \C: \quad  \A^e_t(s) & : = \big\{ a \in \cE(s) \setminus \A^e_{t-1}(s) \mid Q^o_t(\bar s_{\ell,k}, (s,a)) > g_{t}(s) - \ud \big\} \qquad \notag \\
 & \quad \ \ \cup \big\{ a \in \A^e_{t-1}(s) \mid Q^o_t(\bar s_{\ell,k}, (s,a)) > g_{t}(s) - \bd \big\};  \label{rule:Ae1}\\
 \forall \, s \in \T: \quad  \A^e_t(s) &: = \big\{ a \in \A(s) \setminus \A^e_{t-1}(s) \mid Q^o_t(s, a) > g_t(s) - \ud \big\} \notag \\
  & \quad \ \ \cup  \big\{ a \in \A^e_{t-1}(s) \mid Q^o_t(s, a) > g_t(s) - \bd \big\}.\label{rule:Ae2}
\end{align}
The first set in each union adds new elements using $\ud$, while the second retains existing elements using $\bd$. 
We then define
\begin{equation} \label{rule:tA}
\tA_t(s) : = \A_\ell(s) \cup \A^e_t(s) \ \  \ \forall \, s \in C_{\ell,k} \in \C; \qquad  \tA_t(s) : = \A^e_t(s) \ \ \  \forall \, s \in \T. \qquad \quad
\end{equation}
\item[\bf 2.] Construct a three-way partition of $\C$ that approximates the idealized partition, by comparing $Q^o_{t,c}$ with $\{r_{\ell, k}^{(t)}\}_{\ell,k}$.\\
Define a partition of $\C$ as follows, with $\tCe_{\bar t-1} = \varnothing$ initially:
\begin{align}
\tCr_t & : = \big\{ C_{\ell,k} \in \C \mid \A^e_t(s) = \varnothing, \forall  s \in C_{\ell,k} \big\}, \label{rule:tC1} \\
 \tCe_t & : = \big\{ C_{\ell,k} \in \C \setminus \tCe_{t-1} \mid r_{\ell, k}^{(t)} < Q^o_{t,c} (\bar s_{\ell,k}) - \bd \big\} 
 \cup \big\{ C_{\ell,k} \in \tCe_{t-1} \mid r_{\ell, k}^{(t)} < Q^o_{t,c} (\bar s_{\ell,k}) - \ud \big\}, \label{rule:tC2} \\
 \tCb_t  & : = \C \setminus \big( \tCr_t \cup \tCe_t \big).  \label{rule:tC3}
\end{align}
The two-threshold rule for $\tCe_t$ uses $\bd$ for adding new elements and $\ud$ for retaining existing ones; $\tCb_t$ is then determined by complement.
\end{itemize}
\smallskip
\noindent {\bf Q-learning step:} We apply total-reward Q-learning, shifting rewards by $-g_t$, treating $\tCr_t$-states as terminal states with terminal reward $v_t$, and further penalizing $\tCb_t$-states by $\epsilon$. In particular:
\begin{itemize}[leftmargin=0.6cm,labelwidth=!]
\item[1.] Let $ \tS^i_t : = \bigcup_{C \in \tC^i_t} C$ for $i \in \{\mathrm{r}, \mathrm{e}, \mathrm{b}\}$.
\item[2.] Sample a subset $Y_t \subset \tG_t := \{ (s,a) \in \tA_t \mid s \notin \tSr_t\}$.\\
For each $(s,a) \in Y_t$, sample a transition  to obtain reward $R^{sa}_{t+1}$ and next state $S^{sa}_{t+1}$.
\item[3.] Update $Q_t$ for each $(s,a) \in Y_t$ by
\begin{align}
\hspace{-1em}  
Q_{t+1}(s, a)  = Q_t(s,a) + \alpha_t(s,a) & \Big(R^{sa}_{t+1} - g_t(s) - \epsilon \, \ind\big\{ s \in \tSb_t \big\}  + \ind \big\{S^{sa}_{t+1} \in \tSr_t\big\} \cdot v_t(S^{sa}_{t+1}) \notag \\
  & \ + \ind\big\{S^{sa}_{t+1} \not\in \tSr_t \big\} \cdot\!\! \max_{a' \in \tA_t(S^{sa}_{t+1})} Q_t(S^{sa}_{t+1}, a') \, - Q_t(s,a) \Big), \label{eq-alg2-ql} 
\end{align} 
where $\alpha_t(s,a) \geq 0$ is a possibly random stepsize that can depend on information available at time $t$, including $Y_t$ (but not on the transition data indexed by $t+1$).
\end{itemize}
This completes the update at iteration $t$. \hfill $\diamond$
\end{myalgorithm}

\subsection{Convergence and performance results}

Let Algorithm~\ref{alg:2} be run under the following conditions, which are assumed throughout this subsection unless stated otherwise:
\begin{enumerate}[
    label=\textbf{(C\arabic*)},
    ref=(C\arabic*),
    leftmargin=1cm,
    labelwidth=!
]
\item \label{cond:C1} 
Algorithm~\ref{alg:mf-multichain} satisfies its convergence conditions. Furthermore, its Phase~2 outputs converge a.s.\ to an AOE solution for each $\M_{\ell,k}$, as specified earlier.
\item \label{cond:C2}
The information-flow structure indicated in Algorithm~\ref{alg:2} is respected. (An explicit formal statement of this condition in terms of stochastic kernels is given in Assumption~\ref{cond-meas-caus}, Section~\ref{sec-conv-proofs}.)
\item \label{cond:C3} 
Standard Q-learning data and stepsize conditions hold a.s.:
\begin{itemize}[leftmargin=0.65cm,labelwidth=!]
\item[(i)] Any state-action pair that occurs in $\tG_t$ infinitely often also occurs in $Y_t$ infinitely often.
\item[(ii)] For any $(s, a)$ that occurs infinitely often in $Y_t$, the stepsizes satisfy $\sum_{t} \alpha_t(s,a) = \infty$ and $\sum_{t} \alpha^2_t(s,a) < \infty$, under the convention that $\alpha_t(s,a) = 0$ whenever $(s,a) \not\in Y_t$.
\end{itemize}
\end{enumerate}

The next lemma establishes finite-time stabilization of the sets constructed by Algorithm~\ref{alg:2} and summarizes the properties of their limiting sets. In particular, the approximation properties stated below are precisely those sought in the algorithmic construction outlined earlier and will also be used in the convergence analysis.
The lemma follows directly from the update rules \eqref{rule:Ae1}--\eqref{rule:tC3}, the convergences $Q^o_t \to Q^o_*$, $g_t \to g^*$, and $r_{\ell,k}^{(t)} \to r_{\ell,k}^*$ from Algorithm~\ref{alg:mf-multichain}, and the relations \eqref{eq-aosp0}--\eqref{eq-aosp2}.

\begin{lemma}[Finite-time set stabilization] \label{lem-1}
Under \ref{cond:C1}, almost surely, the sets $\tC^i_{t}$, $\tS^i_{t}$ ($i \in \{\mathrm{r}, \mathrm{e}, \mathrm{b}\}$), $\tA_t$, and $\tG_t$ remain unchanged after finitely many iterations. Denote the resulting sets by $\tC^i$, $\tS^i$, $\tA$, and $\tG$; then $\tCr, \tCe, \tCb$ form a partition of $\C$, and 
$$\tS^i = \bigcup_{C \in \tC^i} C,  \quad i \in \{\mathrm{r}, \mathrm{e}, \mathrm{b}\}; \qquad \tG = \{ (s, a) \in \tA \mid s \not\in \tSr \}.$$
Moreover, the following hold:
\begin{equation} \label{eq-lem1-p1}
   \tA(s) = \A_*(s) \ \ \forall \, s \in \tSr, \qquad  \A_*(s) \subset \tA(s) \subset \A_{\bd}(s) \ \ \  \forall \, s \not\in \tSr,
\end{equation}
\begin{equation} \label{eq-lem1-p2}
  \C_0 \subset  \tCr \subset \Cr, \qquad \tCe \subset \Ce, \qquad \tCb \supset \Cb. \qquad \qquad
\end{equation}
For each $C_{\ell,k} \in \C$,
\begin{equation} \label{eq-lem1-p3}
 r^*_{\ell,k} - Q^o_{c} (\bar s_{\ell,k}) \in \begin{cases}
  [\ud, \infty) & \text{if} \ \,  C_{\ell,k} \in \tCr;   \\
  (-\infty, - \ud]  & \text{if} \  \, C_{\ell,k} \in \tCe; \\
  [-\bd, \bd]  & \text{if} \  \, C_{\ell,k} \in \tCb.
  \end{cases}
\end{equation}
Furthermore, $\tA(s) = \A_\ell(s)$ for all $s \in C_{\ell,k}$ with $C_{\ell,k} \in \tCr$, 
and for every $C_{\ell,k} \in \tCe \cup \tCb$,
\begin{equation} \label{eq-lem1-p4}
    \tA(s) \cap \cE(s) \not= \varnothing \ \ \ \text{for some} \ s \in C_{\ell,k}.
\end{equation}
\end{lemma}

The finite-time stabilization in Lemma~\ref{lem-1} allows the asymptotic behavior of the Q-learning step to be characterized in terms of the resulting fixed sets and a sample-path-dependent, SSP-type total-reward MDP determined by these sets, $v_*$, and the penalty parameter $\epsilon$. 
The resulting convergence theorem is stated below; its proof is given in Section~\ref{sec-conv-proofs}. 

\begin{theorem}[Convergence of Algorithm~\ref{alg:2}] \label{thm-1}
Let $\tC^i$, $\tS^i$ ($i \in \{\mathrm{r}, \mathrm{e}, \mathrm{b}\}$), $\tA$, and $\tG$ be the sample-path-dependent, eventually fixed sets guaranteed by Lemma~\ref{lem-1}.
Then the sequence $\{Q_t\}$ generated by Algorithm~\ref{alg:2} converges a.s.\ componentwise on $\tG$ to the unique solution $\hat Q^*$ of
\begin{align} 
   Q(s,a) & = r(s,a) - g^*(s) - \epsilon \1_{\tSb}(s) + \sum_{s' \in \tSr} p(s' \mid s, a) v_*(s') \notag \\
    & \quad + \sum_{s' \not\in \tSr} p(s' \mid s, a)  \max_{a' \in \tA(s')} Q(s',a'), \quad \qquad \qquad (s,a) \in \tG. \label{eq-thm1-Q-ssp-opt}
\end{align}
Consequently, if $x_t \in \R^{|\S|}$ is define by
\begin{equation} \label{def-xt}
  x_t(s) := \begin{cases}
   v_t(s)  & \text{if} \  s \in \tSr_t,\\
  \max_{a \in \tA_t(s)} Q_{t}(s, a)  &  \text{if} \  s \not\in \tSr_t,
  \end{cases}
\end{equation}
then $\{x_t\}$ converges a.s.\ to a solution $\bar x$ of the perturbed AOE,
\begin{equation} \label{eq-paoe}
    x  = \max_{\mu \in \Pi_{\tA}} \left\{ r_\mu - g^* + P_\mu x \right\} - \epsilon\1_{\tSb}, 
\end{equation}
and $\bar x$ is the unique solution that coincides with $v_*$ on $\tSr$. Finally, there exists a deterministic constant $\bd_0 > 0$ such that, for every $\bd < \bd_0$, $\tA = \A_*$ and the corresponding solution $\bar x$ converges, as $\epsilon \to 0$, to a solution of the AOE~\eqref{eq-aoe2}.
\end{theorem}

Thus, the theorem shows that, as a consequence of the Q-learning convergence,
Algorithm~\ref{alg:2} asymptotically constructs a solution to a sample-path-dependent perturbed AOE that extends the sub-MDP AOE solutions on $\tSr$, and thereby approximately solves the AOE~\eqref{eq-aoe2}. The relaxation of the constraint set $\A_*$ to $\tA$ and the penalty $\epsilon$ on $\tSb$ both contribute to the approximation error. If $\Cb = \varnothing$, i.e., there is no sub-MDP in the third category where remaining and exiting are both gain-optimal, then for sufficiently small $\bd$, $\tSb = \varnothing$ by Lemma~\ref{lem-1}, $\tA = \A_*$, and the algorithm asymptotically solves the AOE exactly.

The approximate AOE solutions targeted by Algorithm~\ref{alg:2}, as well as its iterates, yield policies whose suboptimality in gain can be bounded in terms of the algorithmic parameter $\bd$, which controls the accuracy of the outer approximations of both $\A_*$ and $\Cb$. We next quantify these performance guarantees for sufficiently large $t$.

Define an induced greedy policy $\mu_t \in \Pi_\A$ as follows: 
\begin{itemize}[leftmargin=0.6cm,labelwidth=!]
\item for $s \not\in \tSr_t$, let $\mu_t(s) \in \argmax_{a \in \tA_t(s)} Q_t(s,a)$; 
\item for $s \in \tSr_t$, let $\mu_t$ coincide on each $C_{\ell,k} \in \tCr_t$ with $\pi^{(t)}_{\ell,k}$, the policy for the sub-MDP $\M_{\ell,k}$ produced by RVI Q-learning in Phase 2 of Algorithm~\ref{alg:mf-multichain}. 
\end{itemize}
For $B \subset \S$, let $\tau_B$ denote the hitting time of $B$ for the state process $\{S_n\}_{n \geq 0}$; let $\bar \tau^\mu_B(s)$ denote its expected value under a policy $\mu$, starting from $S_0 = s$. That is, 
\[\tau_B := \min \{n \geq 0 \mid S_n \in B \}, \qquad \bar \tau^\mu_B(s) := \E^\mu_s [ \tau_B ].\]
We define $\phi_\mu(s, B)$ as the expected number of times, before hitting $B$, that the state lies in $\T$ or $\mu$ applies an exit action of Bather's decomposition: 
\begin{equation}
  \phi_\mu(s, B) :=  \E^\mu_s \Big[ \sum_{n=0}^{\tau_B-1} \ind \big\{ S_n \in \T \ \text{or} \  \mu(S_n) \in \cE(S_n)  \big\} \Big].
\end{equation}  

\begin{theorem}[Eventual performance bounds for Algorithm~\ref{alg:2}] \label{thm-2}
In the setting of Theorem~\ref{thm-1}, almost surely, for all sufficiently large $t$, the induced policy $\mu:=\mu_t$ satisfies $g_\mu(s) = g^*(s)$ for $s \in \tSr$, while all states in $\S \setminus \tSr$ are transient under $\mu$, with
\[g_{\mu}(s) - g^*(s) \geq -  \bd \,  \phi_\mu(s,  \tSr), \qquad  s \in \S \setminus \tSr.
\]
Hence, for all sufficiently small $\bd$, $\mu_t$ is eventually gain-optimal.
\end{theorem}

Before giving the proof, we make several remarks.

\begin{remark}[Sharpness and estimability of the bounds] \label{rem-sharp-est-bound}
In general, $\phi_\mu(s, B)$ can be bounded as
\begin{equation}
 \phi^*_\mu(s, B) :=  \E^\mu_s \Big[ \sum_{n=0}^{\tau_{B}-1} \ind \big\{ \mu(S_n) \not\in \A_*(S_n) \big\} \Big] \,  \leq \phi_\mu(s, B) \leq \bar \tau_{B}^\mu(s). \qquad \qquad
\end{equation} 
As the proof will show, the term $\phi_\mu(s, \tSr)$ in Theorem~\ref{thm-2} can be replaced by $\phi^*_\mu(s, \tSr)$, yielding a sharper performance bound. However, $\phi^*_\mu(s, \tSr)$ depends on the unknown set $\A_*$, whereas $\phi_\mu(s,  \tSr)$ is, in principle, estimable from sample trajectories under $\mu$, using Bather's decomposition.
\end{remark}

\begin{remark}[Effect of algorithmic parameters] \label{rem-alg2-param}
Theorem~\ref{thm-2} shows how the algorithmic parameters affect the induced policies $\mu_t$ for sufficiently large $t$. Smaller $\bd$, which controls the accuracy of the outer approximations of both $\A_*$ and $\Cb$, yields a tighter gain bound. In contrast, the penalty $\epsilon$ used to stabilize Q-learning does not affect this worst-case asymptotic bound, since it is applied only to $\tSb$, where remaining and exiting are nearly equally good in average reward (Lemma~\ref{lem-1}). It does affect policy behavior: the penalty ensures exit from the sub-MDPs in $\tSb$, with larger $\epsilon$ encouraging earlier exit.
\end{remark}

\begin{remark}[Alternative policies and algorithmic variants] \label{rem-alg2-variants}
\hfill \\
(a) Besides the greedy policies studied above, near-greedy policies induced by Algorithm~\ref{alg:2} can be analyzed similarly. We illustrate this with a case in which $\tSb$ need not be entirely transient under the induced policy. Let $\bar \epsilon$, $\epsilon$, $\ud$ satisfy $\epsilon \leq \bar \epsilon < \ud$. 
Define a policy $\mu'_t$ similarly to $\mu_t$, except that, for $s \not\in \tSr_t$, let $\mu'_t(s)$ be $\bar \epsilon$-greedy w.r.t.\ $Q_t$:
\begin{equation}
 \mu'_t(s) \in \Big\{ a \in \tA_t(s) \mid Q_t(s,a) > \max_{a \in \tA_t(s)} Q_t(s,a) - \bar \epsilon \Big\}, \qquad s \in \S \setminus \tSr_t.
 \end{equation}
Then, for sufficiently large $t$, $\mu:=\mu'_t$ is gain-optimal on $\tSr$, as before. On $\S \setminus \tSr$, the following additional properties hold under $\mu$: $\tSe$ is transient, while $\tSb$ may contain recurrent states. For the latter states,
\begin{equation} \label{eq-alg2var-bound1}
 g_\mu(s) - g^*(s) \geq - (\bar \epsilon - \epsilon), \quad s \in \Rec_\mu \cap \tSb;
\end{equation} 
and, for all transient states $s \in \S \setminus \tSr$,
\begin{equation} \label{eq-alg2var-bound2}
 g_\mu(s) - g^*(s) \geq - \bar \epsilon -  \bd \,  \phi_\mu(s, \tSr \cup \Rec_\mu).
\end{equation}
Derivations of these performance bounds are given at the end of Section~\ref{sec-proof-performance-alg2}.\\*[0.3em]
(b) Consider settings in which solving the AOE is not a primary goal and, for $C_{\ell,k} \in \tCb$, it is preferable to use the eventually optimal policies $\pi^{(t)}_{\ell,k}$ for $\M_{\ell,k}$, so as to remainwithin the sub-MDP rather than exit it. Algorithm~\ref{alg:2} can then be modified as follows.
In the Q-learning step, treat $\tCb_t$-states as terminal states as well, with terminal rewards specified by $v_t$, rather than penalizing them. The update rule \eqref{eq-alg2-ql} can be modified accordingly and carried out only for selected state-action pairs in $\{(s, a) \in \tA_t \mid s \not\in \tSr_t \cup \tSb_t\}$. On $\S \setminus (\tSr \cup \tSb)$, eventual performance bounds for the resulting greedy policies (omitted here) can be derived by similar arguments.\\*[0.3em]
(c) Finally, as an alternative to classical Q-learning, the Q-learning step can be implemented using a policy-iteration-like Q-learning variant \cite{BeY12,YuB12}, which avoids maximization over all actions at each iteration while retaining the same convergence guarantee.
\end{remark}

\subsection{Proof of Theorem~\ref{thm-2} and supporting derivations} \label{sec-proof-performance-alg2}

For a stationary policy $\mu$, let $P^*_\mu$ denote the \emph{limiting matrix} and $Z_\mu$ the \emph{fundamental matrix}. Recall that $P^*_\mu = \lim_{n \to \infty} P^n_\mu$ in the aperiodic case and in general equals the Ces\`{a}ro limit of $P_\mu^n$. By definition, $Z_\mu = (I - P_\mu + P^*_\mu)^{-1}$. In terms of these matrices, the gain and bias of $\mu$ are given, respectively, by 
\[g_\mu = P^*_\mu r_\mu, \qquad h_\mu = Z_\mu ( r_\mu - g_\mu).\]
See, e.g., \cite[Secs.\ 8.2 and A.5]{Put94} for these standard facts. We will use the following properties of $Z_\mu$ (see, e.g., \cite{Bla62,FeS84}):
\begin{align}
 & Z_\mu (P_\mu - I )  = P^*_\mu - I,  \label{eq-Z-1}\\
 & Z_\mu  = I + \lim_{\beta \uparrow 1} \sum_{n=1}^\infty \beta^n (P_\mu^n - P^*_\mu). \label{eq-Z-2}
\end{align}
The following lemma is a direct consequence of \eqref{eq-Z-2}.

\begin{lemma} \label{lem-Z}
For any $B \subset \T_\mu$ such that $B^c$ is closed under $\mu$, the submatrix $[Z_\mu]_{B,B}$ is nonnegative, and its row sums equal $\bar \tau^\mu_{B^c}(s)$ for $s \in B$. Moreover, for $D \subset B$ and $s \in B$,
\[   \big([Z_\mu]_{B,B} \1_D\big)(s) = \E_s^\mu \Big[ \sum_{n=0}^{\tau_{B^c} -1} \ind \{ S_n \in D \} \Big] \leq \bar \tau^\mu_{B^c}(s).\]
\end{lemma}

By the definition of $\A_*$, $(P_\mu g^*)(s) = g^*(s)$ if and only if $\mu(s) \in \A_*(s)$.
If $s$ is recurrent under $\mu$, then $(P_\mu g^*)(s) = (P^*_\mu g^*)(s) = g^*(s)$. The next lemma bounds $(P^*_\mu g^*)(s) - g^*(s)$ for transient states under $\mu$ and will be used repeatedly in the performance-bounds analysis.

\begin{lemma} \label{lem-bound-gain}
Let $\mu$ be a policy in $\Pi_{\A_\delta}$ with $\delta > 0$. Let $B \subset \T_\mu$ be such that $B^c$ is closed under $\mu$ and $(P_\mu g^*)(s) = g^*(s)$ for all $s \in B^c$.
Then, for $s \in B$,
\begin{align}
  (P^*_\mu g^*)(s) - g^*(s) & \geq - \delta \, \E^\mu_s \Big[ \sum_{n=0}^{\tau_{B^c}-1} \ind \big\{ \mu(S_n) \not\in \A_*(S_n) \big\} \Big] \label{eq-bound-gain1} \\
  & \geq - \delta \, \E^\mu_s \Big[ \sum_{n=0}^{\tau_{B^c}-1} \ind \big\{ S_n \in \T \ \text{or} \  \mu(S_n) \in \cE(S_n)  \big\} \Big], \label{eq-bound-gain2}
\end{align}  
or equivalently, $(P^*_\mu g^*)(s) - g^*(s) \geq - \delta \, \phi^*_\mu (s, B^c) \geq - \delta \, \phi_\mu (s, B^c)$.
\end{lemma}

\begin{proof}
Let $E := \big\{ s \in \S \mid \mu(s) \not\in \A_*(s) \big\}$ and $E' := \big\{ s \in \S \mid s \in \T \ \text{or} \  \mu(s) \in \cE(s) \big\}.$
Since each $s$ lies either in $\T$ or in some $C_{\ell,k} \in \C$, and in the latter case $\A(s) = \A_\ell(s) \cup \cE(s)$ with $\A_\ell(s) \subset \A_*(s)$, we have $E \subset E'$.

By~\eqref{eq-Z-1}, 
\[ P^*_\mu g^* - g^* = Z_\mu (P_\mu - I ) g^*.\]
From the definition of $\A_\delta$ and $E$, 
\[ (P_\mu - I ) g^* \geq -\delta \1_E. \] 
Moreover, $[(P_\mu - I ) g^*](s) = 0$ for $s \in B^c$ by assumption. Since $[Z_\mu]_{B,B}$ is nonnegative, it follows that, for $s \in B$,
\[ [Z_\mu (P_\mu - I ) g^*](s) \geq -\delta \, [Z_\mu \1_{E \cap B} ](s).\]
Lemma~\ref{lem-Z} identifies $[Z_\mu \1_{E \cap B} ](s)$ with the expected number of visits to $E \cap B$ before hitting $B^c$, giving \eqref{eq-bound-gain1}. The inclusion $E \subset E'$ then yields \eqref{eq-bound-gain2}.
\end{proof} 

We now prove Theorem~\ref{thm-2}. For brevity, we omit the qualifier ``a.s.'' throughout. 

\begin{proof}[Proof of Theorem~\ref{thm-2}]
Let $t$ be sufficiently large so that the sets constructed in Algorithm~\ref{alg:2} remain unchanged and identify with the respective limiting sets given by Lemma~\ref{lem-1}.

\medskip
\noindent {\bf (i) Eventual gain-optimality of $\mu_t$ on $\tSr$.}\\*[0.5em]
The claim follows from the convergence of RVI Q-learning. In particular, for each $C_{\ell,k} \in \tCr$, RVI Q-learning converges in the communicating sub-MDP $\M_{\ell,k}$, producing policies $\pi^{(t)}_{\ell,k}$ that are gain-optimal for the sub-MDP for all sufficiently large $t$; see \cite{WYS24,YWS25b} and the asynchronous stochastic approximation results in \cite{YWS25a}. 

Since $\tCr \subset \Cr$ (Lemma~\ref{lem-1}), the defining property for $\Cr$ (cf.\ \eqref{eq-def-C1}) implies that these policies $\pi^{(t)}_{\ell,k}$ are also eventually gain-optimal in the original MDP $\M$ for states in $C_{\ell,k}$. As $\mu_t$ eventually coincides with these policies on each $C_{\ell,k} \in \tCr$, it follows that $\mu_t$ is eventually gain-optimal on $\tSr$.

\medskip
\noindent {\bf (ii) Bounding $g_{\mu_t}$ on $\S \setminus \tSr$.}\\*[0.5em] 
We now bound the eventual performance of $\mu_t$ for initial states in $\S \setminus \tSr$.
By Theorem~\ref{thm-1}, $Q_t$ converges to $\hat Q^*$ on $\tG$, where $\hat Q^*$ is the optimal action-value function of an associated SSP-type total-reward MDP $\hat \M$ defined on $\tG$ (Lemma~\ref{lem-4.1}). Therefore, for all sufficiently large $t$, the policy $\mu_t$ is greedy w.r.t.\ $\hat Q^*$ on $\tG$ and satisfies the optimality equation of $\hat \M$.
Hence, writing $\mu := \mu_t$, we have
\begin{equation} \label{eq-prf-thm2-1}
 \bar x = r_\mu - g^* +  P_\mu \bar x - \epsilon \1_{\tSb}.
\end{equation}
Left-multiplying \eqref{eq-prf-thm2-1} by $P^*_\mu$, and noting that $P^*_\mu \bar x = P^*_\mu P_\mu \bar x$ and $g_\mu = P^*_\mu r_\mu$, we obtain
\begin{equation} \label{eq-prf-thm2-2}
   g_\mu = P^*_\mu g^* + \epsilon P^*_\mu \1_{\tSb}.
\end{equation}   

Under $\mu$, $\tSr$ is closed, while $\S \setminus \tSr \subset \T_\mu$. Indeed, since $\mu$ restricted to $\S \setminus \tSr$ solves the optimality equation of the SSP-type MDP $\hat \M$, it is optimal and hence proper for $\hat \M$. Thus, the entire effective state space $\S \setminus \tSr$ of $\hat \M$ is transient under $\mu$. In particular, $\tSb \subset \T_\mu$ and hence $P^*_\mu \1_{\tSb} = 0$, reducing \eqref{eq-prf-thm2-2} to
\begin{equation} \label{eq-prf-thm2-2b}
   g_\mu = P^*_\mu g^*.
\end{equation}  

We now use Lemma~\ref{lem-bound-gain} to bound $P^*_\mu g^* - g^*$ on $\S \setminus \tSr$. By \eqref{eq-lem1-p1} from Lemma~\ref{lem-1}, we have $\mu \in \Pi_{\A_{\bd}}$ and $(P_\mu g^*)(s) = g^*(s)$ for $s \in \tSr$. Applying Lemma~\ref{lem-bound-gain} with $B =  \S \setminus \tSr$ then yields the desired bound
\[ g_\mu(s) - g^*(s) = (P^*_\mu g^*)(s) - g^*(s) \geq - \bd \, \phi_\mu(s, \tSr), \qquad \forall \, s \in \S \setminus \tSr.\]

Since $\Pi_\A$ and $\S$ are finite, the non-zero gain gaps are bounded away from zero. Moreover, $\phi_\mu(s, \tSr) \leq \bar{\tau}_{\tSr}(s)$, and the latter is uniformly bounded over all possible $s$, $\mu$, and $\tSr$. Hence, if the parameter $\bd$ is below a deterministic $\M$-dependent threshold, the above bound implies $g_\mu = g^*$, so $\mu$ is gain-optimal.
\end{proof}

\medskip
\noindent {\bf Derivations of \eqref{eq-alg2var-bound1}--\eqref{eq-alg2var-bound2} for the near-greedy policies in Remark~\ref{rem-alg2-variants}(a):}\\*[0.5em]
Consider now the near-greedy policies $\mu'_t$ in Remark~\ref{rem-alg2-variants}(a). Their definition together with the convergence of Algorithm~\ref{alg:2} (Theorem~\ref{thm-2}) implies that, for sufficiently large $t$, the policy $\mu := \mu'_{t}$ satisfies
\begin{equation} \label{eq-prf-thm2-4a}
  r_\mu - g^* +  P_\mu \bar x - \epsilon \1_{\tSb}  \geq \bar x - \bar \epsilon \1_{\S \setminus \tSr},
\end{equation}
where $\bar x$ is the solution of the perturbed AOE in Theorem~\ref{thm-2}. 
As in the preceding proof, left-multiplying \eqref{eq-prf-thm2-4a} by $P^*_\mu$ and rearranging terms yields
\begin{equation} \label{eq-prf-thm2-4b}
   g_\mu \geq P^*_\mu g^* + \epsilon P^*_\mu \1_{\tSb} - \bar \epsilon P^*_\mu \1_{\S \setminus \tSr}.
\end{equation}

Suppose that $\mu$ induces a recurrent class $E$ outside $\tSr$. Then $g^*$ is constant on $E$ and hence $(P^*_\mu g^*)(s) = g^*(s)$ for all $s \in E$. Moreover, by the properties of Bather's decomposition, $E \subset C_{\ell,k}$ for some $C_{\ell,k} \in \tCe \cup \tCb$. It then follows from \eqref{eq-prf-thm2-4b} that, for recurrent states $s \in \S \setminus \tSr$,
\begin{equation} \label{eq-prf-thm2-5}
     g_\mu(s) - g^*(s) \geq \begin{cases}
      - \bar \epsilon & \text{if} \  s \in \Rec_\mu \cap \tSe; \\
      - (\bar \epsilon - \epsilon) & \text{if} \ s \in \Rec_\mu \cap \tSb.
     \end{cases}
\end{equation}
The second bound establishes \eqref{eq-alg2var-bound1}.

The first bound in \eqref{eq-prf-thm2-5} implies that $\Rec_\mu \cap \tSe = \varnothing$. Indeed, for $s \in \Rec_\mu \cap \tSe$, that bound gives $g_\mu(s) - g^*(s) \geq - \bar \epsilon$, whereas \eqref{eq-lem1-p3} in Lemma~\ref{lem-1}, applied to $C_{\ell,k} \in \tCe$, gives $g_\mu(s) \leq g^*(s) - \ud$, contradicting $\bar \epsilon < \ud$.

Finally, for transient states $s \in \S \setminus \tSr$, note that $\tSr \cup \Rec_\mu$ is closed and $(P_\mu g^*)(s) = g^*(s)$ on this set. We can therefore apply Lemma~\ref{lem-bound-gain} with $B = \S \setminus (\tSr \cup \Rec_\mu) = \T_\mu \setminus \tSr$ to obtain
\[  \big(P^*_\mu g^* - g^*\big)(s) \geq - \bd \, \phi_\mu(s, \tSr \cup \Rec_\mu), \qquad s \in \T_\mu \setminus \tSr.\]
Combining this bound with \eqref{eq-prf-thm2-4b} establishes \eqref{eq-alg2var-bound2}.

\section{Toward bias-optimal policies via approximation} \label{sec:alg3}

In this section, we consider the computation of (near) bias-optimal policies, building on the base algorithm. Whereas gain optimality concerns long-run average reward, bias provides a stronger criterion for comparing gain-optimal policies by accounting for their transient reward behavior. When multiple policies attain the same optimal gain, bias optimality can therefore provide a further criterion for distinguishing among them.

Let $\mu$ be a stationary policy. By standard MDP theory \cite[Sec.\ 8.2]{Put94}, its bias $h_\mu = Z_\mu (r_\mu - g_\mu)$, or equivalently,
\[ 
    h_\mu(s) =  \Clim_{n \to \infty} \E^\mu_s \Big[  \sum_{t=0}^n \big( r_\mu(s_t) - g_\mu(s_t) \big) \Big], \qquad s \in \S,
\]    
where $\Clim$ denotes the Ces\`{a}ro limit and can be replaced by the ordinary limit in the aperiodic case. The terms $\tfrac{g_\mu}{1 - \beta}$ and $h_\mu$ are the first two terms in the Laurent series expansion of the $\beta$-discounted value function $v_{\beta,\mu}$ of $\mu$. We call $\mu$ \emph{bias-optimal} if it is gain-optimal and $h_\mu \geq h_\pi$ for all stationary gain-optimal policies $\pi$. Such a policy $\mu$ exists in $\Pi_\A$. The \emph{optimal bias function} $h^* : = \sup_{\mu:\,g_\mu = g^*} h_\mu$ is a solution of the AOE~\eqref{eq-aoe2}. 

The bias optimality criterion is closely related to two other equivalent criteria: the \emph{$0$-discount optimality} criterion and the \emph{average-overtaking optimality} criterion. Any bias-optimal policy is also optimal w.r.t.\ these criteria, meaning that, for all policies $\pi$ and $s \in \S$, 
\[ 
\liminf_{\beta \uparrow 1} \big(v_{\beta, \mu}(s) - v_{\beta,\pi}(s) \big) \geq 0, \qquad 
\liminf_{n \to \infty} \frac{1}{n} \sum_{k=1}^n (v_{k, \mu}(s) - v_{k,\pi}(s)) \geq 0,
\]
where $v_{k, \pi}$ denotes the $k$-stage value function of $\pi$. See \cite[Chap.\ 10]{Put94} for details.

MDP theory \cite[Sec.\ 10.2]{Put94} tells us that, given $g^*$, a bias-optimal policy can be found by solving the following pair of nested equations, which uniquely determines the optimal bias function $h^*$ and involves an associated vector $y \in \R^{|\S|}$:
\begin{equation} \label{eq-aoe-bias}
  h^* = \max_{\mu \in \Pi_{\A_*}} \{ r_\mu - g^* + P_\mu h^* \}, \qquad y = \max_{\mu \in \Pi_{\A_*[h^*]}} \{ - h^* + P_\mu y \}.
\end{equation}  
Here, $\A_*[h^*] = \big\{\A_*[h^*](s)\big\}_{s\in \S}$, where $\A_*[h^*](s)$ consists of the actions in $\A_*(s)$ that attain the maximum in the first equation for state $s$. Any policy in $\Pi_{\A_*[h^*]}$ attaining the maximum in the second equation is bias-optimal. Rather than tackling these nested equations directly, we shall again take an approximation approach inspired by Federgruen and Schweitzer \cite{FeS84}. The main idea is, roughly speaking, as follows. 

We seek an approximation $\tilde{h}$ of the optimal bias function $h^*$ by leveraging the convergent estimates of $g^*$ and outer approximations of $\A_*$ that can be obtained from the outputs of Algorithm~\ref{alg:mf-multichain}. One option, pursued below, is to do this indirectly: we approximate $h^*$ by a function $\tilde{h} \approx \tilde{h}^*$, where $\tilde{h}^*$ is the optimal bias function of another MDP $\widetilde{\M} = (\S, \tA, p, r - g^*)$ with $\tA$ an outer approximation of $\A_*$ constructed as in Algorithm~\ref{alg:2}. We then apply the base algorithm to solve an average-reward MDP $\M'$ whose one-stage reward is $- \tilde{h}$ and whose action constraints approximate $\A_*[h^*]$. When these approximations are sufficiently accurate, we expect the gain-optimal policies of $\M'$ to be near bias-optimal in the original MDP.

A crucial step in the approach just outlined is to approximate the optimal bias function $\tilde{h}^*$ of $\widetilde{\M}$. In this section, we consider a straightforward way to do so: discounted approximation with large discount factors, applied to the gain-shifted MDP $\widetilde{\M}$. (Alternative approaches are possible but appear more challenging from both computational and analytical perspectives.) The next lemma gives the properties of the discounted approximations used here and largely follows from the results of \citet{Bla62}.

\begin{lemma} \label{lem-approx-hatM}
For any action constraint set $\tA$ with $\A_* \subset \tA \subset \A$, consider the MDP $\widetilde{\M} = (\S, \tA, p, r - g^*)$. Then the optimal gain in $\widetilde{\M}$ is $0$, and its optimal bias function $\tilde{h}^*$ satisfies $\tilde{h}^* \geq h^*$, with $\tilde{h}^* = h^*$ if $\tA = \A_*$. Moreover, for $\beta \in [0,1)$, let $v_\beta$ denote the $\beta$-discounted optimal value function of $\widetilde{\M}$. Then 
\begin{equation} v_\beta \to \tilde{h}^* \ \ \ \text{as} \  \beta \uparrow 1. 
\end{equation}
Consequently, $\{v_\beta\}_{\beta \in [0,1)}$ is uniformly bounded (indeed, uniformly over admissible $\tA$ as well), and
\begin{equation}
  \eta_\beta : = \| v_\beta - \tilde{h}^*\|_\infty \to 0, \qquad \epsilon_\beta := (1 - \beta) \| v_\beta \|_\infty \to 0, \quad \text{as} \  \beta \uparrow 1.
\end{equation}  
\end{lemma}

\begin{remark}
An important feature of the discounted approximation above is the uniform boundedness of $\{v_\beta\}_{\beta \in [0,1)}$, which stems from shifting the rewards by the exact optimal gain function $g^*$. The ability to obtain $g^*$ in the limit, rather than merely an approximation of it, is therefore an important advantage of the base algorithm in the present approach. This boundedness is also favorable numerically, as it prevents the value estimates from growing in magnitude as $\beta \uparrow 1$, in contrast to discounted approximation applied directly to the original MDP, whose discounted optimal value functions usually diverge in magnitude.
\end{remark}

\subsection{Algorithm}

\setcounter{algorithm}{2}
\FloatBarrier
\begin{algorithm}[!t]
\caption{Approximate Method toward Bias-Optimal Policies (Conceptual Description)} \label{alg:3}
\begin{algorithmic}[1]
\Statex \hspace{-\algorithmicindent}%
\begin{tabular}[t]{@{}l@{\quad}l@{}}
\textbf{Information flow:}
  & Algorithm~3 (Phase 1) $\leftarrow$ Algorithm~1;\\
  & Algorithm~3 (Phase 2) $\leftarrow$
    \{Algorithm~1, Algorithm~3 (Phase 1)\}.
\end{tabular}  
\Statex \hspace{-\algorithmicindent} Initialize Phase-1 starting time $\bar t$ and discount factor $\beta \approx 1$
\Statex \hspace{-\algorithmicindent} Execute Phases 1 and 2 concurrently with the above dependency structure:\smallskip

% phase 1
\myPhase{1} Action selection and $\beta$-discounted Q-learning (rewards shifted by $- g_t$)
\InPhase{Initialize $Q_{\beta, \bar t}$, parameters $\bd_1 > \ud_1 > 0$, and action constraint set $\tA_{\bar t}$}
\ForInPhase{$t \geq \bar t$}
    \InFor{Sample a random subset $Y_t \subset \tA_{t}$}
    \InFor{Apply Q-learning update \eqref{eq-alg3g-ql} to $Q_{\beta, t}$ on $Y_t$}
    \InFor{$t \leftarrow t+1$ and update $\tA_t$ via Rules \eqref{rule:Ae1}--\eqref{rule:tA} with two thresholds $(\ud_1, \bd_1)$}
\EndForPhase
\smallskip

% phase 2
\myPhase{2} Average-reward optimization (reward $-v_{\beta,t}$, action sets induced by $Q_{\beta,t}$) 
\InPhase{Initialize starting time $\bar t'$, parameters $\bd_2 > \ud_2 > 0$, and action constraint set $\tB_{\bar t'}$}
\ForInPhase{$t \geq \bar t'$}
\InFor{Update Bather's decomposition for $\tB_t$ if $\tB_t \not= \tB_{t-1}$}
\InFor{Run one iteration of Phases 2--4 of the base algorithm} 
\InFor{\ \ \ with rewards $\{-v_{\beta,t}(s)\}_{s \in \S}$ and action sets $\{\tB_t(s) \}_{s \in \S}$}
\InFor{$\mu_t \leftarrow \pi^{(t)}$ (Phase 4 policy)}
\InFor{$t \leftarrow t+1$ and update $\tB_t$ via Rule~\eqref{rule:tB} with two thresholds $(\ud_2, \bd_2)$}
\EndForPhase
\end{algorithmic}
\end{algorithm}

Algorithm~\ref{alg:3}, whose pseudocode and details are presented in what follows, implements the two steps of the approach outlined above in two concurrent phases. It initializes and updates the action constraint sets $\tA_t$ according to the same set-construction rules \eqref{rule:Ae1}--\eqref{rule:tA} as in Algorithm~\ref{alg:2}. These sets eventually stabilize to a sample-path-dependent outer approximation $\tA$ of $\A_*$. The resulting $\tA$ directly defines the action constraints of $\widetilde{\M}$ and indirectly those of $\M'$, so both MDPs are sample-path dependent.
\begin{itemize}[leftmargin=0.65cm,labelwidth=!]
\item[(i)] Phase~1 approximates the optimal bias function $h^*$ via discounted Q-learning with discount factor $\beta$ on selected state-action pairs in $\tA_t$, using $g_t$ from Algorithm~\ref{alg:mf-multichain} to shift the rewards. The resulting Q-function estimates are denoted by $Q_{\beta,t}$.
\item[(ii)] Phase~2 constructs action constraint sets $\tB_t \subset \tA_t$ from near-greedy actions w.r.t.\ $Q_{\beta,t}$, and applies the base algorithm to the resulting approximation of $\M'$.
\end{itemize}
For each phase, the starting time, parameters, and stepsizes (if any) can be chosen based on past information while respecting the information-flow structure.

We now give the update details. Denote 
\[ v_{\beta,t}(s) := \max_{a \in \tA_t(s)} Q_{\beta,t}(s, a), \qquad s \in \S.\] 
In Phase~1, Algorithm~\ref{alg:3} samples a subset $Y_t \subset \tA_{t}$ and updates $Q_{\beta,t}$ on $Y_t$ as follows:
\begin{itemize}[leftmargin=0.6cm,labelwidth=!]
\item[1.] For each $(s,a) \in Y_t$, sample a transition to obtain a reward $R^{sa}_{t+1}$ and next state $S^{sa}_{t+1}$.  
\item[2.] Update $Q_{\beta,t}$ for each $(s,a) \in Y_t$ by
\begin{align}
  Q_{\beta,t+1}(s, a)  = Q_{\beta,t}(s,a) + \alpha_t(s,a)  \big(R^{sa}_{t+1} - g_t(s) + \beta v_{\beta,t}(S^{sa}_{t+1}) \, - Q_{\beta,t}(s,a) \big), \label{eq-alg3g-ql}
\end{align} 
where $\alpha_t(s,a) \geq 0$ is a possibly random stepsize that can depend on information available at time $t$, including $Y_t$ (but not on the transition data indexed by $t+1$).
\end{itemize}
In Phase~2, define $\tB_{t} \subset \tA_t$, with $\tB_{{\bar t'} -1}(s) = \varnothing$ initially, by
\begin{align} \label{rule:tB}
\tB_t(s) : = & \big\{ a \in \tA_t(s) \setminus \tB_{t-1}(s) \mid Q_{\beta,t}(s, a) > v_{\beta,t}(s) - \ud_2 \big\} \notag \\
    & \cup \, \big\{ a \in \tA_t(s) \cap \tB_{t-1}(s) \mid Q_{\beta,t}(s, a) > v_{\beta,t}(s) - \bd_2 \big\}, \quad \forall \, s \in \S.
\end{align}
Phase~2 then runs Phases 2--4 of the base algorithm with one-stage reward $- v_{\beta,t}$ and action constraint set $\tB_t$, as specified in the pseudocode.
As in the construction of $\tA_t$, two thresholds $(\ud_2, \bd_2)$ are used in \eqref{rule:tB} for adding new elements and retaining existing ones; finite-time stabilization of $\tB_t$ is established in Lemma~\ref{lem-stab-sets} below.
 
\subsection{Convergence and performance results}
 
Let Algorithm~\ref{alg:3} be run under the following conditions, which are assumed throughout this subsection:
\begin{enumerate}[
    label=\textbf{(D\arabic*)},
    ref=(D\arabic*),
    leftmargin=1cm,
    labelwidth=!
]
\item \label{cond:D1} 
The base algorithm satisfies its required convergence conditions for each MDP instance to which it is applied and each initialization of its internal parameters (including its application within Algorithm~\ref{alg:mf-multichain} to the original MDP).
\item \label{cond:D2}
The information-flow structure specified in Algorithm~\ref{alg:3} is respected.
\item \label{cond:D3} The Q-learning updates in Phase~1 satisfy the standard sampling and stepsize conditions: specifically, condition~\ref{cond:C3} holds with $Y_t$ and $\tA_t$ here in place of $Y_t$ and $\tG_t$, respectively, in Algorithm~\ref{alg:2}.
\end{enumerate}

The next lemma establishes the desired finite-time stablization of the action sets that define the limiting MDPs $\widetilde{\M}$ and $\M'$. The subsequent theorem then establishes that the policies produced in Phase~2 are eventually gain-optimal for $\M'$. The proofs of these results are given in Section~\ref{sec-conv-proofs}.

\begin{lemma}[Finite-time stabilization of action sets and asymptotic convergence of Phase~1] \label{lem-stab-sets} \hfill \\
Almost surely, $\tA_t$ and $\tB_t$ stabilize after a finite time, yielding (random) limiting sets $\tA$ and $\tB \subset \tA$, such that for all $s \in \S$,
\begin{align}
  \A_*(s) \subset \tA(s)  \subset \A_{\bd_1}(s),   \qquad %\label{eq-stab-tA} \\   
   \B_{*}(s) \subset \tB(s)  \subset \B_{\bd_2}(s),    \label{eq-stab-tAB}
\end{align}
where the sets $\B_{*}(s)$ and $\B_{\delta}(s)$, $\delta > 0$, are defined as follows. For the sample-path-dependent MDP $\widetilde{\M} := (\S, \tA, p, r - g^*)$, let $v_\beta$ (resp.\ $Q_\beta$) denote its $\beta$-discounted optimal value (resp.\ action-value) function. Then
\[ \B_{*}(s) := \{ a \in \tA(s) \mid Q_\beta(s,a) = v_\beta(s) \}, \qquad \B_{\delta}(s) := \{ a \in \tA(s) \mid Q_\beta(s,a) \geq v_\beta(s) - \delta \}.\]
Moreover, on $\tA$, the sequence $\{Q_{\beta,t}\}$ produced in Phase~1 converges a.s.\ to $Q_\beta$.
\end{lemma}

Note that, for each state $s$, $\B_{*}(s)$ consists of the greedy actions w.r.t.\ the $\beta$-discounted optimal action-value function $Q_\beta$ for $\widetilde{\M}$, while $\tB(s)$ outer-approximates $\B_*(s)$ by near-greedy actions, with $\bd_2$ controlling the approximation tolerance. Thus, $\Pi_{\B_*}$ is the class of (nonrandomzied stationary) $\beta$-discounted optimal policies for $\widetilde{\M}$, whereas $\Pi_{\tB}$ is the corresponding enlarged class induced by near-greedy actions.

\begin{theorem} \label{thm-b-conv}
Almost surely, for all sufficiently large $t$, the policy $\mu_t$ generated by Algorithm~\ref{alg:3} (Phase~2) is gain-optimal for the sample-path-dependent MDP $\M' :=(\S, \tB, p, - v_\beta)$, where $\tB$ and $v_\beta$ are as given in Lemma~\ref{lem-stab-sets}.
\end{theorem} 

Next, we bound the performance of the policies generated in Algorithm~\ref{alg:3} for sufficiently large $t$. For the sample-path-dependent MDP $\widetilde{\M} = (\S, \tA, p, r - g^*)$ in Lemma~\ref{lem-stab-sets}, recall the quantities introduced in Lemma~\ref{lem-approx-hatM}:
\[
 \eta_\beta  = \| v_\beta - \tilde{h}^*\|_\infty, \qquad \epsilon_\beta = (1 - \beta) \| v_\beta \|_\infty.
 \]
By that lemma, both vanish as $\beta \uparrow 1$. These quantities capture the errors associated with the discounted approximation and enter the performance bounds below. 

The next theorem quantifies how these approximations translate into gain-optimality guarantees in the original MDP $\M$ for policies admitted by the limiting action constraints $\tB$, with stronger bias-optimality guarantees for gain-optimal policies of $\M'$.  

\begin{theorem}[Eventual performance bounds] \label{thm-3-gen}
Let $\tA$, $\tB$, $v_\beta$, and $\M'=(\S,\tB,p,-v_\beta)$ be as above.
\begin{itemize}[leftmargin=0.75cm,labelwidth=!]
\item[\rm (i)] Every $\mu\in\Pi_{\tB}$ satisfies
\begin{equation}
g_\mu(s)-g^*(s)
\geq -\epsilon_\beta-\bd_2,
\qquad s\in\Rec_\mu,
\end{equation}
\begin{equation}
g_\mu(s)-g^*(s)
\geq -\epsilon_\beta-\bd_2
-\bd_1\phi_\mu(s,\Rec_\mu),
\qquad s\in\T_\mu.
\end{equation}
Consequently, every $\mu\in\Pi_{\tB}$ is gain-optimal for $\M$
if $\bd_1,\bd_2$, and $1-\beta$ are sufficiently small.
\end{itemize}
Let $\bar \mu \in \Pi_{\tB}$ be any gain-optimal policy for $\M'$. Then:
\begin{itemize}[leftmargin=0.75cm,labelwidth=!]
\item[\rm (ii)] The policy $\bar \mu$ is
bias-optimal within $\Pi_{\tB}$ for a perturbed MDP $\M_{\rm pert}$ whose reward
function differs from that of $\M$ by the shift $-g^*$ and a
perturbation uniformly bounded by $\epsilon_\beta+\bd_2$.
\item[\rm (iii)] If $\bar \mu$ is gain-optimal for $\M$,
then it is $\epsilon'$-optimal for $\M$ w.r.t.\ bias-optimality, where
\[
\epsilon'=\eta_\beta+c(\epsilon_\beta+\bd_2)
\]
for some $\M$-dependent constant $c \geq 1$. In particular, under the gain-optimality condition above, $\bar \mu$ is bias-optimal for $\M$ whenever $\epsilon'$ falls below an $\M$-dependent threshold, which is ensured by sufficiently small $1 - \beta$ and $\bd_2$.
\end{itemize}
Consequently, combining Lemma~\ref{lem-stab-sets} and Theorem~\ref{thm-b-conv}, almost surely, for all sufficiently large $t$, the following hold:
$\Pi_{\tB_t}=\Pi_{\tB}$; every policy in $\Pi_{\tB_t}$ satisfies (i); and the policy $\mu_t$ generated in Phase~2 of Algorithm~\ref{alg:3} additionally satisfies (ii) and (iii). In particular, for sufficiently small $\bd_1, \bd_2$, and $1 - \beta$, $\mu_t$ is gain-optimal and $\epsilon'$-optimal w.r.t.\ bias-optimality for $\M$; it is bias optimal whenever $\epsilon'$ is below the threshold in (iii). 
\end{theorem}

The proof of this theorem is given in Section~\ref{sec-alg3-performance-proof}. 

The bias-optimality bound in part (iii) is a consequence of stronger statewise bounds developed below (see \eqref{eq-b-p2-pol-a}). These bounds also characterize the performance of $\bar \mu$ when it is not gain-optimal from every initial state and will be useful in interpreting the numerical results. 

Another subtlety is that, with sufficiently accurate approximations, Phase~1 alone is sufficient to produce near bias-optimal policies in the limit. This raises the question of what additional advantages Phase~2 can offer. We investigate this question below. \vspace*{-0.5em}

\subsubsection*{Statewise bias bounds and the role of Phase~2}

For any policy $\mu \in \Pi_\A$, let 
$D^*_{\mu} := \{ s \in \S \mid g_{\mu}(s) = g^*(s) \},$ 
the set of initial states from which $\mu$ attains the optimal gain in the original MDP $\M$, referred to below as the gain-optimal states for $\mu$. Whenever $D^*_{\mu}$ is nonempty, it is \emph{closed} under $\mu$. Otherwise, there would exist $s \in D^*_{\mu}$ such that
$g_\mu(s) = (P_\mu g_\mu)(s) < (P_\mu g^*)(s) \leq g^*(s),$
contradicting $g_{\mu}(s) = g^*(s)$.

Define $\eta_\beta(s) : = |v_\beta(s) - \tilde{h}^*(s) |$ for $s \in \S$; thus $\eta_\beta = \max_{s \in \S} \eta_\beta(s)$.
For $\mu \in \Pi_\A$ and $\hat \Pi \subset \Pi_\A$, define
\[ L_\mu(s) : =  ( |Z_{\mu}| \1)(s) \geq 1,  \qquad L_{\hat \Pi}(s) := \max_{\mu \in \hat \Pi} L_\mu(s).\]

\begin{proposition}[Bias bounds for gain-optimal states] \label{prop-p1p2-pol} \hfill
\begin{enumerate}[leftmargin=0.75cm,labelwidth=!]
\item[\rm (i)] For any policy $\mu \in \Pi_{\B_*}$ and all $s \in D^*_\mu$,
\begin{equation} \label{eq-b-p1-pol}
 h_{\mu}(s) \geq h^*(s) - \eta_\beta(s) - L_{\mu}(s) \epsilon_\beta.
\end{equation} 
\item[\rm (ii)]
For any policy $\bar \mu \in \Pi_{\tB}$ that is gain-optimal for $\M'$ and all $s \in D^*_{\bar \mu}$,
\begin{equation} \label{eq-b-p2-pol-a}
h_{\bar \mu}(s) \geq h^*(s) - \eta_\beta(s) - L_{\bar\mu}(s) (\epsilon_\beta + \bd_2),
\end{equation}
and also, with $\Pi_{\tB}^s: = \{ \mu \in \Pi_{\tB} : g_\mu(s) = g^*(s)\}$,
\begin{equation} \label{eq-b-p2-pol-b}
h_{\bar \mu}(s) \geq \max_{\mu \in \Pi_{\tB}^s} h_\mu(s) - 2 L_{\Pi_{\tB}^s}(s) (\epsilon_\beta + \bd_2).
\end{equation}
\end{enumerate}
\end{proposition} 

The next proposition provides a necessary and sufficient characterization of the condition $s \in D^*_{\bar \mu}$. In particular, it makes explicit several necessary conditions that may not be immediately obvious.

\begin{proposition} \label{prop-cond-Dst}
Let $\bar \mu \in \Pi_{\tB}$ be a gain-optimal policy for $\M'$. Then $s \in D^*_{\bar \mu}$ if and only if the following two conditions hold:
 \begin{enumerate}
\item[\rm (i)] $\bar{\mu}(s) \in \A_*(s)$, and 
\begin{equation} \label{eq-cond-D*1b}
 (P^*_{\bar \mu} v_\beta)(s) = \inf_{\mu \in \Pi_{\tB}} (P^*_{\mu} v_\beta)(s) = 0,
\end{equation} 
i.e., the optimal gain in $\M'$ starting from state $s$ is $0$.
\item[\rm (ii)] If $s \in \Rec_{\bar \mu}$, then $\bar \mu(s) \in \B_*(s)$. Otherwise,  $\bar \mu(s') \in \B_*(s')$ for every recurrent state $s' \in \Rec_{\bar \mu}$ that is reachable from $s$ under $\bar \mu$.
\end{enumerate}
Furthermore, \eqref{eq-cond-D*1b} implies the following necessary property:
\begin{enumerate}
\item[\rm (iii)] Every policy $\mu \in \Pi_{\B_*}$ is gain-optimal for state $s$ in the MDP $\widetilde{\M} = (\S, \tA, p, r - g^*)$, and therefore also in the original MDP whenever $\tA(s) = \A_*(s)$.
\end{enumerate} 
\end{proposition} 

The optimal gain function in $\M'=(\S, \tB, p, - v_\beta)$ is always nonnegative (see Lemma~\ref{lem-M'-gain} below). A larger outer approximation $\tB$ of $\B_*$ yields higher optimal gain values, which deviate further from the zero gain required by the necessary condition \eqref{eq-cond-D*1b}. The degree of departure from this condition translates directly into degradation in the gain of $\bar \mu$ for the original MDP $\M$ (cf.\ \eqref{eq-prf-pwb5b} below).

Thus, the preceding propositions appear to suggest that policies in $\Pi_{\B_*}$, being optimal for the $\beta$-discounted limiting MDP $\widetilde{\M}$ arising from Phase~1, may already suffice from the viewpoint of the preceding performance guarantees and can, in fact, have stronger theoretical guarantees than policies produced by Phase~2. In particular:

\begin{cor}
For a state $s \in \S$, suppose that $\tA(s) = \A_*(s)$ and \eqref{eq-cond-D*1b} holds. Then every policy $\mu \in \Pi_{\B_*}$ is gain-optimal from $s$ in $\M'$. Moreover, $\mu$ is also gain-optimal for $s$ in the original MDP so that its bias $h_\mu(s)$ can be bounded as in~\eqref{eq-b-p1-pol}. 
\end{cor}

In the remainder of this subsection, we further compare the policies produced by the two phases and discuss what Phase~2 can offer in practice, using an illustrative example.

\begin{remark}[Practical role of Phase~2]
Although all greedy policies w.r.t.\ $Q_{\beta,t}$ eventually lie in $\Pi_{\B_*}$ by the convergence of Phase~1 (Lemma~\ref{lem-stab-sets}), it could be numerically challenging to find a policy in $\Pi_{\B_*}$ via Phase~1 alone in the situation described below. Such a situation can arise, for example, in an important class of total-reward MDPs known as positive bounded models \cite[Sec.\ 7.2]{Put94}. 

Consider the pair of nested equations in \eqref{eq-aoe-bias} for finding a bias-optimal policy:
\[ h^* = \max_{\mu \in \Pi_{\A_*}} \{ r_\mu - g^* + P_\mu h^* \}, \qquad y = \max_{\mu \in \Pi_{\A_*[h^*]}} \{ - h^* + P_\mu y \}.\]
Suppose that the first equation alone cannot identify a bias-optimal policy: $\Pi_{\A_*[h^*]}$ consists of multiple gain-optimal policies, not all of which are bias-optimal. For simplicity, suppose that $\tA = \A_*$, as desired in Algorithm~\ref{alg:3}. Then, as we take $\beta$ close to $1$ so that $v_\beta$ approximates $h^*$ increasingly accurately, all policies $\mu$ in $\Pi_{\A_*[h^*]}$, including those that are not bias-optimal, approximately attain the statewise maxima in the $\beta$-discounted optimality equation for $\widetilde{\M}$, 
\[ v_\beta \approx r_\mu - g^* + \beta P_\mu v_\beta = Q_\beta \big(\cdot, \mu(\cdot)\big),\]
with the corresponding errors vanishing as $\beta \uparrow 1$. 

Consequently, when only the stochastic iterates $Q_{\beta, t}$ computed in Phase~1 are used, distinguishing non-bias-optimal policies from bias-optimal ones based on greediness w.r.t.\ $Q_{\beta, t}$ may require $Q_{\beta, t}$ to approximate its limit $Q_\beta$ extremely accurately. As $\beta \uparrow 1$, the number of iterations needed to attain such accuracy may become prohibitively large. Moreover, the required accuracy may itself be difficult to achieve reliably in practice.

These observations suggest that, although one could in theory rely solely on the discounted approximation in Phase~1 to find (near-)bias-optimal policies for $\M$ by taking $\beta$ close to $1$, Phase~2 may provide a less numerically demanding and more practical approach for computing such policies. 

As noted earlier, positive bounded models provide examples in which the numerical challenges of using Phase~1 alone can arise. To illustrate this, we run Algorithm~\ref{alg:3} on an example adapted from \citet[Ex.~7.2.3]{Put94}, with zero-mean random noise added to the rewards. The results in Figure~\ref{fig:alg3-phase-comparison}
support the observations above.
\hfill $\diamond$
\end{remark}

\begin{figure}[htbp]
    \centering
    \includegraphics[width=\linewidth]{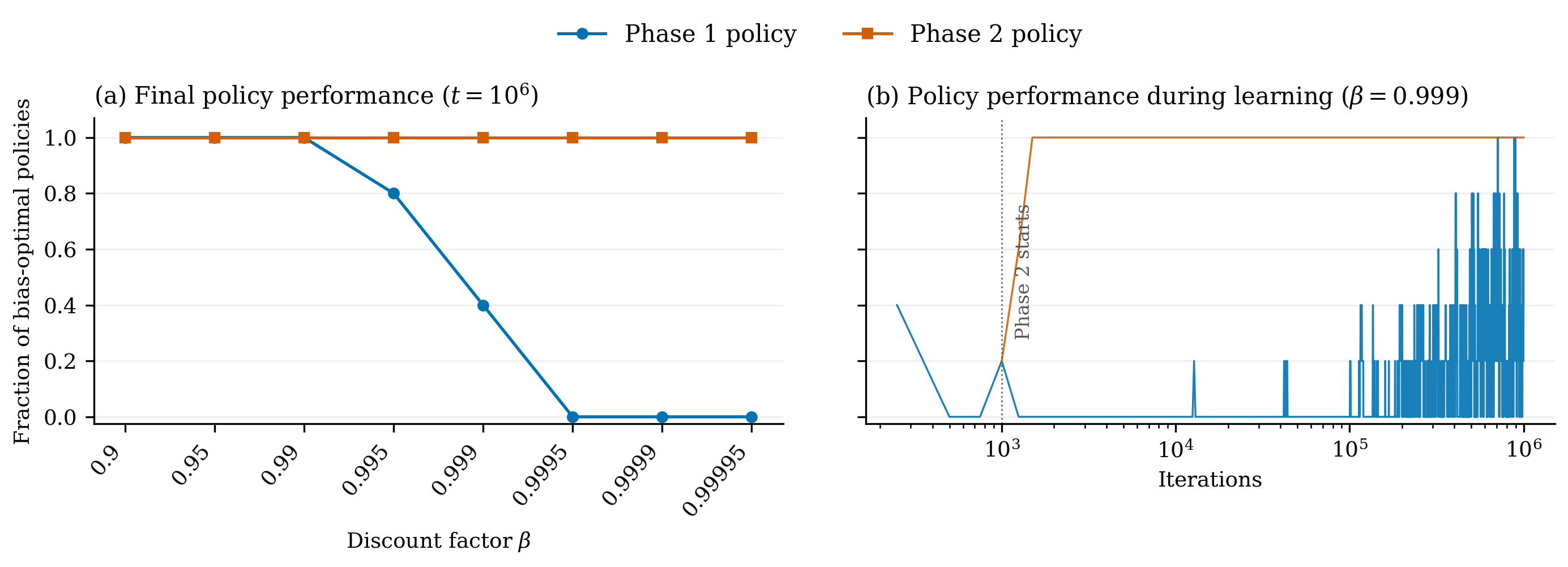}
    \caption{
        Comparison of the Phase 1 greedy policies and Phase 2 policies
        of Algorithm~\ref{alg:3}.
        Curves show the fraction of runs producing a bias-optimal policy
        across five random seeds.
        (a) Performance after $10^6$ iterations for different discount
        factors $\beta$.
        (b) Performance during learning for $\beta=0.999$.
    }
    \label{fig:alg3-phase-comparison}
\end{figure}

\begin{remark}
Finally, since the optimal gain in $\M'$ is always nonnegative, Prop.~\ref{prop-cond-Dst} suggests the possibility of using the optimal gain estimates for $\M'$ produced by the base algorithm in Phase 2 of Algorithm~\ref{alg:3} to adaptively tune the parameters $\bd_1$, $\bd_2$, and $\beta$. The development and convergence analysis of such an adaptive algorithm are beyond the scope of the present work and remain a direction for future research.
\end{remark}

\subsection{Performance-bound analysis} \label{sec-alg3-performance-proof}

\subsubsection{Proof of Theorem~\ref{thm-3-gen}}

\noindent {\bf (i) Bounding $g_\mu$ for $\mu \in \Pi_{\tB}$:}\\*[0.5em]
Let $\Delta_\mu$ be the nonnegative vector defined through the equality
\begin{equation} \label{eq-bg-ave-prf1}
  v_\beta = r_\mu - g^* + \beta P_\mu v_\beta + \Delta_\mu.
\end{equation}  
Since $\mu(s) \in \tB(s) \subset \B_{\bd_2}(s)$ for all $s \in \S$ (Lemma~\ref{lem-stab-sets}), we have
$\|\Delta_\mu \|_\infty \leq \bd_2.$
Left-multiplying \eqref{eq-bg-ave-prf1} by $P^*_\mu$ and rearranging the terms, we obtain
\begin{align} \label{eq-bg-ave-prf2}
P^*_\mu( r_\mu - g^*) & = ( 1- \beta) P^*_\mu v_\beta - P^*_\mu \Delta_\mu  
\geq - ( 1- \beta) \| v_\beta \|_\infty \1 - \|\Delta_\mu\|_\infty \1 \geq - (\epsilon_\beta + \bd_2) \1.
\end{align}

Next, we bound the term $P^*_\mu g^* - g^*$. For recurrent states under $\mu$,
\begin{equation} \label{eq-bg-ave-prf3}
(P^*_\mu g^*)(s) = g^*(s), \qquad \forall \, s \in \Rec_\mu.
\end{equation} 
For transient states under $\mu$, Lemma~\ref{lem-bound-gain} applies with $B = \T_\mu$, $B^c = \Rec_\mu$, and $\delta = \bd_1$, since $\mu \in \Pi_{\A_{\bd_1}}$ by $\tB \subset \tA \subset \A_{\bd_1}$ (Lemma~\ref{lem-stab-sets}). Hence 
\begin{equation} \label{eq-bg-ave-prf4}
 (P^*_\mu g^*)(s) - g^*(s)  \geq - \bd_1 \, \phi_\mu(s , \Rec_\mu), \quad \forall \, s \in \T_\mu.
\end{equation}

Since $g_\mu - g^* = P^*_\mu( r_\mu - g^*) + P^*_\mu g^* - g^*$, it follows from \eqref{eq-bg-ave-prf2}--\eqref{eq-bg-ave-prf4} that 
\[g_\mu(s) - g^*(s) \geq - \epsilon_\beta - \bd_2, \qquad \forall \, s \in \Rec_\mu,\]
\[g_\mu(s) - g^*(s) \geq - \epsilon_\beta - \bd_2 -  \bd_1 \phi_\mu(s , \Rec_\mu) , \qquad \forall \, s \in \T_\mu.\]
Since $\Pi_{\A}$ and $\S$ are finite, these bounds imply that for sufficiently small $\bd_1, \bd_2$, and $1- \beta$, $\mu$ is gain-optimal.
This proves part (i) of Theorem~\ref{thm-3-gen}. 

\medskip
\noindent {\bf (ii) Bias-optimality of $\bar \mu$ in a perturbed MDP:}\\*[0.5em]
For $\mu \in \Pi_{\tB}$, let $\Delta_{\beta,\mu} : = - P_\mu (1-\beta) v_\beta + \Delta_\mu$. Note that
$\|  \Delta_{\beta,\mu} \|_\infty \leq  \epsilon_\beta + \bd_2.$
Left-multiplying \eqref{eq-bg-ave-prf1} separately by $P^*_\mu$ and $Z_\mu$, and rearranging, we obtain
\begin{align} 
 P^*_\mu(r_\mu + \Delta_{\beta, \mu} - g^*) & = 0,   \label{eq-bg-prf1} \\
Z_\mu(r_\mu + \Delta_{\beta, \mu} - g^*) & = Z_\mu ( I - P_\mu ) v_\beta  = v_\beta - P^*_\mu v_\beta, \label{eq-bg-prf2} 
\end{align}
where we used \eqref{eq-Z-1} in \eqref{eq-bg-prf2}. Equations~\eqref{eq-bg-prf1}-- \eqref{eq-bg-prf2} show that, under the reward perturbation $\Delta_{\beta, \mu}$ and shift $- g^*$, the gain of $\mu$ is $0$ and its bias is $v_\beta - P^*_\mu v_\beta$.

Since $v_\beta$ is independent of $\mu$ and $- P^*_\mu v_\beta$ is the gain of $\mu$ in the MDP $\M' = (\S, \tB, p, - v_\beta)$, the gain-optimality of $\bar \mu$ for $\M'$ implies
\begin{equation} \label{eq-bg-prf3} 
  Z_{\bar \mu}(r_{\bar \mu} + \Delta_{\beta, \bar \mu} - g^*) \geq Z_{\mu}(r_{\mu} + \Delta_{\beta, \mu} - g^*), \qquad \forall \, \mu \in \Pi_{\tB}.
\end{equation}  
This shows that $\bar \mu$ is bias-optimal within $\Pi_{\tB}$ for a perturbed MDP whose rewards differ from those of the original MDP by the shift $-g^*$ and a perturbation uniformly bounded by $\epsilon_\beta + \bd_2$.

\medskip
\noindent {\bf (iii) Near bias-optimality of $\bar \mu$ in the original MDP:}\\*[0.5em]
Under the assumption in part (iii) that $\bar \mu$ is gain-optimal for $\M$, $D^*_{\bar \mu} = \S$. The assertion then follows directly from the statewise bias bound \eqref{eq-b-p2-pol-a} in Prop.~\ref{prop-p1p2-pol}(ii) (established below), upon noting that $\max_{s \in \S} \eta_\beta(s) = \eta_\beta$ and 
\[ \max_{s \in \S} L_{\bar \mu}(s) \leq \max_{\mu \in \Pi_\A} \big\| |Z_\mu| \1\big\|_\infty =:c,\]
where $c$ is a finite $\M$-dependent constant since $\Pi_\A$ is finite.

This completes the proof of Theorem~\ref{thm-3-gen}.

\subsubsection{Proofs of Propositions~\ref{prop-p1p2-pol} and~\ref{prop-cond-Dst}}

We begin with several useful relations. For any $\mu \in \Pi_\A$, recall that $g_\mu = P^*_\mu r_\mu$ and satisfies $g_\mu = P_\mu g_\mu = P^*_\mu g_\mu$, which implies $P^*_\mu(r_\mu - g_\mu) = 0$ and 
\begin{equation} \label{eq-bg-new-prf0}
 P^*_{\mu}(r_{\mu}  - g^*) = P^*_{\mu}(g_{\mu}  - g^* ) \leq 0.
\end{equation} 
Suppose $D^*_\mu$ is nonempty. Then, since $D^*_\mu$ is closed under $\mu$, for all $s \in D^*_{\mu}$, 
\begin{equation} \label{eq-bg-new-prf1a}
 [P^*_{\mu}(g_{\mu}  - g^* )](s) = [P^*_{\mu}(r_{\mu}  - g^*)](s) = 0,
\end{equation} 
\begin{equation} \label{eq-bg-new-prf1b}
[Z_{\mu}(g_{\mu}  - g^* )](s) = 0,
\end{equation}
\begin{equation} \label{eq-bg-new-prf2}
   h_\mu(s) := [ Z_{\mu}(r_{\mu}  - g_{\mu})](s) = [ Z_{\mu}(r_{\mu}  - g^*)](s),
\end{equation}
where \eqref{eq-bg-new-prf1b} follows from the expression $Z_\mu  = I + \lim_{\beta \uparrow 1} \sum_{n=1}^\infty \beta^n (P_\mu^n - P^*_\mu)$ and implies \eqref{eq-bg-new-prf2}.

For $\mu \in  \Pi_{\tB}$, left-multiplying \eqref{eq-bg-ave-prf1} by $P^*_{\mu}$ and $Z_\mu$, respectively, yields
\begin{equation} \label{eq-prf-pwb5b}
 (1 - \beta) P^*_\mu v_\beta = P^*_\mu(r_\mu - g^*) + P^*_\mu \Delta_{\mu},
\end{equation} 
\begin{equation} \label{eq-prf-pwb5}
 Z_{\mu}(r_{\mu} - g^*) = (I - P^*_{\mu}) v_\beta + (1 - \beta) Z_{\mu} P_{\mu} v_\beta - Z_{\mu} \Delta_{\mu},
\end{equation}
where we used $Z_\mu(I - P_\mu) = I - P^*_\mu$ in deriving \eqref{eq-prf-pwb5}.
If $\mu \in \Pi_{\B_*}$, then $\Delta_\mu = 0$ by the definition of $\B_*$ and we have
\begin{equation} \label{eq-prf-pwb0}
 v_\beta  = r_{\mu} - g^* + \beta P_{\mu} v_\beta,
\end{equation} 
\begin{equation} \label{eq-prf-pwb1}
 (1 - \beta) P^*_\mu v_\beta = P^*_\mu(r_\mu - g^*),
\end{equation}
\begin{equation} \label{eq-prf-pwb2}
 Z_\mu(r_{\mu} - g^*) = (I - P^*_\mu) v_\beta + (1 - \beta) Z_\mu P_\mu v_\beta.
\end{equation}

\begin{lemma} \label{lem-M'-gain}
The optimal gain function in $\M' = (\S, \tB, p, - v_\beta)$ is nonnegative. Equivalently, 
\[\inf_{\mu \in \Pi_{\tB}} P^*_\mu v_\beta \leq 0.\]
\end{lemma}

\begin{proof}
Since $\Pi_{\B_*} \subset \Pi_{\tB}$ and, by \eqref{eq-prf-pwb1}, 
\begin{equation} \label{eq-prf-M'-gain}
  P^*_\mu v_\beta = \tfrac{1}{1 - \beta} P^*_\mu(r_\mu - g^*) \leq 0, \qquad \mu \in \Pi_{\B_*},
\end{equation}   
the result follows.
\end{proof}

We now proceed to prove Props.~\ref{prop-p1p2-pol} and~\ref{prop-cond-Dst}.

\begin{proof}[Proof of Prop.~\ref{prop-p1p2-pol}]
(i) Let $\mu  \in \Pi_{\B_*}$ and $s \in D^*_\mu$. Then, by \eqref{eq-bg-new-prf1a} and \eqref{eq-prf-pwb1},
\begin{equation} \label{eq-prf-pwb3}
   (P^*_\mu v_\beta)(s) = 0.
\end{equation}
Combining this with \eqref{eq-bg-new-prf2} and \eqref{eq-prf-pwb2}, we obtain
\begin{align} 
h_\mu(s) & = v_\beta(s) +  [Z_\mu P_\mu (1 - \beta) v_\beta](s)  \notag \\
   & = \tilde{h}^*(s) + (v_\beta(s) - \tilde{h}^*(s)) + [Z_\mu P_\mu (1 - \beta) v_\beta](s) \notag \\
   & \geq h^*(s) - \eta_\beta(s) - L_{\mu}(s) \epsilon_\beta, \label{eq-prf-pwb4}
\end{align}   
where, in deriving \eqref{eq-prf-pwb4}, we used the fact that $\tilde{h}^* \geq h^*$, implied by $\tA \supset \A_*$ (cf.\ Lemma~\ref{lem-approx-hatM}), and also the fact that $\| P_\mu (1 - \beta) v_\beta \|_\infty \leq \| (1 - \beta) v_\beta \|_\infty = \epsilon_\beta$. This proves \eqref{eq-b-p1-pol}.

\medskip
\noindent (ii) Let $\bar \mu \in \Pi_{\tB}$ be gain-optimal for the MDP $\M' = (\S, \tB, p, - v_\beta)$. Then
\begin{equation} \label{eq-prf-bmu-opt}
 P^*_{\bar \mu} v_\beta \leq P^*_{\mu} v_\beta, \qquad \forall \, \mu \in \Pi_{\tB}.
\end{equation} 

We prove the bound \eqref{eq-b-p2-pol-b} first. Let $s \in D^*_{\bar \mu}$, and consider any $\mu \in \Pi_{\tB}^s$ so that $s \in D^*_{\mu}$. By \eqref{eq-bg-new-prf2} and \eqref{eq-prf-pwb5},
\begin{align*}
h_{\bar \mu}(s) & = v_\beta(s) - (P^*_{\bar \mu} v_\beta)(s)  +  [Z_{\bar \mu} P_{\bar \mu} (1 - \beta) v_\beta](s) -  (Z_{\bar \mu} \Delta_{\bar \mu})(s), \\
h_{\mu}(s) & = v_\beta(s) - (P^*_{\mu} v_\beta)(s)  +  [Z_{\mu} P_{\mu} (1 - \beta) v_\beta](s) -  (Z_{\mu} \Delta_{\mu})(s).
\end{align*}
Since $- (P^*_{\bar \mu} v_\beta)(s) \geq - (P^*_{\mu} v_\beta)(s)$, this implies
\begin{align*}
 h_{\bar \mu}(s)  \geq h_\mu(s) - [Z_{\mu} P_{\mu} (1 - \beta) v_\beta](s) +  (Z_{\mu} \Delta_{\mu})(s) + [Z_{\bar \mu} P_{\bar \mu} (1 - \beta) v_\beta](s) -  (Z_{\bar \mu} \Delta_{\bar \mu})(s).
 \end{align*}
The bound \eqref{eq-b-p2-pol-b} then follows by taking maximum over $\mu \in \Pi_{\tB}^s$:
\[ h_{\bar \mu}(s)  \geq \max_{ \mu \in \Pi_{\tB}^s} h_\mu(s) - 2 L_{\Pi_\tB^s}(s) (\epsilon_\beta + \bd_2).\]

Finally, we prove the bound \eqref{eq-b-p2-pol-a}. For $s \in D^*_{\bar \mu}$, by \eqref{eq-bg-new-prf1a} and \eqref{eq-prf-pwb5b},
\begin{equation} \label{eq-prf-bmu-0}
 (P^*_{\bar \mu} v_\beta)(s) = \tfrac{1}{1 - \beta} (P^*_{\bar \mu} \Delta_{\bar \mu})(s) \geq 0,
\end{equation} 
whereas $(P^*_{\bar \mu} v_\beta)(s) \leq 0$ by Lemma~\ref{lem-M'-gain}. 
Therefore,  
\begin{equation} \label{eq-prf-bmu-1}
 (P^*_{\bar \mu} v_\beta)(s) = 0. 
\end{equation} 
Combining this with \eqref{eq-bg-new-prf2} and \eqref{eq-prf-pwb5} for $\bar \mu$, it follows similarly to the derivation of \eqref{eq-prf-pwb4} that 
\begin{align*} 
h_{\bar \mu}(s) = [Z_{\bar \mu}(r_{\bar \mu} - g^*)](s) \geq h^*(s) - \eta_\beta(s) - L_{\bar \mu}(s) (\epsilon_\beta + \bd_2).
\end{align*}
This proves \eqref{eq-b-p2-pol-a}.
\end{proof}

\begin{proof}[Proof of Prop.~\ref{prop-cond-Dst}]
Let $s \in D^*_{\bar \mu}$. Then $\bar \mu(s) \in \A_*(s)$; otherwise, $(P_{\bar \mu} g^*)(s) < g^*(s)$, whereas $(P_{\bar \mu} g^*)(s) \geq (P_{\bar \mu} g_{\bar \mu})(s) = g_{\bar \mu}(s) = g^*(s)$, a contradiction. Equality \eqref{eq-cond-D*1b} has already been established (see \eqref{eq-prf-bmu-1}). This proves (i). Moreover, \eqref{eq-prf-bmu-0} and \eqref{eq-prf-bmu-1} yield
$(P^*_{\bar \mu} \Delta_{\bar \mu})(s) = 0,$
which is equivalent to (ii).

Conversely, under conditions (i) and (ii), we have 
\[ (P^*_{\bar \mu} v_\beta)(s) = 0, \quad (P^*_{\bar \mu} \Delta_{\bar \mu})(s) = 0, \quad \text{and } \ (P_{\bar \mu} g^*)(s) = g^*(s),\]
where the last equality follows from $\bar \mu(s) \in \A_*(s)$. Combining the first two equalities with \eqref{eq-prf-pwb5b} yields $[P^*_{\bar \mu} (r_{\bar \mu} - g^*)](s) = 0$, which together with the third equality implies $g_{\bar \mu}(s) = g^*(s)$. Hence, $s \in D^*_{\bar \mu}$, completing the proof of the equivalence.

For part (iii), suppose that \eqref{eq-cond-D*1b} holds: $(P^*_{\bar \mu} v_\beta)(s) = 0$. Then, for every $\mu \in \Pi_{\B_*}$, $(P^*_{\mu} v_\beta)(s) \geq (P^*_{\bar \mu} v_\beta)(s) =0$, whereas, by \eqref{eq-prf-pwb1}, 
$(1 - \beta) (P^*_{\mu} v_\beta)(s) = [P^*_{\mu} (r_\mu - g^*)](s) \leq 0$. Therefore,
\[(P^*_{\mu} v_\beta)(s) = 0 \quad \text{and} \ \  [P^*_{\mu} (r_\mu - g^*)](s) =  0,\]
which is equivalent to (iii). Thus, \eqref{eq-cond-D*1b} implies (iii).
\end{proof}

% Experimental results section
\section{Experimental results}
\label{sec:main-results}

\begin{figure}[t]
    \centering
    
    \begin{subfigure}[t]{0.48\linewidth}
        \centering
        \includegraphics[width=\linewidth]{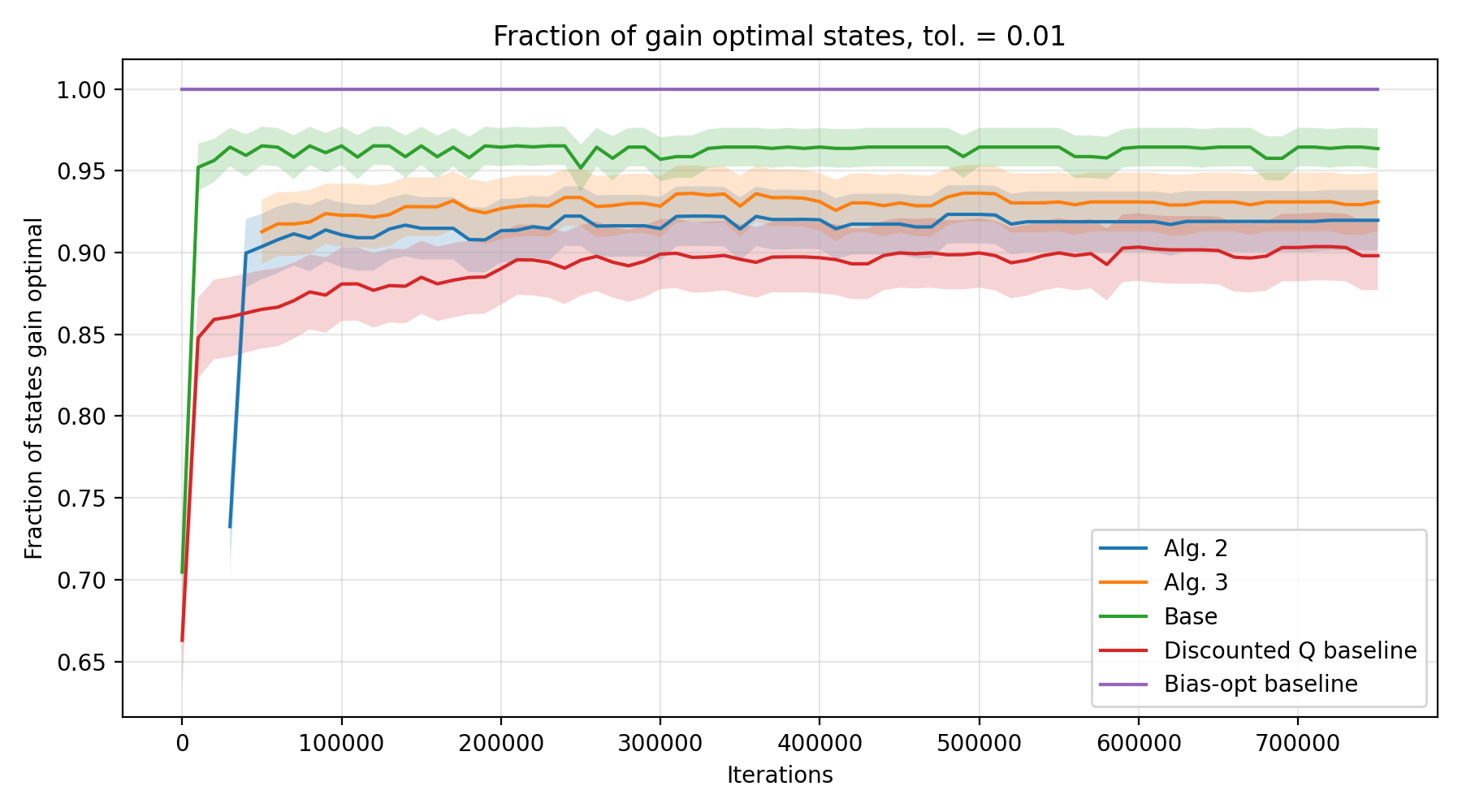}
        \caption{Fraction of gain-optimal states}
        \label{fig:gain}
    \end{subfigure}
    \hfill
    \begin{subfigure}[t]{0.48\linewidth}
        \centering
        \includegraphics[width=\linewidth]{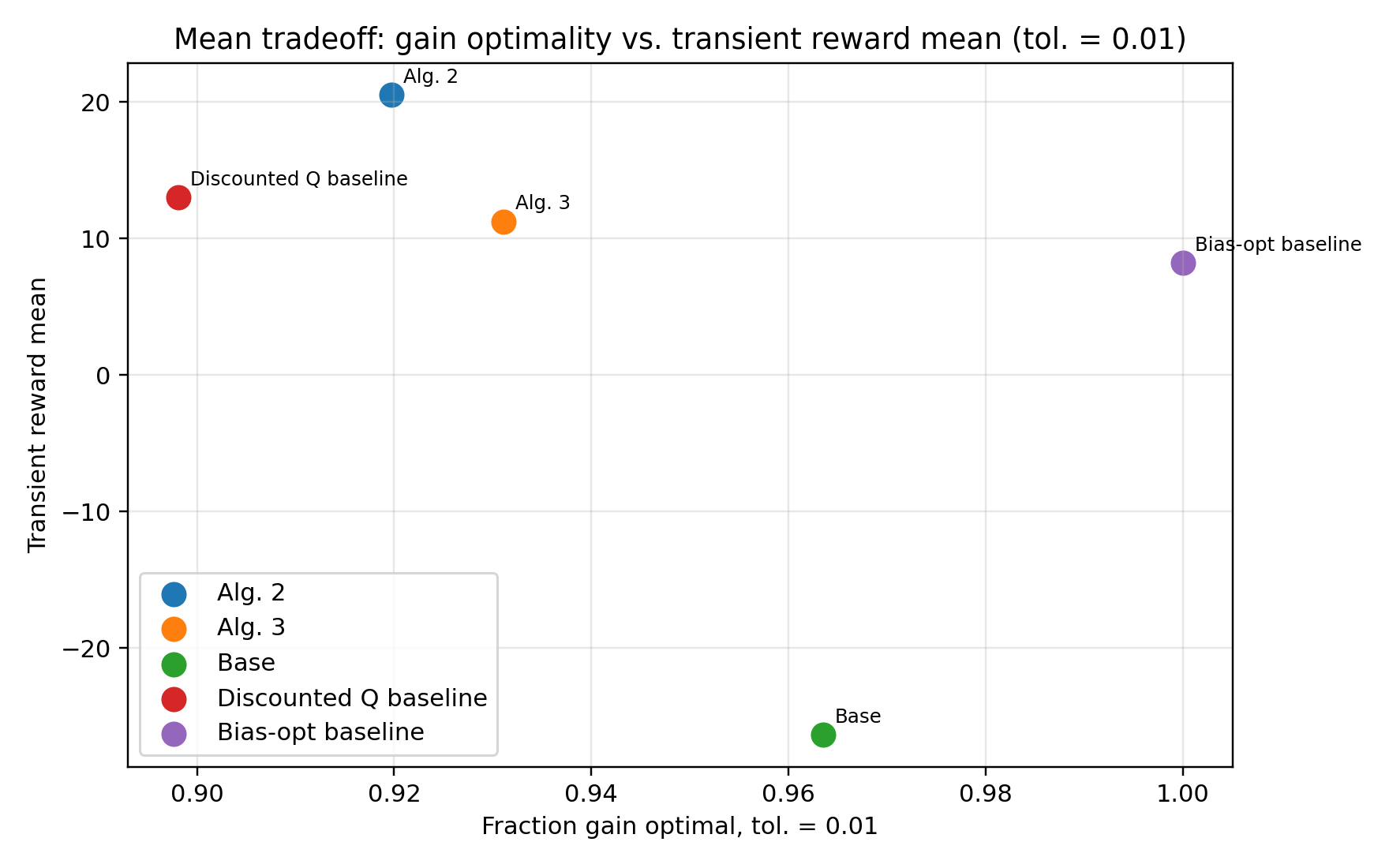}
        \caption{Gain-transient tradeoff}
        \label{fig:tradeoff}
    \end{subfigure}
    
    \caption{Comparison of gain optimality and transient performance across
    algorithms. In panel (a), the shaded region denotes one standard
    deviation across the 100 MDPs.}
    \label{fig:main_results}
\end{figure}

The following results are preliminary; a more extensive empirical evaluation is in progress and will be included in a future version of this work. We evaluate our algorithms on 100 randomly generated multichain MDPs with hierarchical structure, constructed using Bather's UHD framework. The number of levels, communicating classes, and transient states are randomized, and rewards and transitions are sampled to induce diverse gain structures and transient behavior. Full details are provided in Appendix~\ref{subsec:app-mdp-gen}.

We compare the base algorithm with Algorithms \ref{alg:2} and \ref{alg:3}, along with discounted Q-learning as a baseline. We run the $0$-discount policy iteration algorithm \cite{Vei66} (see also \cite[Sec.~10.3]{Put94}) using the true MDP model to compute a bias-optimal policy, which serves as an exact benchmark. For reproducibility, all algorithms are updated using a round-robin scheme, in which each outer iteration advances every phase by one full, deterministic sweep over its state-action pairs (full details in Appendix~\ref{subsec:app-alg-implementation-details}). We emphasize that our theoretical results hold under general asynchronous exploration schemes; round-robin updates were adopted here purely for implementation simplicity and ease of comparison across algorithms.

We say a state $s$ is \emph{gain-optimal under a policy $\pi$} if
$g_\pi(s) = g^*(s)$. Figure~\ref{fig:gain} shows, at each training iteration, the fraction of states that are gain-optimal under each algorithm's current policy, within a tolerance of $10^{-2}$ (i.e., $|g_{\pi^{(t)}}(s) - g^*(s)| < 10^{-2}$), averaged across the 100 MDPs, with the shaded band denoting one standard deviation across MDPs. This tolerance is chosen to account for numerical precision in our computation of $g_\pi$ and $g^*$ rather than as a substantive design choice; results are similar under smaller tolerances. The
base algorithm achieves the highest coverage, stabilizing near $96\%$. Algorithm~\ref{alg:3} follows closely, while Algorithm~\ref{alg:2} achieves slightly lower coverage. These results confirm that the base algorithm is most effective at finding gain-optimal policies, while extensions introduce a modest degradation in gain performance.

Comparing bias across policies with different gains
requires care. For this preliminary evaluation, we
compare transient performance using the expected
cumulative reward obtained prior to entering a recurrent class.
Figure~\ref{fig:tradeoff} reports this quantity for
each algorithm's final policy after $750{,}000$
iterations, averaged across all initial states within
each MDP and then across the 100 MDPs.
The $x$-axis gives the fraction of gain-optimal states,
and the $y$-axis measures this transient reward.
We leave a direct comparison of policy bias to future work.

The base algorithm achieves the highest gain-optimality but exhibits
significantly worse transient performance, stemming from the policy
$\pi_{\ell, k}^{\mathrm{reach}}$ for states in classes where exiting occurs. Algorithm~\ref{alg:2} achieves the best transient performance but with reduced gain-optimality. Algorithm~\ref{alg:3} provides a strong compromise: although its aggregate gain-optimality in Figure~\ref{fig:gain} sits between the base algorithm and Algorithm~\ref{alg:2}, its transient performance remains close to Algorithm~\ref{alg:2}'s and to the bias-optimal benchmark. Since gain-optimality is the primary objective in our setting, and transient performance is only meaningfully compared among near-gain-optimal policies, we view this as evidence that Algorithm~\ref{alg:3} achieves a favorable tradeoff between the two objectives.

To validate these aggregated results, we examine per-MDP performance (Appendix~\ref{subsec:app-extra-exp-results}). The trends remain consistent: the base algorithm is strongest in terms of gain-optimality but exhibits large transient failures, while the extensions enhance transient performance while maintaining a high degree of gain-optimality.

% Section 7: base algorithm convergence proofs
\section{Convergence proofs for the base algorithm}
\label{sec:conv-proofs-base-alg}

In this section, we prove the convergence results for
Algorithm~\ref{alg:mf-multichain}.
The aggregated OSP is already known to be a transient
total-reward problem, and its optimal values coincide
with the optimal gains of the original MDP.
These properties underlie the reduced formulation of
Schweitzer~\cite{Schweitzer1984} and the explicit OSP
formulation of Ohno~\cite{Ohno}.
We include proofs for completeness in the notation
of our formulation.

Section~\ref{subsec:supporting-lemmas-base} contains
the supporting results.
We first establish convergence of the Phase~2 and
Phase~3 learning procedures using existing results
for RVI Q-learning and total-reward Q-learning.
We then relate the optimal values and policy behavior
in the original MDP and the aggregated OSP.
Section~\ref{subsec:main-thms-base} combines these
results to prove convergence of the gain estimates
and eventual gain optimality of the constructed policies.

For the comparison between the original and aggregated problems, we use the following componentwise form of the AOE:

\begin{align}
g(s) &= \max_{a \in \mathcal{A}(s)} \sum_{s' \in \mathcal{S}} p(s' \mid s,a)\, g(s'), \quad \forall s \in \mathcal{S}, \label{eq:aoe1} \\
g(s) + h(s) &= \max_{a \in \mathcal{A}_*(s)} \left\{ r(s,a) + \sum_{s' \in \mathcal{S}} p(s' \mid s,a)\, h(s') \right\}, \quad \forall s \in \mathcal{S}, \label{eq:aoe2}
\end{align}
where
\[
\mathcal{A}_*(s) := \left\{ a \in \mathcal{A}(s) :
g(s) = \sum_{s' \in \mathcal{S}} p(s' \mid s,a)\, g(s')
\right\}.
\]
This system of equations is equivalent to the vector form given in \eqref{eq-aoe}.

\subsection{Supporting lemmas for base algorithm convergence}
\label{subsec:supporting-lemmas-base}

We first consider the learning procedures in Phases~2 and~3. Since each
restricted sub-MDP obtained from Bather's UHD is communicating, the
convergence theory for RVI Q-learning
\cite{WYS24,YWS25a,YWS25b} directly yields the Phase~2 result below.

\begin{lemma}[Phase~2 convergence of RVI Q-learning]
\label{lem:phase2_rvi_convergence}
Fix a communicating sub-MDP $\mathcal M_{\ell,k}$
obtained in Phase~1.
Suppose its Phase~2 learner uses a reference function
$f_{\ell,k}$ satisfying Assumption~\ref{cond-f},
a stepsize sequence satisfying Assumption~\ref{cond-ss},
and an update schedule satisfying
Assumption~\ref{cond-us}.
Then, almost surely:
\begin{enumerate}[label=(\roman*)]
\item
$r_{\ell,k}^{(t)}
=f_{\ell,k}(Q_{\ell,k}^{(t)})
\to r_{\ell,k}^*$, where $r_{\ell,k}^*$ is the
optimal gain of $\mathcal M_{\ell,k}$;
\item
for all sufficiently large $t$, every policy greedy
with respect to $Q_{\ell,k}^{(t)}$ is gain-optimal
for $\mathcal M_{\ell,k}$.
\end{enumerate}
\end{lemma}

\begin{proof}
By Bather's UHD, each $\mathcal M_{\ell,k}$ is communicating. The result follows directly from Theorem 3.1 in \cite{YWS25b}.
\end{proof}

We next consider Phase~3.
Properness of every stationary policy makes the aggregated
OSP a transient total-reward problem and provides the
structural condition used in the Q-learning convergence
argument.

\begin{lemma}[Properness of all stationary policies on the aggregated OSP]\label{lem:osp_properness}
Let $\bar{\mathcal M}=(\bar{\mathcal S},\bar{\mathcal A},\bar p,\bar r)$ be the limiting
aggregated OSP constructed in Phase~3 of Algorithm \ref{alg:mf-multichain}. Then every stationary policy $\bar{\pi}$ on $\bar{\mathcal{M}}$ is proper; that is, it reaches the terminal state $\Delta$ with probability 1 from every initial state. 
\end{lemma}

\begin{proof}

By Bather's UHD, all states in $\mathcal T$ are transient under every stationary
policy. Hence, starting from any $s\in\mathcal T$, the process reaches an aggregated state
$\bar s_{\ell,k}$ with probability $1$. Therefore, it suffices to prove properness when the process is started from an aggregated state.

Fix a stationary policy $\bar{\pi}$ on $\bar{\mathcal M}$ and suppose, for the sake of contradiction, that the Markov chain induced by $\bar{\pi}$ has a recurrent class $\bar R \neq \{\Delta\}$. Since stopping actions send the system to $\Delta$ with probability $1$, no state at level $0$ can belong to $\bar R$, as level $0$ states admit only the stopping action. Consequently, every state in $\bar R$ must have level $\ell \ge 1$, and the actions selected by $\bar{\pi}$ at such states must belong to the set of exit actions ${\mathcal E}_{\ell,k}$.

Let $\bar s_{\ell,k}^{\min}$ be a state of minimal level among the states in $\bar R$. By construction of the aggregated OSP, every action in ${\mathcal E}_{\ell,k}$ has positive probability of transitioning to a state at a strictly lower level. In particular, the action $\bar{\pi}(\bar s_{\ell,k}^{\min})$ leads to a lower level with positive probability. However, by the choice of $\bar s_{\ell,k}^{\min}$, no state of lower level belongs to $\bar R$. Hence, the process leaves $\bar R$ with positive probability, contradicting the assumption that $\bar R$ is a recurrent class.

Therefore, the induced Markov chain has no recurrent class other than $\{\Delta\}$. Equivalently, every stationary policy on $\bar{\mathcal M}$ is proper. \qedhere
\end{proof}

The stopping rewards used in Phase~3 are the gain estimates produced by Phase~2 and therefore vary during learning. By Lemma~\ref{lem:phase2_rvi_convergence}, their deviations from the limiting stopping rewards vanish almost surely. Together with the properness of all stationary policies established in Lemma~\ref{lem:osp_properness}, this allows us to apply Q-learning convergence results with vanishing reward perturbations.

\begin{lemma}[Phase~3 convergence of Q-learning]
\label{lem: phase3_conv}
Let $\bar{\mathcal M}$ be the limiting aggregated OSP
constructed in Phase~3 of Algorithm~\ref{alg:mf-multichain}.
Under the stepsize and asynchronous update conditions
of Assumption~\ref{cond-ssp-ql}, the total-reward
Q-learning iterates $\bar Q^{(t)}$ generated by the
sampling and update scheme in
Section~\ref{subsec:phase3} satisfy
\[
\bar Q^{(t)}\to\bar Q^*
\qquad\text{almost surely},
\]
where $\bar Q^*$ is the optimal action-value function
of $\bar{\mathcal M}$.
\end{lemma}

\begin{proof}

First note that the rewards $\bar r^{(t)}$ used in the aggregated OSP are time-varying because the stopping rewards are given by the estimates $r_{\ell,k}^{(t)}$ produced by RVI Q-learning in Phase~2 of Algorithm \ref{alg:mf-multichain}. By Lemma~\ref{lem:phase2_rvi_convergence}, these estimates converge almost surely to the optimal gains $r_{\ell,k}^*$. Hence, we may write
\[
\bar r^{(t)}
=
\bar r+ \varepsilon^{(t)},
\]
where $\bar r$ is the limiting reward function obtained by replacing each $r_{\ell,k}^{(t)}$ with $r_{\ell,k}^*$ and $\varepsilon^{(t)} \to 0$ almost surely. Thus, the Q-learning recursion is a vanishing perturbation of standard Q-learning. By Lemma \ref{lem:osp_properness}, all stationary policies on $\bar{\mathcal{M}}$ are proper. By robustness results for stochastic approximation (see Proposition~4.5 of \cite{BeT96}), the convergence guarantees of Proposition 5.5 part (a) of \cite{BeT96} hold. Thus, $\bar Q^{(t)} \to \bar Q^*$ almost surely. \qedhere
\end{proof}

We next relate the optimal values of the aggregated OSP
to the optimal gains of the original MDP, beginning
with the aggregated optimality equations.

\begin{lemma}[Optimality equations of the aggregated OSP]
\label{lem: solve_agg}
Let $\bar Q^*$ denote the optimal action-value function of the limiting aggregated OSP $\bar{\mathcal{M}}$, and define its optimal value function by 
\[
v^*(\bar{s}) = \max_{\bar{a}\in \bar{\mathcal{A}}(\bar{s})}{\bar Q^*(\bar{s}, \bar{a})}, \qquad \forall \bar{s} \in \bar{\mathcal{S}}.
\]
Then $\bar g^*:=v^*$ satisfies
$\bar g^*(\Delta)=0$ and the following optimality equations. For each class state $\bar s_{\ell,k}$,

\begin{equation}
\label{eq:agg-eq-class}
\bar{g}^*(\bar s_{\ell,k})
=
\max\Bigg\{
r^*_{\ell,k},\;
\max_{\bar a \in{\mathcal E}_{\ell,k}}
\Big[
\sum_{\ell',k'} \bar p(\bar s_{\ell',k'}\mid \bar s_{\ell,k},\bar a)\,\bar g^*(\bar s_{\ell',k'})
+
\sum_{s\in\mathcal T} \bar p(s\mid \bar s_{\ell,k},\bar a)\,\bar g^*(s)
\Big]
\Bigg\}.
\end{equation}
When $\mathcal E_{\ell,k}=\varnothing$, the exit
term in~\eqref{eq:agg-eq-class} is omitted, so
$\bar g^*(\bar s_{\ell,k})=r_{\ell,k}^*$.
For retained transient states,
\begin{equation}
\label{eq:agg-eq-transient}
\bar g^*(s)
=
\max_{a\in \mathcal{A}(s)}
\Bigg\{
\sum_{\ell,k} \bar p(\bar s_{\ell,k}\mid s,a)\,\bar{g}^*(\bar s_{\ell,k})
+
\sum_{s'\in\mathcal T} \bar p(s'\mid s, a)\,\bar g^*(s')
\Bigg\},
\qquad s \in \mathcal{T}.
\end{equation}
\end{lemma}
\begin{proof}

For all $\bar{s} \in \bar{\mathcal{S}}$, let
\begin{align}
v^*(\bar{s})
&= \max_{\bar{a}\in \bar{\mathcal{A}}(\bar{s})}{\bar Q^*(\bar{s}, \bar{a})} \nonumber\\
&= \max_{\bar{a}\in \bar{\mathcal{A}}(\bar{s})}
\sum_{\bar{s}'} \bar{p}(\bar{s}' \mid \bar{s}, \bar{a})
\big(\bar{r}(\bar{s}, \bar{a}) + v^*(\bar{s}')\big).
\label{eq:v_star}
\end{align}

We now show that setting $\bar{g}^*(\bar{s}) = v^*(\bar{s})$ solves the optimality equations \eqref{eq:agg-eq-class} and~\eqref{eq:agg-eq-transient} by specializing to the two types of non-terminal states in $\bar{\mathcal{S}}$: aggregated and transient states. Note that $v^*(\Delta) = 0$ by definition.

\emph{(i) Aggregated states $\bar s_{\ell,k}$.} In these states, there are two types of admissible actions: the stopping action $a_{\ell, k}^{stop}$ and exit actions $\bar{a} \in {\mathcal{E}}_{\ell,k}$. If $a_{\ell, k}^{stop}$ is selected, then $\bar p(\Delta \mid \bar s_{\ell,k},a_{\ell, k}^{stop})=1$ and
\(\bar r(\bar s_{\ell,k}, a_{\ell, k}^{stop})=r^*_{\ell,k}\),
thus,
\[
\sum_{\bar{s}'} \bar p(\bar{s}' \mid \bar s_{\ell,k},a_{\ell,k}^{stop})\big(\bar r(\bar s_{\ell,k},a_{\ell, k}^{stop})+v^*(\bar{s}')\big)
=
r^*_{\ell,k} + v^*(\Delta)
=
r^*_{\ell,k}.
\]

If instead $\bar{a} \in {\mathcal{E}}_{\ell,k}$ is selected, then $\bar r(\bar s_{\ell,k},\bar a)=0$
for all $\bar{s}$ and transitions can only lead to aggregated states $\bar s_{\ell',k'}$
or transient states $s\in\mathcal T$. Hence,
\[
\sum_{\bar{s}'} \bar p(\bar{s}' \mid \bar s_{\ell,k},\bar a)\big(\bar r(\bar s_{\ell,k},\bar a)+v^*(\bar{s}')\big)
=
\sum_{\ell',k'} \bar p(\bar s_{\ell',k'}\mid \bar s_{\ell,k},\bar a)v^*(\bar s_{\ell',k'})
+\sum_{s\in\mathcal T}\bar p(s\mid \bar s_{\ell,k},\bar a)v^*(s).
\]
Thus, maximizing over these two scenarios yields \eqref{eq:agg-eq-class}.

\emph{(ii) Transient states $s \in \mathcal{T}$.} Here, $\bar{\mathcal{A}}(s) = \mathcal{A}(s)$ and all one-step rewards are zero by definition. Hence, equation \eqref{eq:v_star} reduces to
\[
v^*(s)=\max_{a\in\mathcal A(s)} \sum_{\bar{s}\in\bar{\mathcal S}} \bar p(\bar{s} \mid s,a)v^*(\bar{s}),
\]
and splitting the sum over $\bar{s}$ into aggregated and transient states yields \eqref{eq:agg-eq-transient}. \qedhere
    
\end{proof}

We now relate the optimal OSP values to the original
optimal gains.
Recall the aggregation map $\phi$ from
Section~\ref{subsec:phase3}:
\[
\phi(s)=
\begin{cases}
\bar s_{\ell,k}, & s\in C_{\ell,k},\\
s, & s\in\mathcal T.
\end{cases}
\]
The following lemma shows that the optimal gain at
each original state equals the optimal OSP value
at its corresponding aggregated state.

\begin{lemma}[Equivalence of optimal gains and OSP values]
\label{lem:agg-upper-bound}
Let $g^*$ be the optimal gain vector of $\mathcal M$,
and let $\bar g^*$ be the optimal value function
of the limiting aggregated OSP $\bar{\mathcal M}$.
Then
\[
g^*(s)=\bar g^*(\phi(s)),
\qquad \forall s\in\mathcal S.
\]
\end{lemma}

\begin{proof}

First note that for every recurrent class $R$ induced by a policy in the original MDP, there exists a unique class $C_{\ell, k}$ such that $R \subseteq C_{\ell, k}$. Moreover, because every $C_{\ell, k}$ is communicating, from any state $s \in C_{\ell, k}$, there exists a policy that can reach any other state $s' \in C_{\ell, k}$. This implies that the optimal gain from $s$ must equal the optimal gain from $s'$. Otherwise, a higher gain from one state could be leveraged to improve the other, leading to a contradiction. Thus, we have that $g^*(s) = g^*(s') \; \forall \; s, s' \in C_{\ell, k}$ \cite{Schweitzer1984}.

We therefore define a lifted function $\tilde{g}^*$ on $\bar{\mathcal{S}}$ such that 
\begin{align*}
& \tilde{g}^*(\bar{s}_{\ell, k}) \coloneq g^*(s), \qquad s \in C_{\ell, k}, \\
& \tilde{g}^*(s) \coloneqq g^*(s), \qquad s \in \mathcal{T}.
\end{align*}

Let $\pi^*$ be an optimal policy of $\mathcal{M}$. Fix a class $C_{\ell, k}$ and consider two cases based on the behavior of $\pi^*$ within it.

\medskip
\noindent
\textbf{Case 1:} $\pi^*$ never exits $C_{\ell,k}$.  

Then $\pi^*(s) \in \mathcal{A}_{\ell,k}(s)$ for all $s \in C_{\ell,k}$.
By Lemma \ref{lem:phase2_rvi_convergence}, $r^*_{\ell,k}$ is the maximal gain
attainable under such policies. Therefore,
\[
g^*(s) = r^*_{\ell,k},
\qquad \forall s \in C_{\ell,k}.
\]

\medskip
\noindent
\textbf{Case 2:} $\pi^*$ exits $C_{\ell,k}$. 

Then there exists $s_0 \in C_{\ell,k}$ such that
$\pi^*(s_0) \in \mathcal{E}_{\ell,k}(s_0)$.
Since $\pi^*$ is optimal, its gain vector satisfies the AOE \eqref{eq:aoe1} - \eqref{eq:aoe2}. Hence, we must have
\[
g^*(s_0)
= \max_{a \in \mathcal{E}_{\ell, k}(s_0)}
\sum_{s \in \mathcal S} p(s \mid s_0, a)\, g^*(s).
\]

Given that optimal gains are constant within $C_{\ell, k}$, this maximum extends uniformly across all states in the class. Thus, for any $s \in C_{\ell, k}$,

\[
g^*(s)
= g^*(s_0) =
\max_{(s,a) \in {\mathcal{E}}_{\ell,k}}
\sum_{s' \in \mathcal S} p(s' \mid s, a)\, g^*(s').
\]

Combining cases 1 and 2, we have that for all $s \in C_{\ell, k}$,

\begin{align*}
g^*(s)
&=
\max\Bigg\{r^*_{\ell, k}, \max_{(s,a) \in {\mathcal{E}}_{\ell,k}}
\sum_{s' \in \mathcal S} p(s' \mid s, a)\, g^*(s') \Bigg\}.
\end{align*}

Now, splitting the inner sum into classes and transient states we have

\begin{align*}
\sum_{s' \in S} p(s' \mid s, a)\, g^*(s') &= \sum_{{\ell', k'}}{\sum_{s' \in C_{\ell', k'}}} p(s' \mid s, a)\, g^*(s') + \sum_{s' \in \mathcal{T}}{p(s' \mid s, a)\, g^*(s')} \\
&= \sum_{{\ell', k'}} \bar{p}(\bar{s}_{\ell', k'} \mid \bar{s}_{\ell, k}, \bar{a})\, \tilde{g}^*(\bar{s}_{\ell', k'}) + \sum_{s' \in \mathcal{T}}{\bar{p}(s' \mid \bar{s}_{\ell, k}, \bar{a})\, \tilde{g}^*(s')},
\end{align*}

where the last equality follows from the definition of $\bar{p}$ and $\tilde{g}^*$.

Thus, for every aggregated state $\bar{s}_{\ell, k}$,

\begin{equation}
\label{eq: tildeg-class}
\tilde{g}^*(\bar{s}_{\ell, k})
=
\max\Bigg\{r^*_{\ell, k}, \max_{\bar{a} \in {\mathcal{E}}_{\ell,k}}
\Big[
\sum_{\ell',k'} \bar p(\bar s_{\ell',k'}\mid \bar s_{\ell,k},\bar a)\,\tilde{g}^*(\bar s_{\ell',k'})
+
\sum_{s\in\mathcal T} \bar p(s\mid \bar s_{\ell,k},\bar a)\,\tilde{g}^*(s)
\Big] \Bigg\}.
\end{equation}

This is precisely equation \eqref{eq:agg-eq-class}.

For any transient state $s \in \mathcal{T}$, the AOE \eqref{eq:aoe1} - \eqref{eq:aoe2} in the original MDP gives
\[
g^*(s)=
\max_{a \in \mathcal{A}(s)}
\sum_{s' \in S} p(s' \mid s, a)\, g^*(s').
\]

Splitting the sum and using the definitions of $\bar{p}$ and $\tilde{g}^*$ we have

\begin{equation}
\label{eq: tildeg-transient}
\tilde{g}^*(s) = \max_{a \in \mathcal{A}(s)} 
\Big[
\sum_{{\ell, k}} \bar{p}(\bar{s}_{\ell, k} \mid s, a)\, \tilde{g}^*(\bar{s}_{\ell, k}) + \sum_{s' \in \mathcal{T}}{\bar{p}(s' \mid s, a)\, \tilde{g}^*(s')}
\Big],
\end{equation}

which is equation \eqref{eq:agg-eq-transient}.

Equations \eqref{eq: tildeg-class} and \eqref{eq: tildeg-transient} show that $\tilde{g}^*$ satisfies the optimality equations \eqref{eq:agg-eq-class}--\eqref{eq:agg-eq-transient} of $\bar{\mathcal{M}}$. By Lemma \ref{lem: solve_agg}, the optimal value function $\bar{g}^*$ of the aggregated OSP $\bar{\mathcal{M}}$ also solves this system. Therefore, 
\[
\tilde{g}^* = \bar{g}^*.
\]

By the definition of $\tilde{g}^*$, this implies

\[
g^*(s) = \bar g^*(\bar s_{\ell,k}), \quad \forall s \in C_{\ell,k},
\]
\[
g^*(s) = \bar g^*(s), \quad \forall s \in \mathcal T. \qedhere
\]

\end{proof}

It remains to relate an aggregated policy to its
Phase~4 lifted policy.
The following lemma shows that the probability of
eventually remaining in a class under the lifted
policy equals the probability of stopping at the
corresponding aggregated state.
When the Phase~2 sub-MDP policies are optimal, this
correspondence equates the gain of the lifted policy
with the total-reward value of the aggregated policy.

\begin{lemma}[Preservation of stopping-class distributions]
\label{lem:equiv-class}
Fix an iteration $t$, and let $\pi^{(t)}$ be the
Phase~4 lifted policy corresponding to $\bar\pi^{(t)}$.
Suppose every reaching policy used in this construction
is class-preserving and reaches its selected exit state
almost surely from every state in the corresponding class.
Then, for every initial state $s\in\mathcal S$
and every class $C_{\ell,k}$,
\[
\mathbb P_{\pi^{(t)}}\!
\left(
\text{absorbed in }C_{\ell,k}
\mid s_0=s
\right)
=
\mathbb P_{\bar\pi^{(t)}}\!
\left(
\text{stops at }\bar s_{\ell,k}
\mid \bar s_0=\phi(s)
\right).
\]
\end{lemma}

\begin{proof}
Let $(s_q)_{q\ge0}$ and $(\bar s_q)_{q\ge0}$ be the Markov chains under the fixed policies $\pi^{(t)}$ and $\bar\pi^{(t)}$ in the original MDP and the aggregated OSP, respectively, with $\bar s_0=\phi(s_0)$.
We construct a sub-process of $(s_q)_{q\ge0}$ by observing it only at times when the aggregated state $\phi(s_q)$ can change. Formally, define $\tau_0 = 0$ and $\tau_{q+1}$ for $q \ge 0$ recursively as follows:
\begin{itemize}
    \item If $s_{\tau_q} \in \mathcal{T}$, set $\tau_{q+1} = \tau_q + 1$.
    \item If $s_{\tau_q} \in C_{\ell,k}$ and $\bar{\pi}^{(t)}(\bar{s}_{\ell,k}) = (s^*, a^*)$,
    let $\sigma_q > \tau_q$ be the first time after $\tau_q$ at which $\pi^{\mathrm{reach}}_{\ell,k}$
    hits $s^*$ starting from $s_{\tau_q}$, and set $\tau_{q+1} = \sigma_q + 1$, where action $a^*$ is taken at time
    $\sigma_q$.
    \item If $s_{\tau_q} \in C_{\ell,k}$ and $\bar{\pi}^{(t)}(\bar{s}_{\ell,k}) =
    a^{\mathrm{stop}}_{\ell,k}$, the sub-process stops and the chain is absorbed
    in $C_{\ell,k}$.
\end{itemize}
By construction, the sub-process $(\phi(s_{\tau_q}))_{q\ge0}$ evolves only when the aggregated state can change. A direct calculation (using the definition of $\bar p$ in Section~\ref{subsec:phase3}) shows that for every aggregated state $\bar s$ and action $\bar a$, the one-step transition probabilities of $\phi(s_{\tau_{q+1}})$ under $\pi^{(t)}$ coincide with those of $\bar s_{q+1}$ under $\bar\pi^{(t)}$, conditional on $\bar s_q = \phi(s_{\tau_q})$.

Since $\bar s_0 = \phi(s_0)$, the one-step transition distributions coincide, and the evolution depends only on the current aggregated state, it follows that $(\phi(s_{\tau_q}))_{q\ge0}$ and $(\bar s_q)_{q\ge0}$ have the same law. In particular, the probability that the sub-process stops in $C_{\ell,k}$ equals the probability that the aggregated process stops at $\bar s_{\ell,k}$.

Finally, by construction, absorption of the original chain $(s_q)_{q\ge0}$ in $C_{\ell,k}$ is equivalent to the sub-process stopping in $C_{\ell,k}$. This yields
\[
\mathbb P_{\pi^{(t)}}\!
\left(
\text{absorbed in }C_{\ell,k}
\mid s_0=s
\right)
=
\mathbb P_{\bar\pi^{(t)}}\!
\left(
\text{stops at }\bar s_{\ell,k}
\mid \bar s_0=\phi(s)
\right),
\]
as claimed.
\end{proof}

\subsection{Proofs of the main convergence results for the base algorithm}
\label{subsec:main-thms-base}

We now combine the supporting lemmas to prove
Theorems~\ref{thm:gain_identification1}
and~\ref{thm:policy_attains_gstar1}. Theorem~\ref{thm:gain_identification1} asserts that the
Phase~3 Q-values converge to the optimal action-value
function of the limiting aggregated OSP and that the
resulting gain estimates converge to $g^*$.

\begin{proof}[Proof of Theorem~\ref{thm:gain_identification1}]
By Lemma~\ref{lem:phase2_rvi_convergence}, for every class $C_{\ell,k}$, Phase~2 RVI Q-learning
converges almost surely and produces the optimal gain $r^*_{\ell,k}$ when actions are restricted to $\mathcal{A}_{\ell,k}$.

Consider the limiting aggregated OSP $\bar{\mathcal M}$ constructed in Phase~3.
By Lemma~\ref{lem: phase3_conv}, Q-learning on $\bar{\mathcal M}$
converges almost surely to the optimal action-value function $\bar Q^*$. Define
\[
\bar{g}^*(\bar s)=\max_{\bar a\in\bar{\mathcal A}(\bar s)} \bar Q^*(\bar s,\bar a), \qquad \bar s\in\bar{\mathcal S}.
\]
By Lemma~\ref{lem: solve_agg}, the function
$\bar g^*$ satisfies the optimality equations \eqref{eq:agg-eq-class}--\eqref{eq:agg-eq-transient} of $\bar{\mathcal{M}}$. Now define $g:\mathcal S\to\mathbb R$ by lifting $\bar g^*$ back to the original state space: $g(s)=\bar g^*(s)$ for $s\in\mathcal T$ and $g(s)=\bar g^*(\bar s_{\ell,k})$ for $s\in C_{\ell,k}$. Then by Lemma~\ref{lem:agg-upper-bound}, we have $g = g^*$. Consequently, by the definition of $g^{(t)}$ and
the convergence $\bar Q^{(t)}\to\bar Q^*$,
we obtain $g^{(t)}\to g^*$ almost surely. \qedhere
\end{proof}

We next prove Theorem~\ref{thm:policy_attains_gstar1},
which establishes that, almost surely, every policy constructed in Phase~4 is gain-optimal after some finite time.

\begin{proof}[Proof of Theorem~\ref{thm:policy_attains_gstar1}]
By Bather's UHD, each $C_{\ell, k}$ is communicating under the restricted action set $\mathcal{A}_{\ell, k}$. Hence, each restricted MDP $\mathcal{M}_{\ell, k}$ is communicating. By Lemma \ref{lem:phase2_rvi_convergence}, for each $(\ell, k)$, there exists a finite time $T_{\ell, k}$ such that for all $t \ge T_{\ell, k}$ the policy $\pi_{\ell, k}^{(t)}$ is optimal for $\mathcal{M}_{\ell, k}$. In particular,
\[
g_{\pi_{\ell, k}^{(t)}}(s) = r^*_{\ell, k}, \qquad \forall s \in C_{\ell, k}, \;t \ge T_{\ell, k}.
\]

By Lemma \ref{lem: phase3_conv}, total-reward Q-learning on the aggregated OSP $\bar{\mathcal{M}}$ converges almost surely. As Q-learning converges, the estimates $\bar{Q}^{(t)}$ approach the optimal action-values $\bar{Q}^*$. Because the state and action spaces are finite, the difference between optimal and suboptimal action values at each state is bounded away from zero. Once the estimation error becomes smaller than this separation, the action that maximizes $\bar{Q}^{(t)}$ must coincide with an action that maximizes $\bar{Q}^*$. Thus, after some finite time $\bar{T}$, we have that for all $t \ge \bar{T}$, the policy $\bar{\pi}^{(t)}$ is optimal for $\bar{\mathcal{M}}$. Hence, 
\[
\bar{g}_{\bar{\pi}^{(t)}}(\bar{s}) = \bar{g}^*(\bar{s}), \qquad \forall \bar{s} \in \bar{\mathcal{S}}, \;t \ge \bar{T}.
\]

Thus, by Lemma \ref{lem:agg-upper-bound}, the optimal gain vector of the original MDP $\mathcal{M}$ satisfies

\begin{equation}
\label{eq: thm2-g*}
g^*(s)=
\begin{cases}
\bar{g}_{\bar{\pi}^{(t)}}(s), & s\in\mathcal T,\\[2pt]
\bar{g}_{\bar{\pi}^{(t)}}(\bar s_{\ell,k}), & s\in C_{\ell,k},
\end{cases}
\qquad \forall t \ge \bar{T}.
\end{equation}

Define 
\[
T \coloneq \max{\big( \bar{T}, \max_{\ell, k}{T_{\ell, k}} \big)}.
\]
Since this is the maximum of finitely many almost surely finite random variables, $T < \infty$ almost surely.

Fix any $t \ge T$. We show that the policy $\pi^{(t)}$ constructed by Algorithm~\ref{alg:mf-multichain} attains $g^*$ in \eqref{eq: thm2-g*}.

In a finite MDP, the gain of a stationary policy is determined by the distribution of possible recurrent classes eventually entered under that policy; transient states contribute only finitely many rewards almost surely and therefore do not affect the long run gain.

All states in $\mathcal{T}$ are transient under every stationary 
policy by Bather's UHD. Since $\bar{\pi}^{(t)}$ is proper on $\bar{\mathcal{M}}$ 
by Lemma~\ref{lem:osp_properness}, all policies reach a state $\bar{s}_{\ell, k}$ where the system stops with probability $1$. Since the absorption and stopping 
distributions coincide by Lemma~\ref{lem:equiv-class}, the original chain under 
$\pi^{(t)}$ is also absorbed with probability $1$. Therefore, all non-stopping 
class states are transient under $\pi^{(t)}$. At stopping classes $C_{\ell,k}$, 
$\pi^{(t)}$ coincides with $\pi^{(t)}_{\ell,k}$, which achieves gain $r^*_{\ell,k}$ 
for all $s \in C_{\ell,k}$ by Lemma~\ref{lem:phase2_rvi_convergence}. Thus, for any $s \in \mathcal{S}$
\begin{equation}
\label{eq:g-sum}
    g_{\pi^{(t)}}(s) = \sum_{(\ell, k)}{\mathbb{P}_{\pi^{(t)}}\!\left(\text{absorbed in } C_{\ell,k} \mid s_0 = s\right)}r^*_{\ell, k},
\end{equation}
where ${\mathbb{P}_{\pi^{(t)}}\!\left(\text{absorbed in } C_{\ell,k} \mid s_0 = s\right)}$ is the probability of being absorbed into $C_{\ell, k}$ when starting from $s_0$ under $\pi^{(t)}$.

In the aggregated OSP, all rewards are zero except for stopping actions in states $\bar{s}_{\ell, k}$ which receive reward $r^*_{\ell, k}$.  Thus, for any $\bar s \in \bar{\mathcal{S}}$
\begin{equation}
\label{eq:bar-g-sum}
    \bar{g}_{\bar{\pi}^{(t)}}(\bar s) = \sum_{(\ell, k)}{\mathbb{P}_{\bar\pi^{(t)}}\!\left(\text{stop in } \bar{s}_{\ell,k} \mid \bar{s}_0 = \bar{s}\right)}r^*_{\ell, k},
\end{equation}
where ${\mathbb{P}_{\bar\pi^{(t)}}\!\left(\text{stop in } \bar{s}_{\ell,k}\mid \bar{s}_0 = \bar{s}\right)}$ is the probability of stopping in $\bar{s}_{\ell, k}$ when starting from $\bar{s}_0$ under $\bar{\pi}^{(t)}$.

By Lemma~\ref{lem:equiv-class}, for any $s_0 \in \mathcal{S}$ and all $(\ell,k)$,
\[
    \mathbb{P}_{\pi^{(t)}}\!\left(\text{absorbed in } C_{\ell,k} \mid s_0\right)
    = \mathbb{P}_{\bar{\pi}^{(t)}}\!\left(\text{stop in } \bar{s}_{\ell,k} 
    \mid \bar{s}_0 = \phi(s_0)\right).
\]

Substituting into equations \eqref{eq:g-sum} -- \eqref{eq:bar-g-sum}, we have that for all $s \in \mathcal{S}$

\begin{equation*}
g_{\pi^{(t)}}(s)=
\begin{cases}
\bar{g}_{\bar{\pi}^{(t)}}(s), & s\in\mathcal T,\\[2pt]
\bar{g}_{\bar{\pi}^{(t)}}(\bar s_{\ell,k}), & s\in C_{\ell,k},
\end{cases}
\qquad \forall t \ge T.
\end{equation*}
Hence, $\pi^{(t)}$ attains the optimal gain vector $g^*$ for all $t \ge T$. \qedhere

\end{proof}

The theorem does not require the policy sequence
$\{\pi^{(t)}\}$ to stabilize. It asserts only that, after an almost
surely finite time, every policy produced by the algorithm is
gain-optimal.

\section{Convergence proofs for Algorithms~\ref{alg:2} and~\ref{alg:3}} 
\label{sec-conv-proofs}

In this section, we prove the convergence theorems for Algorithms~\ref{alg:2} and~\ref{alg:3}. The proof of Theorem~\ref{thm-1} for Algorithm~\ref{alg:2} is given in Section~\ref{appsec-alg2-prf-conv}. The proofs of Lemma~\ref{lem-stab-sets} and Theorem~\ref{thm-b-conv} for Algorithm~\ref{alg:3} are given in Section~\ref{appsec-alg3-prf-conv}.

To begin with, we restate the information structure condition \ref{cond:C2} of Algorithm~\ref{alg:2} in a more formal manner and introduce several conventions to simplify the presentation of the main convergence proof in the next subsection.

\subsection{Formal statement of assumptions and conventions} \label{appsec:alg2-1c}

We first fix a few conventions that impose no restriction on the algorithm and are adopted without loss of generality. Although the algorithm starts at time $\bar t$, we extend the definitions of $Q_t$ and $Y_t$ to all $t \geq 0$ by setting $Q_t$ to a fixed value and $Y_t = \varnothing$ for $t < \bar t$. We define $\alpha_t(s, a)$ for all state-action pairs and all $t \geq 0$, with $\alpha_t(s,a) = 0$ whenever $(s, a) \not\in Y_t$. Under this convention, $Y_t = \{(s,a) : \alpha_t(s,a) \not=0\}$, so specifying $\alpha_t$ uniquely determines $Y_t$. Finally, we define the transition data $(R^{sa}_{t}, S^{sa}_{t})$ for all state-action pairs $(s,a)$ and all $t \geq 0$, even though some of them are not used in the algorithm.

Assumption \ref{cond-meas-caus} below summarizes a set of non-restrictive regularity conditions. These conditions set up a basic measurability structure for treating the randomized algorithm probabilistically and reaffirm its non-anticipatory nature. 

Throughout, all spaces are assumed to be standard Borel spaces (e.g., $[0,1]$) with their Borel $\sigma$-algebras, and all measurability is understood in the Borel sense. While the conditions could be formulated more generally, we restrict attention to Borel spaces and Borel measurability for clarity and to avoid unnecessary technicalities.

\begin{assumption}[Measurability, non-anticipation, and correct conditional distributions] \label{cond-meas-caus} \hfill 
\begin{itemize}[leftmargin=0.4cm,labelwidth=!] 
\item {\rm \bfseries Initialization:} The starting time $\bar t$ of Algorithm~\ref{alg:2} is an a.s.\ finite stopping time relative to the histories of Algorithm~\ref{alg:mf-multichain}. The parameters $(\epsilon, \ud, \bd)$ and the initial $Q_{\bar t}$ are measurable functions of the history of Algorithm~\ref{alg:mf-multichain} up to time $\bar t$. 
\item {\rm \bfseries Stochastic evolution:} All other random elements in Algorithm~\ref{alg:2}, namely $\alpha_t$ and $(R^{sa}_{t}, S^{sa}_{t})_{(s,a) \in \A}$, admit measurable and non-anticipatory realizations driven by exogenous noises. Specifically, there exist i.i.d.\ exogenous noise sequences $\{w_t\}_{t \geq 0}$ and $\{w'_t\}_{t \geq 0}$, taking values in standard Borel spaces, such that:
\begin{itemize}[leftmargin=0.75cm,labelwidth=!] 
\item[\rm (i)] For each $t \geq 0$, 
$\alpha_t = \Psi_t\big(w_t, (g_n, v_n, Q_n)_{n \leq t}, (\alpha_n)_{n < t}, (R^{sa}_{n}, S^{sa}_{n})_{(s, a) \in \A, n \leq t} \big)$,
for some Borel-measurable map $\Psi_t$. 
\item[\rm (ii)] For each $t \geq 0$, 
$(R^{sa}_{t}, S^{sa}_{t})_{(s,a) \in \A} = \Psi'_t\big(w'_t, (g_n, v_n, Q_n, \alpha_n)_{n  < t}, (R^{sa}_{n}, S^{sa}_{n})_{(s, a) \in \A, n < t} \big)$,
for some Borel-measurable map $\Psi'_t$. 
\end{itemize}
Moreover, the following conditional distribution requirement holds:
\begin{itemize}[leftmargin=0.75cm,labelwidth=!]
\item[\rm (iii)] For each $t \geq 0$ and each $(s, a) \in \A$, for any given realization of the past outcomes, the distribution of $(R^{sa}_{t}, S^{sa}_{t})$ induced by the fresh noise $w'_t$ through the map $\Psi'_t$ coincides with the reward and transition distribution of the MDP at state $s$ under action $a$.
\end{itemize}
\end{itemize}
\end{assumption}

%\smallskip
\begin{remark} \rm
In Assumption~\ref{cond-meas-caus}(ii,iii), we do not require the transition data $(R^{sa}_{t}, S^{sa}_{t})_{(s,a) \in \A}$ to be mutually independent across different state-action pairs $(s,a)$. These conditions accommodate situations in which one exploits the MDP structure to improve the efficiency of Q-learning by using a shared random outcome to generate transition data for multiple state-action pairs involved in the Q-learning updates.
\end{remark}

We also restate condition \ref{cond:C3} of Algorithm~\ref{alg:2} under the conventions introduced above:

\begin{assumption}[Adequate transition data and stepsize conditions] \label{cond-data-stepsize} 
The following hold a.s.:
\begin{itemize}[leftmargin=0.55cm,labelwidth=!] 
\item[\rm (i)] Any state-action pair that occurs in $\tG_t$ infinitely often also occurs in $Y_t$ infinitely often.
\item[\rm (ii)] For any $(s, a)$ that appears infinitely often in $Y_t$, the stepsizes satisfy 
\[\sum_{t} \alpha_t(s,a) = \infty, \qquad \sum_{t} \alpha^2_t(s,a) < \infty.\] 
\end{itemize}
\end{assumption}

\subsection{Algorithm~\ref{alg:2}: proof of Theorem~\ref{thm-1}} \label{appsec-alg2-prf-conv}

Lemma~\ref{lem-1} establishes finite-time stabilization of the sets constructed in Algorithm~\ref{alg:2}. Thus, the limiting sets $\tS^i$ ($i\in\{\mathrm{r},\mathrm{e},\mathrm{b}\}$), $\tA$, and $\tG$ appearing in Theorem~\ref{thm-1}, as well as the associated partition $\C=\tCr\cup\tCe\cup\tCb$, are all well-defined.

Theorem~\ref{thm-1} concerns the convergence of the Q-learning iterates in Algorithm~\ref{alg:2} and properties of their limits. We prove it in three steps:
\begin{itemize}[leftmargin=0.75cm,labelwidth=!]
\item[(i)] First, we extend existing stability and convergence results for Q-learning in SSP-type total-reward MDPs (\cite{Tsi94,YuB13} and \cite[Sec.\ 4.3]{BeT96}).
\item[(ii)] Using this extension, we establish the convergence of a Q-learning algorithm corresponding to a limiting form of Algorithm~\ref{alg:2} that solves a given SSP-type MDP. 
\item[(iii)] We then tackle Algorithm~\ref{alg:2} in its original form, establish the a.s.\ convergence of its Q-learning iterates to solutions of random SSP-type MDPs induced by the outputs of Algorithm~\ref{alg:mf-multichain}.  
\end{itemize}

\subsubsection{Step 1: extending Q-learning results for SSP problems} \label{sec-ssp-analysis-extension}
We now begin the proof by presenting an extended convergence theorem for Q-learning in SSP problems, building on prior work. Let $\hat \M := \big\{\hat \S \cup \{\Delta\}, \hat \A, \hat p, \hat r \big\}$ by an SSP-type total-reward MDP as defined in Section~\ref{sec-background-ssp}. We consider Q-learning iterates for solving $\hat \M$, generated from a given initial condition $\hat Q_0$ and for all $(s,a) \in \hat \A$, by 
\begin{align} 
  \hat Q_{t+1}(s, a)   = \hat Q_t(s,a) + \hat \alpha_t(s,a) \Big( & \hat r(s,a) + \omega_{t+1}(s,a) + \delta_{t+1}(s,a) \notag \\
  &  + \ind\{\hat{S}^{sa}_{t+1} \not= \Delta \} \cdot \max_{a' \in \hat \A(\hat{S}^{sa}_{t+1})} \hat Q_{t}(\hat{S}^{sa}_{t+1}, a') - \hat Q_t(s,a) \Big). \label{eq-gen-ql-ssp}
\end{align}
Here, the term $\hat r(s,a) + \omega_{t+1}(s,a) + \delta_{t+1}(s,a)$ represents a random reward deviating from the expected reward $\hat r(s,a)$, and the stepsizes $\hat \alpha_t(s,a)$ at iteration $t$ implicitly determine which components of $\hat Q_t$ are updated: if $\hat \alpha_t(s, a) = 0$, then $\hat Q_{t+1}(s, a)   = \hat Q_t(s,a)$. 

Moreover, there is an associated increasing sequence of $\sigma$-algebras $\F_t$, together with which, the random variables involved in \eqref{eq-gen-ql-ssp} satisfy the properties listed below:
\begin{itemize}
\item[(A1)] $\F_t \supset \sigma \big(\hat Q_0, \hat \alpha_k(s,a), \omega_k(s,a), \delta_k(s,a), \hat{S}^{sa}_k; \, k \leq t, \, (s, a) \in \hat \A \big)$.
\end{itemize}
For each $(s,a) \in \hat \A$:
\begin{itemize}
\item[(A2)] $\sum_t \hat \alpha_t(s,a) = \infty$ and $\sum_t \hat \alpha^2_t(s,a) < \infty$ a.s.
\item[(A3)] For all $t \geq 0$, $\E [ \, \omega_{t+1}(s,a) \mid \F_t] =  0$ and $ \E \big[ \omega^2_{t+1}(s,a) \mid \F_t\big] < c$ a.s., where $c$ is a deterministic constant.
\item[(A4)] As $t \to \infty$, $\delta_{t+1}(s,a) \to 0$ a.s.
\item[(A5)] Conditional on $\F_t$, $\hat{S}^{sa}_{t+1}$ follows the transition probability distribution $\hat p(\cdot \mid s, a)$ of $\hat \M$.
\end{itemize}

Iteration \eqref{eq-gen-ql-ssp} differs from the standard Q-learning iteration for $\hat \M$ only in that the reward includes an additional term $\delta_{t+1}(s,a)$, which vanishes by (A4) but need not have conditional zero mean given $\F_t$. Because of these \emph{biased} diminishing noise terms $\delta_{t+1}$, we extend prior results in \cite{Tsi94,YuB13} on the boundedness and convergence of Q-learning in SSP-type MDPs to establish convergence of \eqref{eq-gen-ql-ssp} (see also \cite[Sec.\ 4.3]{BeT96} for a presentation of the results in \cite{Tsi94} in a more general formulation, albeit without communication delays). 

We give the extended theorem below, along with an outline of the proof highlighting the minor modifications needed to make the prior results applicable here. We omit the full proof to avoid repeating the lengthy and involved arguments in those works.

\begin{proposition} \label{prop-4.1}
For the SSP-type MDP $\hat \M$, the sequence $\{\hat Q_t\}$ generated by \eqref{eq-gen-ql-ssp} under (A1)--(A5) converges a.s.\ to $\hat{Q}^*$, the optimal action-value function of $\hat \M$.
\end{proposition}

\begin{proof}[Proof (outline)]
We first show how to modify prior results to establish convergence of $\{\hat Q_t\}$, assuming its a.s.\ boundedness. We apply the convergence proof of \cite[Prop.\ 4.6]{BeT96} (which assumes a.s.\ boundedness), together with the following extension to account for the additional terms $\delta_{t+1}$ not considered in the original proof.  Define
\begin{equation} \label{eq-prop4.1-prf1}
  Z_{t+1}(s,a) = (1 - \hat \alpha_t(s,a)) Z_t(s,a) +  \hat \alpha_t(s,a) \delta_{t+1}(s,a),
\end{equation}  
with $Z_0(\cdot) \equiv 0$. Since $\delta_t \to 0$ a.s.\ by (A4) and $\{\hat \alpha_t\}$ satisfies (A3), it follows that $Z_t \to 0$ a.s. We follow the same proof steps as in \cite[Prop.~4.6, pp.~168--171]{BeT96} until near the end. There, we modify \cite[Lem.\ 4.7, p.\ 170]{BeT96} by adding the term $Z_t(s,a)$ to its upper bound on $\hat Q_{t+1}(s,a)$, using the same reasoning as in the original proof. Using this modified lemma and the fact that $Z_t \to 0$ a.s., the final step \cite[Lem.~4.8, pp.~170--171]{BeT96} continues to hold for $\{\hat Q_t\}$, completing the convergence proof under the boundedness condition.

For the boundedness of $\{\hat Q_t\}$, we explain how the boundedness proof in \cite{YuB13} extends to our setting. Following the analysis in \cite{YuB13}, we show separately that $\{\hat Q_t\}$ is bounded above a.s.\ and bounded below a.s. The two parts rely on substantially different arguments. 

For boundedness from above, we apply the reasoning in the proof of \cite[Prop.\ 3.1]{YuB13}, where we bound $\{\hat Q_t\}$ from above by an auxiliary sequence of Q-learning iterates constructed using a proper policy and argue that the auxiliary sequence is bounded a.s. The proof of \cite[Prop.\ 3.1]{YuB13} carries over to the present setting once we note that, under a proper policy, the mapping underlying Q-learning is a contraction w.r.t.\ a weighted sup-norm. Consequently, by \cite[Prop.\ 4.5, pp.~157--158]{BeT96}, even in the presence of the additional diminishing biased noise terms $\{\delta_{t}\}$, the auxiliary iterates remain bounded a.s. Therefore, $\{\hat Q_t\}$ is bounded above a.s.

For boundedness from below, the proof of \cite[Prop.~3.3, Section~3.3]{YuB13} carries over to the present setting. While the proof involves the construction of auxiliary processes and other arguments, which we do not repeat here due to their complexity and length, it does not depend critically on the particular form of the random one-stage rewards appearing in Q-learning. Instead, it only requires the following estimates $\{\hat r_t(s,a)\}$, defined by the random rewards, to converge to the expected one-stage reward $\hat r(s,a)$ for each $(s,a)$: 
\[ \hat r_{t+1}(s,a) = (1 - \hat \alpha_t(s,a)) \hat r_{t}(s,a) + \hat \alpha_t(s,a) (\hat r(s,a) + \omega_{t+1}(s,a) + \delta_{t+1}(s,a)).\]
By standard arguments under (A2)--(A4),%footnote starts
\footnote{Express $\hat r_{t+1}(s,a)$ as the sum of three iterates: $\hat r_{t+1}(s,a) = X_{t+1}(s,a) + W_{t+1}(s,a) + Z_{t+1}(s,a)$, where 
\begin{align*}
X_{t+1}(s,a) & = (1 - \hat \alpha_t(s,a)) X_{t}(s,a) + \hat \alpha_t(s,a) \hat r(s,a), \quad \text{with} \  X_0(s,a) = \hat r_0(s,a); \\
W_{t+1}(s,a) & = (1 - \hat \alpha_t(s,a)) W_{t}(s,a) + \hat \alpha_t(s,a) \omega_{t+1}(s,a), \quad \text{with} \ W_0(s,a) = 0;
\end{align*}
and $Z_{t+1}(s,a)$ is given by \eqref{eq-prop4.1-prf1}. By the stepsize condition (A3), $X_t(s,a) \to \hat r(s,a)$ a.s. By \cite[Cor.~4.1, p.\ 161]{BeT96}, $W_t(s,a) \to 0$ a.s.\ under (A2) and (A3). As shown earlier, $Z_t(s,a) \to 0$ a.s.\ under (A2) and (A4). It then follows that $\hat r_{t+1}(s,a) \to \hat r(s,a)$ a.s.}
%footnote ends,
$\hat r_{t+1}(s,a) \to \hat r(s,a)$ holds a.s.\ despite the additional diminishing noise terms $\delta_{t+1}$. Consequently, \cite[Prop.~3.3]{YuB13} extends to our setting and implies that $\{\hat Q_t\}$ is bounded below a.s.
Algorithm~\ref{alg:2}
Combining the two parts yields the a.s.\ boundedness of $\{\hat Q_t\}$.
\end{proof}

\subsubsection{Step 2: convergence of a limiting form of Algorithm~\ref{alg:2}}

We now return to the original MDP $\M$ and consider a limiting form of Algorithm~\ref{alg:2}. 
In particular, fix arbitrary sets satisfying the limiting properties \eqref{eq-lem1-p1}--\eqref{eq-lem1-p4} in Lemma~\ref{lem-1}: a partition of $\C$ into $\tCr, \tCe, \tCb$, the associated state sets $\tS^i := \bigcup_{C \in \tC^i} C$, and action constraint set $\tA$. Define 
\[ \tG := \{ (s,a) \in \tA \mid s \not\in \tSr \}. \]
Let $\{g_t\} \subset \R^{|\S|}$ and $\{v_t\} \subset \R^{|\tSr|}$ be two convergent sequences with $g_t \to g^*$ and $v_t \to v_*$. All the objects just introduced are treated as given and nonrandom---an assumption that we will relax later.

Given initial $\hat Q_0$, consider Q-learning iterates defined for all $(s,a) \in \tG$ by
\begin{align}
  \hat Q_{t+1}(s, a)  = \hat Q_t(s,a) + \hat \alpha_t(s,a) & \Big(R^{sa}_{t+1} - g_t(s) - \epsilon \, \ind\big\{ s \in \tSb \big\}  + \ind \big\{S^{sa}_{t+1} \in \tSr\big\} \cdot v_t(S^{sa}_{t+1}) \notag \\
  & \ + \ind\big\{S^{sa}_{t+1} \not\in \tSr\big\} \cdot\!\! \max_{a' \in \tA(S^{sa}_{t+1})} \hat Q_t(S^{sa}_{t+1}, a') \, - \hat Q_t(s,a) \Big), \label{eq-prf-ql-ite}
\end{align} 
together with these associated conditions:
\begin{itemize}[leftmargin=0.75cm,labelwidth=!]
\item[(i)] There is an increasing sequence of $\sigma$-algebras $\F_t$ with 
\[\F_t \supset \sigma \big(\hat Q_0, \hat \alpha_k(s,a), R^{sa}_{k}, S^{sa}_k; \, k \leq t, \, (s, a) \in \tG \big).\]
\item[(ii)] The stepsizes $\hat \alpha_t(s,a)$ satisfy (A2).
\item[(iii)] Conditional on $\F_t$, $R^{sa}_{t+1}$ and $S^{sa}_{t+1}$ are distributed according to the reward and transition distributions of $\M$ at state $s$ under action $a$.
\end{itemize}

The next lemma establishes the convergence of $\{\hat Q_t\}$ to a limiting vector $\hat Q^*$ satisfying the same equation as \eqref{eq-thm1-Q-ssp-opt} in Theorem~\ref{thm-1}, which concerns the limits of the Q-learning iterates of Algorithm~\ref{alg:2}.

\begin{lemma} \label{lem-4.1}
In the above setting, the sequence $\{\hat Q_t\}$ generated by \eqref{eq-prf-ql-ite} converges a.s.\ to the unique solution $\hat Q^*$ of the following equation: for all $(s, a) \in \tG$, 
\begin{equation} \label{eq-Q-ssp-opt}
   \hat Q^*(s,a) = r(s,a) - g^*(s) - \epsilon \1_{\tSb}(s) + \sum_{s' \in \tSr} p(s' \mid s, a) v_*(s') + \sum_{s' \in \S \setminus \tSr} p(s' \mid s, a)  \max_{a' \in \tA(s')} \hat Q^*(s',a').
\end{equation}
Moreover, this equation is the optimality equation of an SSP-type total-reward MDP defined on $\tG$, for which $\hat Q^*$ is the optimal action-value function.
\end{lemma}

\begin{proof}
We view $\{\hat Q_t\}$ as Q-learning iterates for an SSP-type MDP $\hat \M =  \big\{\hat \S \cup \{\Delta\}, \hat \A, \hat p, \hat r \big\}$ defined below. This MDP $\hat \M$ is similar to the MDP $\M^{\epsilon}$ introduced before Lemma~\ref{lem-aux-idealized-alg} for the idealized algorithm in Section~\ref{sec:alg2-1a}, except that $\hat \M$ is defined using $\tSr$, $\tSb$, $\tCe$, $\tCb$, and $\tA$ in place of $\Sr$, $\Sb$, $\Ce$, $\Cb$, and $\A_*$, respectively. In particular, we aggregate the states in $\tSr$ into a special termination state $\Delta$, and take $\hat \S = \S \setminus \tSr$ as the effective state space. Let $\hat \A(s) = \tA(s)$ for $s \in \hat \S$; thus, $\hat \A = \tG$.
For each feasible state-action pair $(s,a) \in \tG$, the transition probabilities $\hat p(\cdot \mid s, a)$ are given by 
\begin{equation} \label{eq-prf-lem4.1-1}
\hat p(s' \mid s, a) = p(s' \mid s, a), \ \ \ s' \in \hat \S, \qquad  \hat p(\Delta \mid s, a) = \sum_{s' \in \tSr} p(s' \mid s, a),
\end{equation}
and the expected one-stage reward $\hat r(s,a)$ is given by
\begin{equation} \label{eq-prf-lem4.1-2}
\hat r(s,a) = r(s,a) - g^*(s) - \epsilon \1_{\tSb}(s) + \sum_{s' \in \tSr} p(s' \mid s, a) \, v_*(s').
\end{equation} 
(Compare \eqref{eq-def-Meps-1}--\eqref{eq-def-Meps-2} for $\M^{\epsilon}$.)
 
That $\hat \M$ is of SSP type follows from the same arguments as in the proof of Lemma~\ref{lem-aux-idealized-alg}, with the set replacements noted above. In particular, by~\eqref{eq-lem1-p4}, for each $C_{\ell,k} \in \tCe \cup \tCb$, there exists at least one state $s_{\ell,k} \in C_{\ell,k}$ with $\tA(s_{\ell,k}) \cap \cE(s_{\ell,k}) \not= \varnothing$. Hence a proper policy for $\hat \M$ can be constructed as in the proof of Lemma~\ref{lem-aux-idealized-alg}. Since $\tCe \subset \Ce$ and $\tCb \supset \Cb$ by \eqref{eq-lem1-p2}, the same reasoning as in the second half of the proof of Lemma~\ref{lem-aux-idealized-alg} shows that any improper stationary policy in $\Pi_{\hat \A}$ has expected total reward $-\infty$ in $\hat \M$ on each recurrent class it induces within the effective state space. 

We can express the iterates $\hat Q_t$ from \eqref{eq-prf-ql-ite} in the general form \eqref{eq-gen-ql-ssp} of Q-learning for $\hat \M$ as follows. 
Take the $\sigma$-algebras associated with \eqref{eq-gen-ql-ssp} to be the $\sigma$-algebras $\F_t$ associated with iterates from \eqref{eq-prf-ql-ite}. 
Define the noise terms $\omega_{t+1}(s,a)$ and $\delta_{t+1}(s,a)$ in \eqref{eq-gen-ql-ssp} as
\begin{align}
  \omega_{t+1}(s,a) & =  R^{sa}_{t+1} - r(s,a) +  \ind\{S^{sa}_{t+1} \in \tSr\} \cdot v_*(S^{sa}_{t+1}) -   \sum_{s' \in \tSr} p(s' \mid s, a) v_*(s'),  \\
\delta_{t+1}(s,a) & = g^*(s) - g_t(s) +  \ind\{S^{sa}_{t+1} \in \tSr\} \cdot \big(v_t(S^{sa}_{t+1}) -  v_*(S^{sa}_{t+1})  \big).
\end{align} 
Recall that $v_*$ lies in a given bounded set, $g_t \to g^*$, and $v_t \to v_*$. Consequently, $\omega_{t+1}$ and $\delta_{t+1}$ satisfy (A3) and (A4). Define the random states $\hat S^{sa}_{t+1}$ in \eqref{eq-gen-ql-ssp} as
\begin{equation}
  \hat S^{sa}_{t+1} = \begin{cases}
   S^{sa}_{t+1} & \text{if} \ S^{sa}_{t+1} \not\in \tSr; \\
   \Delta & \text{otherwise}.
   \end{cases}
\end{equation}
With this definition, (A5) holds. The conditions in Prop.~\ref{prop-4.1} are therefore all satisfied.

Then, by Prop.~\ref{prop-4.1}, $\{\hat Q_t\}$ converges a.s.\ to the optimal action-value vector $\hat Q^*$ of $\hat \M$. Finally, by Theorem~\ref{thm:ssp-opt} for SSP-type MDPs, $\hat Q^*$ is the unique solution to the action-value optimality equation \eqref{eq-Q-ssp-opt} for $\hat \M$.
\end{proof}

\subsubsection{Step 3: convergence of Algorithm~\ref{alg:2} in original form} \label{appsec:alg2-conv-prf-step3}

We apply Lemma~\ref{lem-4.1} to establish Theorem~\ref{thm-1}. However, because the outputs of Algorithm~\ref{alg:mf-multichain} are random, the parameter $\epsilon$ and the limiting sets $\tC^i, \tS^i$, and $\tA$ in Theorem~\ref{thm-1} are also random, preventing a direct application of Lemma~\ref{lem-4.1}. Instead, we start with the joint probability measure for Algorithm~\ref{alg:mf-multichain} and Algorithm~\ref{alg:2}, and then work with the associated conditional distributions in explicit form, as shown below.  

First, we define the probability space $(\Omega, \Sigma_\Omega, \P)$ for analyzing Algorithm~\ref{alg:2}.  Let $(\Xi, \Sigma_{\Xi}, \P^b)$ be the probability space associated with Algorithm~\ref{alg:mf-multichain}. By Assumption~\ref{cond-meas-caus}, at initialization of Algorithm~\ref{alg:2},  the starting time $\bar t$, the parameters $(\epsilon, \ud, \bd)$, and the initial $Q_{\bar t}$ are measurable functions of $\xi \in \Xi$. We take the sample space for Algorithm~\ref{alg:2} to be $\Omega := \Xi \times \Z$, where $\Z$ consists of sequences $\zeta := (\zeta_t)_{t \geq 0}$, with each $\zeta_t$ representing a realization of the collection of random variables 
\[\big\{ Q_t(s,a), \, \alpha_t(s,a), \, (R^{sa}_{t}, S^{sa}_t) \mid (s, a) \in \A \big\}.\]
We write $\omega = (\xi, \zeta)$ for a generic element of $\Omega$. We equip $\Z$ with its Borel $\sigma$-algebra and take the $\sigma$-algebra $\Sigma_\Omega$ on $\Omega$ to be the corresponding product $\sigma$-algebra. By construction, under the regularity conditions in Assumption~\ref{cond-meas-caus}(i,ii), Algorithm~\ref{alg:2} induces a stochastic kernel $\K(\xi, d\zeta)$ on $\Z$ given $\Xi$. Together with Algorithm~\ref{alg:mf-multichain}, this defines a unique probability measure $\P$ on $(\Omega, \Sigma_\Omega)$ by 
\begin{equation} \label{eq-ker-rep1}
\P(d\omega) = \P^b(d\xi)\, \K(\xi, d \zeta).
\end{equation}

By the sequential and non-anticipatory nature of Algorithm~\ref{alg:2}, under the regularity conditions in Assumption~\ref{cond-meas-caus}(i,ii), $\K(\xi, d\zeta)$ admits an explicit representation as a composition of an infinite sequence of stochastic kernels, corresponding to the successive randomized updates of the algorithm based on the history up to each iteration. We now introduce notation to express this representation, for each $t \geq 0$, as a composition of two stochastic kernels corresponding to the iterations up to time $t$ and after time $t$, respectively. These representations will be used in the proof below.

For each $t \geq 0$, let $\Z_t$ denote the space of $\zeta_t$, and let $\Z_{0:t}$ and $\Z_{t+1:\infty}$ denote the spaces of $\zeta_{0:t}:= (\zeta_n)_{n \leq t}$ and $\zeta_{t+1:\infty}:= (\zeta_n)_{n \geq t+1}$, respectively (each equipped with its Borel $\sigma$-algebra). 
Under the regularity conditions stated in Assumption~\ref{cond-meas-caus}(i,ii), Algorithm~\ref{alg:2} induces, for each $t \geq 0$, a stochastic kernel $\K_{t}((\xi, \zeta_{0:t-1}), d \zeta_{t})$ on $\Z_{t}$ given $\Xi \times \Z_{0:t-1}$. 
Consequently, for each $t \geq 0$, 
$\K(\xi, d \zeta)$ can be written in the product form 
\begin{equation} \label{eq-ker-rep2}
\K (\xi, d \zeta) = \K_{0:t} \big( \xi, d \zeta_{0:t}\big) \K_{t+1:\infty} \big((\xi, \zeta_{0:t}), d \zeta_{t+1:\infty}\big),
\end{equation}
where $\K_{0:t}$ is the stochastic kernel on $\Z_{0:t}$ given $\Xi$ obtained by composing $\K_0, \K_1, \ldots \K_t$, and $\K_{t+1:\infty}$ is the stochastic kernel on $\Z_{t+1:\infty}$ given $\Xi \times \Z_{0:t}$ obtained by composing $\K_{t+1}, \K_{t+2}, \ldots$. 

Furthermore, for a given state-action pair $(s,a)$, by Assumption~\ref{cond-meas-caus}(iii), for each $t \geq 0$ and each $(\xi, \zeta_{0:t}) \in \Xi \times  \Z_{0:t}$, the marginal distribution of $(R^{sa}_{t+1}, S^{sa}_{t+1})$ under the probability measure $\K_{t+1}((\xi, \zeta_{0:t}), d \zeta_{t+1})$ on $\Z_{t+1}$ coincides with the joint distribution of the reward and next state in the MDP at state $s$ under action $a$, denoted by
\begin{equation} \label{eq-marg-trans-data}
   \P^{sa}_o\big(d(r^{sa}_{t+1}, s^{sa}_{t+1})\big).
\end{equation}   

With these preparations, we are now ready to prove Theorem~\ref{thm-1}.

\begin{proof}[Proof of Theorem~\ref{thm-1}]
There are finitely many distinct configurations for the collection of sets $\{\tA, \tC^i, \tS^i  \,|\, i \in \{\mathrm{r}, \mathrm{e}, \mathrm{b}\}  \}$ satisfying \eqref{eq-lem1-p1}--\eqref{eq-lem1-p2}. Denote the set of all such configurations by $\J$.
The smallest random time $\tau$ after which the sets $\tA_t$, $\tC_{i,t}$ and $\tS_{i,t}$ remain unchanged for all $i \in \{\mathrm{r}, \mathrm{e}, \mathrm{b}\} $ may be $\infty$ or a finite integer $k \geq 0$. These random outputs, as well as the other random outputs of Algorithm~\ref{alg:mf-multichain} are measurable functions on the sample space $\Xi$, and this dependence on $\xi$ will be made explicit below by writing, for example, $\tA$ as $\tA(\xi)$ and $\tA(s)$ as $\tA(s; \xi)$.

For each $j \in \J$ and $k \geq 0$, let $\Xi_{j,k}$ denote the subset of all sample points $\xi \in \Xi$ for which:
\begin{itemize}[leftmargin=0.6cm,labelwidth=!]
\item[(i)] $\tau(\xi) = k$ and the limiting sets $\tC^i(\xi), \tS^i(\xi)$, and $\tA(\xi)$ match the $j$th configuration in $\J$;
\item[(ii)] as $t \to \infty$, $g_t(\xi) \to g^*$, and $v_t(\xi)$ converges to a limit $v_*(\xi)$ such that the restriction of $v_*(\xi)$ to each $C_{\ell,k} \in \tC$ solves the AOE of the sub-MDP $\M_{\ell,k}$.
\end{itemize}
Our conditions on Algorithm~\ref{alg:mf-multichain}, together with Lemma~\ref{lem-1}, ensure that 
\begin{equation} \label{eq-prf-thm1-1}
\textstyle{\sum_{j \in \J} \sum_{k \geq 0} \P^b( \Xi_{j,k}) = 1.}
\end{equation}

Let $\hat \Xi : = \cup_{j \in \J} \cup_{k\geq 0} \Xi_{j,k}$. For $\xi \in \hat \Xi$, with $\tG(\xi) := \big\{ (s,a) \in \tA(\xi) \mid s \not\in  \tSr(\xi) \big\}$,
define the event
\[ E_\alpha : = \Big\{ \omega = (\xi, \zeta) \in \hat \Xi \times \Z  \, \Big|\, \textstyle{\sum_{t} \alpha_t(s,a; \omega) = \infty, \   \sum_{t} \alpha^2_t(s,a;\omega) < \infty,} \ \forall  (s, a) \in \tG(\xi) \Big\} \]
(where we make explicit the dependence of $\alpha_t$ on the sample point $\omega$).
By \eqref{eq-prf-thm1-1}, Lemma~\ref{lem-1}, and Assumption~\ref{cond-data-stepsize} on Algorithm~\ref{alg:2}, we have $\P(E_\alpha) = 1$, or equivalently, using the stochastic kernel representations for $\P$ discussed earlier [see \eqref{eq-ker-rep1}, \eqref{eq-ker-rep2}],
\[ \sum_{j \in \J} \sum_{k \geq 0}  \int_{\Xi_{j,k}} \!\P^b(d\xi) \int_{\Z_{0:k}} \!\! \K_{0:k}\big(\xi, d\zeta_{0:k}) \int_{\Z_{k+1: \infty}} \!\! \1_{E_\alpha}\big(\xi, \zeta_{0:k}, \zeta_{k+1:\infty} \big) \,  \K_{k+1:\infty}\big((\xi, \zeta_{0:k}),\, d \zeta_{k+1:\infty} \big) = 1.\]
Together with \eqref{eq-prf-thm1-1}, this implies that there exist measurable sets $D_{j,k} \subset \Xi_{j,k} \times \Z_{0:k}$ for $j \in \J$ and $k\geq 0$ such that 
\begin{equation}\label{eq-prf-thm1-2}
 \textstyle{\sum_{j \in \J} \sum_{k \geq 0}  \P \big\{ (\xi, \zeta_{0:k}) \in  D_{j,k} \big\} = 1},
\end{equation}  
and, for all $(j,k)$ with $\P^b(\Xi_{j,k}) > 0$,
\begin{equation}\label{eq-prf-thm1-3}
 \int \1_{E_\alpha}\big(\xi, \zeta_{0:k}, \zeta_{k+1:\infty} \big) \, \K_{k+1:\infty}\big((\xi, \zeta_{0:k}),\, d \zeta_{k+1:\infty} \big) = 1, \quad \forall \, (\xi, \zeta_{0:k}) \in D_{j,k}.
\end{equation} 
Note that since the sets $\Xi_{j,k}$ are disjoint, the events $\{ (\xi, \zeta_{0:k}) \in  D_{j,k}\}$, $j \in \J, k \geq 0$, are disjoint.

For each set $D_{j,k}$ such that $\P^b(\Xi_{j,k}) > 0$, consider the stochastic kernel $\K_{k+1:\infty}((\xi, \zeta_{0:k}), \, d \zeta_{k+1:\infty})$. Specifically, fix an arbitrary pair $(\xi, \zeta_{0:k}) \in D_{j,k}$, and consider the probability measure $\hat \P (\cdot) := \K_{k+1:\infty}\big((\xi, \zeta_{0:k}), \cdot \,\big)$ on $\Z_{k+1:\infty}$. 
For $n \geq 0$, denote
\[\hat Q_n = Q_{k+n}, \quad \hat \alpha_n = \alpha_{k+n}, \quad (\hat R^{sa}_{n}, \hat S^{sa}_{n}) = (R^{sa}_{k+n}, S^{sa}_{k+n}),\]
and note that, for $n =0$, the values of these variables are fully determined by $\zeta_k$.
By \eqref{eq-prf-thm1-3} and the definition of $E_\alpha$, for all $(s,a) \in \tG(\xi)$,
\begin{equation}\label{eq-prf-thm1-4}
 \textstyle{\sum_n \hat \alpha_n(s,a) = \infty, \quad  \sum_n \hat \alpha^2_n(s,a) < \infty,} \ \ \ \text{$\hat \P$-a.s.}
\end{equation} 
Define an increasing sequence of $\sigma$-algebras $\{\hat \F_n\}_{n \geq 0}$ on $\Z_{k+1:\infty}$ by 
\[\hat \F_n : = \sigma\big(\hat Q_m, \hat \alpha_m, \hat R^{sa}_{m}, \hat S^{sa}_{m}; \, m \leq n, (s, a) \in \A\big).\] 
in particular, $\hat \F_0$ is the trivial $\sigma$-algebra $\{ \varnothing, \Z_{k+1:\infty}\}$.

Define a probability measure $\kappa_0$ and stochastic kernels $\kappa_n$, $n \geq 1$, by
\begin{align*}
 \kappa_0(d \zeta_{k+1}) & := \K_{k+1}\big((\xi, \zeta_{0:k}), d \zeta_{k+1}\big), \\\kappa_n\big((\zeta_{k+1}, \ldots, \zeta_{k+n}),\, d \zeta_{k+n+1} \big) & := \K_{k+n+1}\big((\xi, \zeta_{0:k+n}), d \zeta_{k+n+1}\big), \quad n \geq 1.
\end{align*} 
For each $n \geq 0$, $\kappa_n$ is a version of the conditional distribution of $\big(\hat Q_{n+1}, \hat \alpha_{n+1}, (\hat R^{sa}_{n+1}, \hat S^{sa}_{n+1})_{(s, a) \in \A} \big)$ under $\hat \P$ given $\hat \F_n$. The discussion leading to \eqref{eq-marg-trans-data} then shows that, under Assumption~\ref{cond-meas-caus}(iii), for all $n \geq 0$ and each $(s, a) \in \A$, the conditional distribution of $(\hat R^{sa}_{n+1}, \hat S^{sa}_{n+1})$ given $\hat \F_n$ coincides the reward and transition distribution of the MDP at state $s$ under action $a$. 
Consequently, Lemma~\ref{lem-4.1} is applicable to the iterates $\{\hat Q_n\}$ on $\tG(\xi)$ under $\hat \P$, and hence
\begin{equation} \label{eq-prf-thm1-6}
  \hat Q_{n}\big|_{\tG(\xi)} \, \to \, \hat Q^*(\xi)\ \ \text{as}  \ n \to \infty,  \quad \text{$\hat \P$-a.s.},
\end{equation}  
where $\hat Q^*(\xi)$ is the unique solution to \eqref{eq-thm1-Q-ssp-opt} associated with $\big\{v_*(\xi), \epsilon(\xi), \tSr(\xi), \tSb(\xi), \tG(\xi) \big\}$
Since this holds for all $(\xi, \zeta_{0:k}) \in D_{j,k}$, using the stochastic kernel representation of $\P$, we obtain
\begin{equation} \label{eq-prf-thm1-7}
  \P \{ (\xi, \zeta_{0:k}) \in D_{j,k}, \,  Q_t |_{\tG(\xi)} \to \hat Q^*(\xi) \} = \P\{ (\xi, \zeta_{0:k}) \in D_{j,k} \}.
\end{equation}

Finally, consider the probability of the event $\big\{ Q_t |_{\tG(\xi)} \to \hat Q^*(\xi) \big\}$ under $\P$, where we interpret the event as false whenever $\tG(\xi)$ or $\hat Q^*(\xi)$ is not well-defined. By \eqref{eq-prf-thm1-2} and \eqref{eq-prf-thm1-7}, we have
\begin{align*}
 \P \big\{ Q_t |_{\tG(\xi)} \to \hat Q^*(\xi) \big\} & = \sum_{j \in \J, \, k \geq 0} \P \big\{ (\xi, \zeta_{0:k}) \in D_{j,k}, \,  Q_t |_{\tG(\xi)} \to \hat Q^*(\xi) \big\} \\
 & = \sum_{j \in \J, \, k \geq 0}  \P\{ (\xi, \zeta_{0:k}) \in D_{j,k} \} = 1.
\end{align*} 
Hence, $Q_t\big|_{\tG(\xi)} \to \hat Q^*(\xi)$ a.s.

Consider now the sequence $\{x_t\}$ defined from $v_t$ and $Q_t$ via \eqref{def-xt}. By Lemma~\ref{lem-1} and the a.s.\ convergences of $v_t \to v_*$ and $Q_t\big|_{\tG} \to \hat Q^*$, 
we have $x_t \to \bar x$ a.s., where
\[  \bar x(s) := \begin{cases}
   v_*(s), &  s \in \tSr, \\
   \max_{a \in \tA(s)} \hat Q^*(s, a), & s \in \S \setminus \tSr.
   \end{cases}
\]
Since $\hat Q^*$ is the unique solution to \eqref{eq-thm1-Q-ssp-opt}, it follows that $\bar x$ restricted to $\S \setminus \tSr$ satisfies
\begin{equation} \notag
  \bar x(s) = \max_{ a \in \tA(s)} \Big\{ r(s,a) - g^*(s) - \epsilon \1_{\tSb}(s) + \sum_{s' \in \tSr} p(s' \mid s, a) v_*(s') + \sum_{s' \not\in \tSr} p(s' \mid s, a) \, \bar x(s') \Big\}, \ \ \ \forall\, s \in \S \setminus \tSr,
\end{equation} 
and is the unique solution to this equation. Combining this with the AOEs for the sub-MDPs $\M_{\ell,k}$ corresponding to $C_{\ell,k} \in \tCr$, which are satisfied by the corresponding components of $v_*$, yields the assertion of Theorem~\ref{thm-1} for \eqref{eq-paoe} and $\bar x$ (for fixed $\epsilon$).

Finally, view $\bar x$ as a function of $\epsilon$ and write $\bar x_\epsilon$. If $\bd$ is sufficiently small, then $\tA = \A_*$ by \eqref{eq-lem1-p1}, and hence $\bar x_\epsilon$ solves
\[  x  = \max_{\mu \in \Pi_{\A_*}} \left\{ r_\mu - g^* + P_\mu x \right\} - \epsilon\1_{\tSb}. \] 
By the final step of the proof of Prop.~\ref{prp-idealized-alg}, as $\epsilon \to 0$, $\bar x_\epsilon$ converges to a solution $x^*$ of the AOE~\eqref{eq-aoe2}, with $x^* = v_*$ on $\tSr$.
\end{proof}

\subsection{Algorithm~\ref{alg:3}: proofs of Lemma~\ref{lem-stab-sets} and Theorem~\ref{thm-b-conv}} \label{appsec-alg3-prf-conv}

We outline the proofs of Lemma~\ref{lem-stab-sets} and Theorem~\ref{thm-b-conv}, omitting lengthy arguments analogous to those in Step~3 of the proof of Theorem~\ref{thm-1} (Section~\ref{appsec:alg2-conv-prf-step3}), arguments that involve explicit representations of (conditional) probability measures as products of stochastic kernels.

\subsubsection{Proof of Lemma~\ref{lem-stab-sets} (outline)} 

The finite-time stabilization of $\tA_t$ and property \eqref{eq-stab-tAB} of its limit $\tA$ were established in Lemma~\ref{lem-1}, following from the convergence of Algorithm~\ref{alg:mf-multichain} and the update rules~\eqref{rule:Ae1}--\eqref{rule:tA} for $\tA_t$.

For the remaining assertions of Lemma~\ref{lem-stab-sets} concerning the convergence of $\{Q_{\beta,t}\}$ and the finite-time stabilization of $\tB_t$, we first prove the following lemma.

\begin{lemma} \label{lem-alg3-conv-phase1}
The sequence $\{Q_{\beta,t}\}$ and $\{v_{\beta,t}\}$ generated in Phase~1 of Algorithm~\ref{alg:3} converge a.s., satisfying
\begin{equation} \label{eq-prf-stab-tB-1}
  Q_{\beta,t} \to Q_\beta, \qquad v_{\beta,t} \to v_\beta,
\end{equation} 
where $Q_\beta$ is the $\beta$-discounted optimal action-value function, and $v_\beta$ is the $\beta$-discounted optimal value function, of the sample path-dependent MDP $\widetilde{\M} = (\S, \tA, p, r - g^*)$.
\end{lemma}

\begin{proof}
Since $\tA_t=\tA$ for all sufficiently large $t$, the Q-learning updates \eqref{eq-alg3g-ql} in Phase~1 of Algorithm~\ref{alg:3} eventually correspond to $\beta$-discounted Q-learning iterates for the MDP $\widetilde{\M} = (\S, \tA, p, r - g^*)$, with rewards involving a vanishing biased noise term $g^*(s) - g_t(s)$. We proceed as in the proof of Theorem~\ref{thm-1} and first consider the limiting form of Phase~1, where $\{g_t\}$ and $\tA$ are fixed (and thus so is the MDP $\widetilde{\M}$). 

By \cite[Prop.~4.5]{BeT96}, $\beta$-discounted Q-learning converges a.s.\ in the presence of vanishing biased reward noise. Therefore, the limiting form produces iterates $Q_{\beta,t}$ and $v_{\beta,t}$ satisfying \eqref{eq-prf-stab-tB-1}. 

By arguments analogous to Step~3 of the proof of Theorem~\ref{thm-1} (see~Section~\ref{appsec:alg2-conv-prf-step3}), this convergence extends to the original form of Phase~1 of Algorithm~\ref{alg:3}, yielding the desired conclusion.
\end{proof}

From Lemma~\ref{lem-alg3-conv-phase1} and the update rule \eqref{rule:tB} for the sets $\tB_t$, we conclude that $\tB_t$ eventually becomes constant, and that the resulting set $\tB$ satisfies property \eqref{eq-stab-tAB}. This establishes Lemma~\ref{lem-stab-sets}.

\subsubsection{Proof of Theorem~\ref{thm-b-conv} (outline)} 

By Lemmas~\ref{lem-alg3-conv-phase1} and~\ref{lem-stab-sets}, almost surely, $v_{\beta,t} \to v_\beta$ and $\tB_t = \tB$ for all sufficiently large $t$. Thus, in Phase~2 of Algorithm~\ref{alg:3}, the base algorithm is eventually applied to the (sample path-dependent) MDP $\M' = (\S, \tB, p, - v_\beta)$, but with RVI Q-learning in Phase~2 using rewards $- v_{\beta,t}(s)$, which satisfy $v_{\beta,t}(s) \to v_\beta(s)$ for all $s \in \S$.

As in the proof of Theorem~\ref{thm-1}, we first consider the limiting form of Phase~2, where $\tB$ and $v_\beta$, as well as the convergent sequence $\{v_{\beta,t}\}$, are treated as given. 

\begin{lemma}  \label{lem-alg3-conv-phase2}
Suppose $\tB$, $v_\beta$, and $\{v_{\beta,t}\}$ are given, with $v_{\beta,t} \to v_\beta$. Run the base algorithm on the MDP $\M' = (\S, \tB, p, - v_\beta)$, except that in Phase~2, RVI Q-learning uses rewards $- v_{\beta,t}(s)$ instead of $- v_\beta(s)$. Then, almost surely for all sufficiently large $t$, the policy $\pi^{(t)}$ produced in Phase~4 of the base algorithm is gain-optimal for $\M'$.
\end{lemma}

\begin{proof}[Proof (outline)]
The only difference from the earlier setting in Section~\ref{sec:baseline-alg} is that, on each communicating sub-MDP $\M_{\ell,k}$ in Bather's decomposition of $\M'$, the RVI Q-learning iterates involve additional biased reward noise bounded by a deterministic diminishing sequence. This scenario is covered by extended analyses of RVI Q-learning under more general noise conditions \cite{YWS25b}, using asynchronous stochastic approximation results from \cite{YWS25a}. In particular, the desired convergence of RVI Q-learning in this setting follows from \cite[Thm.~2.2 and Cor.\ 2.1]{YWS25a} and \cite[Proof of Thm.~3.1]{YWS25b}. (The latter applies straightforwardly after specializing from SMDPs to MDPs and redefining the noise terms $M_{n+1}$ and $\epsilon_{n+1}$ \cite[Lem.~4.2]{YWS25b} to match the present setting.)

This convergence of RVI Q-learning in Phase~2 ensures that Phases~3 and~4 of the base algorithm retain the same convergence properties as in the previous setting. In particular, the policy $\pi^{(t)}$ from Phase~4 is eventually gain-optimal for $\M'$ a.s.
\end{proof}

Finally, by arguments analogous to Step 3 of the proof of Theorem~\ref{thm-1} (see~Section~\ref{appsec:alg2-conv-prf-step3}), the conclusion of Lemma~\ref{lem-alg3-conv-phase2} extends from the limiting form to the original form of Phase~2 of Algorithm~\ref{alg:3}. This establishes Theorem~\ref{thm-b-conv}.

\section{Concluding remarks} \label{sec:conclusion}
In this work, we developed three asynchronous value-iteration-based RL algorithms for solving general multichain MDPs under the average-reward criterion. The algorithms leverage Bather's UHD to recast the global decision problem into structured subproblems and require no model knowledge beyond the MDP's transition graph. They address, respectively, optimal gain estimation and gain-optimal policy learning, approximate solution of the multichain AOE with near gain-optimality guarantees, and the computation of policies approaching bias optimality. We provided almost-sure convergence and eventual performance guarantees for all three algorithms, and our preliminary experiments illustrate their different tradeoffs in gain optimality and transient performance. The ideas underlying these algorithms can also be combined in different ways to yield a broader family of essentially model-free RL algorithms, as illustrated by several variants discussed in the paper.

The performance guarantees established here are eventual: although the policies produced by our algorithms attain the stated (near) optimality properties after finite time a.s., the algorithms do not certify when this has occurred. This differs from high-probability policy identification results in the RL literature. We mention two such results for comparison. For communicating MDPs, the model-based policy iteration algorithm of \cite{BoT25} identifies a bias-optimal policy with high probability at a finite stopping time under strong policy uniqueness conditions that ensure robustness to model perturbations. For multichain MDPs, high-probability identification of near gain-optimal policies, along with sample complexity analysis, has been obtained in \cite{ZuC24} through a discounted reduction and a model-based approach. These results provide a different type of guarantee from the almost-sure eventual guarantees studied here.

An important direction for future work is to remove the transition-graph assumption. This assumption is considerably weaker than full model knowledge and is natural in many simulation-based settings, where possible successor states can be enumerated from a simulator's code or interface even when their transition probabilities are difficult to compute or depend on unobserved factors, such as other agents' behavior or physical interactions. Nevertheless, it remains a genuine restriction. Within the almost-sure convergence framework considered here, the transition graph could in principle be learned in parallel with our algorithms through sufficiently persistent exploration.

Rather than first learning the entire transition graph, it would be more interesting to discover and maintain the relevant structure online. Since Bather's UHD depends only on which transitions are possible, an estimated decomposition could be dynamically refined as new transitions are observed. This also raises broader questions about how qualitative structural information about complex stochastic dynamics can be learned through exploration and exploited for control, separately from learning quantitative transition probabilities. We plan to investigate these questions in future work.

There are several other directions we plan to explore. The subproblem structure induced by Bather's UHD suggests connections to the options framework and, in particular, the possibility of reusing learned sub-policies as options. We also plan more extensive experiments to better understand the relative strengths of the algorithms and their variants, as well as more flexible and adaptive choices of algorithmic parameters in place of the fixed thresholds used in our theoretical analysis.

\section*{Acknowledgments}
\addcontentsline{toc}{section}{Acknowledgments}
This research was supported by the Alberta Machine Intelligence Institute (Amii) and the Natural Sciences and Engineering Research Council of Canada (NSERC) under grants RGPIN-2024-04939 and DGECR-2024-00312. H.Y.\ thanks Professor Eugene Feinberg for helpful discussions on average-reward MDPs and for drawing her attention to Bather’s work. Both authors thank Professors Csaba Szepesvari and Rich Sutton for valuable feedback. We thank anonymous reviewers of an earlier version of this paper for helpful comments and suggestions that led to improvements in the presentation of this work. In preparing this paper, we used OpenAI's ChatGPT (GPT-4 and GPT-5) to help refine the writing style.

Both authors contributed equally to this work.

\appendix
\section{Background: RVI Q-learning} \label{append-rvi-ql}

RVI Q-learning, introduced by Abounadi, Bertsekas, and Borkar \cite{ABB01}, is a stochastic analogue of the classical RVI algorithm for average-reward MDPs \cite{Whi63,ScF77}. Its development and analysis build on a ``self-regulating'' algorithmic design (reflected in the function $f$ and its conditions discussed below), stochastic approximation theory \cite{Bor98,BoM00,Bor23,YWS25a}, and Borkar and Soumyanath's general results concerning continuous-time analogs of fixed-point iteration involving nonexpansive mappings \cite{BoS97}. We use the asynchronous version of RVI Q-learning in our base algorithm for solving communicating sub-MDPs. Below we briefly explain the algorithm and the conditions and convergence guarantees needed for our purpose.

Consider a \emph{weakly communicating} MDP $\M = (\S, \A, p, r)$. Let $\{\alpha_n\}$ be a given setpsize sequence and $f : \R^{|\A|} \to \R$ a Lipschitz continuous function. Further conditions on $\{\alpha_n\}$ and $f$ will be given shortly. RVI Q-learning generates a sequence $\{Q_t\}$ by iteratively updating selected components of $Q_t$ as follows. At iteration $t \geq 0$: 
\begin{itemize}[leftmargin=0.5cm,labelwidth=!]
\item Select a nonempty subset $Y_t \subset \A$. For each $(s,a) \in Y_t$, sample a random state transition and reward  $(S_{t+1}^{sa}, R_{t+1}^{sa})$ from $(s, a)$ according to $\M$.
\item For $(s, a) \in Y_t$, let
\begin{equation}
   Q_{t+1}(s, a) = Q_{t}(s, a)  +  \alpha_{\nu(t, (s, a))}  \Big( R_{t+1}^{sa}  -  f(Q_t) + \max_{a' \in \A(S_{t+1}^{sa}) } Q_t (S_{t+1}^{sa}, a') - Q_t(s, a) \Big).  \label{eq: q-ite}  
\end{equation}
For $(s, a) \not\in Y_t$, let $Q_{t+1}(s, a) = Q_{t}(s, a)$. In \eqref{eq: q-ite}, 
\[ 
\nu(t,(s,a))  := \sum_{k=0}^{t-1} \ind \{ (s,a) \in Y_k \} 
\]
counts how many times the state-action pair $(s,a)$ has been chosen prior to iteration $t$. Thus, if a pair has previously been selected $n$ times, the update at iteration $t$ uses the stepsize $\alpha_n$. \end{itemize}
We next explain what the algorithm aims to compute and state the conditions needed for its convergence.

The goal is to solve an equivalent form of the AOE for $\M$ in terms of Q-factors:
\begin{equation}   \label{eq-aoe-Q}
Q(s,a) = r(s,a) - \bar r + \sum_{s' \in \S} p(s' \mid s, a) \, \max_{a' \in \A(s') } Q (s', a'), \qquad (s,a) \in \A,
\end{equation}
where $\bar r$ is a scalar variable playing the role of the unknown optimal gain $g^*$. (This equation can also be interpreted as the AOE of an MDP on enlarged state and action spaces, so average-reward MDP theory directly applies to it.)  For any solution $(\bar r, Q)$ of this equation, necessarily $\bar r = g^*$, and the function $h(s) = \max_{a \in \A(s)} Q(s, a)$, $s \in \S$, solves the AOE~\eqref{eq-aoe2} for $\M$.

The function $f$ employed in RVI Q-learning serves to ``self-regulate'' and normalize the iterates, with $f(Q_t)$ intended to converge to $g^*$. This normalization is what makes convergence of the $Q_t$ iterates to the AOE solution set possible. Recently, \citet{YWS25b} introduced the following general monotonicity conditions on $f$ and its scaling limit $f_\infty$.

Let $d = |\A|$. We call a function $\phi: \R^d \to \R$ \emph{strictly increasing under scalar translation} (SISTr) if, for every $x \in \R^d$, the function $c \in \R \mapsto \phi(x + c \1)$ is strictly increasing and maps $\R$ onto $\R$. If this condition holds at a specific point $x$, we say that $\phi$ is \emph{SISTr at} $x$.

\begin{assumption}[Conditions on function $f$; {\cite{YWS25b}}]\label{cond-f} \hfill
\begin{itemize}[leftmargin=0.7cm,labelwidth=!]
\item[\rm (i)] $f$ is Lipschitz continuous and SISTr.
\item[\rm (ii)] As $c \uparrow \infty$, the functions $x \mapsto f(c x)/c$ converge pointwise to a real-valued function $f_\infty(x)$, and $f_\infty$ is SISTr at the origin.
\end{itemize}
\end{assumption}

Functions satisfying this assumption include the following family of Lipschitz continuous functions considered in \cite{ABB01,WNS21a}: 
\begin{itemize}[leftmargin=0.6cm,labelwidth=!]
\item For some scalar $u > 0$, $f(x + c \1) = f(x) + c u$ for all $x \in \R^d$ and $c \in \R$; and 
\item for $c \geq 0$, $f(c x) = f(0) + c (f(x) - f(0))$.
\end{itemize}
For these functions, the scaling limit is $f_\infty(x) = f(x) - f(0)$, i.e., $f_\infty$ coincides with $f$ after subtracting $f(0)$. Examples include $f(x) = b + \theta^\top x$, where $b \in \R$, $\theta \in \R^d$, and $\sum_{i=1}^d \theta_i > 0$; and $f(x)  = b + \beta \max_{i \in D} x_i$ or $f(x)  = b + \beta \min_{i \in D} x_i$, where $b \in \R$, $\beta > 0$, and $D \subset \{1, 2, \ldots, d\}$. Examples beyond this family but admissible under the more general Assumption~\ref{cond-f} include $f(x) = \max \{ g_1(x), \ldots, g_m(x)\}$ or $\min \{ g_1(x)$, $\ldots, g_m(x)\}$, where each function $g_k$ satisfies Assumption~\ref{cond-f}, e.g., by fitting into one of the aforementioned types. See~\cite[Ex.~3.2]{YWS25b} for more general examples involving a compositional mapping in place of $\max$ or $\min$.
 
We next state the remaining conditions, concerning the stepsizes and the asynchronous update schedule.

\begin{assumption}[Stepsize conditions]  \label{cond-ss} \hfill
\begin{enumerate}[leftmargin=0.7cm,labelwidth=!] 
\item[{\rm (i)}] $\sum_n \alpha_n = \infty$, $\sum_n \alpha_n^2 < \infty$, and 
$\alpha_{n+1} \leq \alpha_n$ for all $n$ sufficiently large.
\item[{\rm (ii)}] For $x \in (0,1)$, 
$\sup_n \frac{\alpha_{[ x n]}}{ \alpha_n} < \infty$,
where $[x n]$ denotes the integral part of $xn$.
\item[{\rm (iii)}] For $x \in (0,1)$, as $n \to \infty$, $\frac{ \sum_{k=0}^{[ y n ]} \alpha_k }{ \sum_{k=0}^{n} \alpha_k} \to 1$ uniformly in $y \in [x, 1]$.
\end{enumerate}
\end{assumption} 

For $x > 0$ and $t \geq 0$, define $N(t,x) := \min  \left\{ m > t : \sum_{k = t}^m \alpha_k \geq x \right\}$.
\begin{assumption}[Asynchronous update conditions] \label{cond-us} \hfill 
\begin{enumerate}[leftmargin=0.7cm,labelwidth=!]
\item[{\rm (i)}]  For some deterministic $\eta > 0$, 
$\liminf_{t \to \infty} \nu(t, i)/t  \geq \eta$ a.s., for all $i = (s,a) \in \A$.
\item[\rm (ii)] For each $x > 0$, the limit $\lim_{t \to \infty} \frac{ \sum_{k = \nu(t, i)}^{\nu(N(t,x), i)} \alpha_k}{ \sum_{k = \nu(t,j)}^{\nu(N(t,x), j)} \alpha_k}$ exists a.s., for all $i = (s,a), j = (s', a') \in \A$.
\end{enumerate}
\end{assumption} 

These conditions are based on those introduced in Borkar~\cite{Bor98} for asynchronous stochastic approximation and adopted in~\cite{ABB01} for RVI Q-learning. Admissible stepsize examples include stepsizes asymptotic to $1/ n$ or $1/ (n \ln n)$. For asynchronous updates, Assumption~\ref{cond-us}(i) requires a strictly positive frequency for updating each $Q$-factor component, while Assumption~\ref{cond-us}(ii) requires the components to receive comparable amounts of updates on the stepsize time scale. Both conditions are met, in particular, if the subset $Y_t$ of state-action pairs is eventually chosen based on an irreducible Markov chain on $\A$ (see \cite[Ex.~3]{WYS24}). 

While these conditions are more restrictive than the stepsize and asynchony conditions required for classical Q-learning \cite{Tsi94}, they are needed here because the mappings underlying RVI Q-learning are generally neither contractive nor nonexpansive. They impose partial asynchrony, allowing the asymptotic behavior of the asynchronous algorithm to align, on average, with that of a synchronous counterpart and thereby facilitating analysis \cite{Bor98,YWS25a}.

Under Assumptions~\ref{cond-f}--\ref{cond-us}, it is shown in \cite[Thm.~3.1]{YWS25b} that RVI Q-learning \eqref{eq: q-ite} converges a.s., with
\begin{equation}
f(Q_t) \to g^*, \qquad  Q_t \to \text{a compact subset of solutions of the AOE~\eqref{eq-aoe-Q}},
\end{equation}
where the subset is generally sample-path-dependent. Moreover, this convergence of $Q_t$ implies that, a.s., any greedy policy w.r.t.\ $Q_t$ is gain-optimal after finite time, and that the state-value functions $v_t(s) := \max_{a \in \A(s)} Q_t(s,a), s \in \S$, converge to a compact subset of solutions of the AOE~\eqref{eq-aoe2}.

Furthermore, under additional stepsize and asynchrony conditions \cite[Assump.~3.4(i)]{YWS25b}, \cite[Thm.~3.2]{YWS25b} establishes the a.s.\ convergence of RVI Q-learning to \emph{a unique sample-path-dependent AOE solution}. (This result itself is an application of the more general result for asynchronous stochastic approximation algorithms given in \cite[Thm.~2.3]{YWS25a}.) Specialized to the context here, the following gives a particular choice of stepsize sequence and asynchronous update scheme satisfying the required conditions (cf.~\cite[Rem.~3.4]{YWS25b}). 
\begin{itemize}[leftmargin=0.6cm,labelwidth=!]
\item Let the selection of the components to be updated eventually be governed by an irreducible Markov chain on $\A$.
\item Let $\alpha_n = \tfrac{1}{A n}$ for $n \geq 1$, where $A > 0$ is a scaling parameter satisfying
\begin{equation}
A > 4 + 2 L_f,
\end{equation}
where $L_f$ is the Lipschitz constant of $f$ w.r.t.\ $\|\cdot\|_\infty$.
\end{itemize}
With these choices, \cite[Thm.~3.2]{YWS25b} implies that $Q_t$ converges a.s.\ to a unique sample-path-dependent solution of the AOE~\eqref{eq-aoe-Q}, and its associated state-value function $v_t$ converges to the corresponding solution of the AOE~\eqref{eq-aoe2}.

\section{Background: Q-learning for SSP-type total-reward MDPs} \label{append-ql-ssp}

Let $\hat \M = (\hat \S \cup {\Delta}, \hat \A, \hat p, \hat r)$ be a total-reward MDP of SSP type (as defined in Section~\ref{sec-background-ssp}). Q-learning for $\hat \M$ is similar to classical Q-learning for discounted problems \cite{WatkinsQ}, except that the discount factor is replaced by $1$ and the Q-value at the terminal state $\Delta$ is fixed at zero. Specifically, at iteration $t \geq 0$, the algorithm performs the following operations:
\begin{itemize}[leftmargin=0.5cm,labelwidth=!]
\item Select a nonempty subset $Y_t \subset \hat \A$. For each $(s,a) \in Y_t$, sample a random state transition and reward  $(S_{t+1}^{sa}, R_{t+1}^{sa})$ from $(s, a)$ according to $\hat \M$.
\item For $(s, a) \in Y_t$, let
\begin{equation}
   Q_{t+1}(s, a) = Q_{t}(s, a)  +  \alpha_t(s,a)  \Big( R_{t+1}^{sa}  + \ind\{S^{sa}_{t+1} \not= \Delta \} \cdot \!\max_{a' \in \hat \A(S_{t+1}^{sa}) } Q_t (S_{t+1}^{sa}, a') - Q_t(s, a) \Big),  \label{eq: ssp-q-ite}  
\end{equation}
where $\alpha_t(s,a) \geq 0$ is a possibly random stepsize that can depend on information available at time $t$, including $Y_t$ (but not on the transition data indexed by $t+1$). For $(s, a) \not\in Y_t$, $Q_{t+1}(s, a) = Q_{t}(s, a)$.
\end{itemize}
We mention that more generally, when implemented using parallel distributed computation, instead of $Q_t (S_{t+1}^{sa}, a')$, the update \eqref{eq: ssp-q-ite} can use delayed Q-factors of the form $Q_{t -\tau} (S_{t+1}^{sa}, a')$, $\tau \geq 0$, to allow communication delays between different processors (see \cite{Tsi94,YuB13} for details). For simplicity of exposition, we do not consider communication delays here.

Almost-sure convergence of $Q_t$ to the optimal action-value function $Q^*$ of $\hat \M$ follows from the convergence results of \citet[Thm.~2]{Tsi94} given stability, together with the stability results of \citet[Thm.~3.1]{YuB13}, under the same general stepsize and asynchrony conditions as in discounted Q-learning:

\begin{assumption}[Asynchrony and stepsize conditions]  \label{cond-ssp-ql} 
The following hold a.s.:
\begin{itemize}[leftmargin=0.7cm,labelwidth=!]
\item[\rm (i)] Every Q-factor component is updated infinitely often.
\item[\rm (ii)] For each $(s,a) \in \hat \A$, under the convention that $\alpha_t(s,a) = 0$ whenever $(s,a) \not\in Y_t$, the stepsizes satisfy $\sum_{t} \alpha_t(s,a) = \infty$ and $\sum_{t} \alpha^2_t(s,a) < \infty$.
\end{itemize}
\end{assumption}

The stability proof in \cite{YuB13} is quite involved for a general SSP-type MDP $\hat \M$ with improper policies. It need not be invoked for several special cases of $\hat \M$, where almost-sure convergence of $Q_t$ to $Q^*$ follows readily from alternative arguments that exploit stronger properties of $\hat \M$. These cases include:
\begin{itemize}[leftmargin=0.65cm,labelwidth=!]
\item When every stationary policy in $\hat \M$ is proper, the mapping underlying the Q-learning iteration~\eqref{eq: ssp-q-ite} is contractive w.r.t.\ a weighted sup-norm \cite[Prop.\ 2.2]{BeT96}. Consequently, $Q_t \to Q^*$ a.s.\ by \cite[Thm.~3]{Tsi94} (see also \cite[Secs.\ 4.3, 5.6]{BeT96}).
\item When improper policies may exist in $\hat \M$ but the one-stage reward function $\hat r(s,a)$ is nonpositive, the a.s.\ convergence $Q_t \to Q^*$ was proved in \cite[Lem.\ 9 and Thm.~4]{Tsi94} for the case where all random rewards are also nonpositive (see also \cite[Prop.\ 5.6]{BeT96}), and later in \cite[Prop.\ 3.2]{YuB13} with most of the sign restrictions on the random rewards removed.
\end{itemize}

Our base algorithm applies the two special cases above: its Phase~3 uses the first case in solving the aggregated OSP, while its Phase~4 uses the second case in a realization of the \textsc{ReachPolicy} subroutine. Algorithm~\ref{alg:2} for computing an approximate AOE solution involves solving a general SSP-type MDP, and its convergence analysis required a slight extension of prior stability and convergence results from \cite{Tsi94,YuB13}, which we provided in Section~\ref{sec-ssp-analysis-extension}.

\section{Experimental methodology and additional results}
\label{sebsec:app-experiments}
This section provides further details on the experimental
methodology and results.
Section~\ref{subsec:app-mdp-gen} describes the MDP
generation procedure and the resulting benchmark.
Section~\ref{subsec:app-alg-implementation-details}
specifies the implementation, update schedules, and
parameter choices.
Section~\ref{subsec:app-extra-exp-results} presents
per-MDP results and discusses the observed differences
in gain and transient performance.
\subsection{MDP generation}
\label{subsec:app-mdp-gen}

\begin{table}[h]
\centering
\caption{Summary statistics over 100 generated MDPs.}
\label{tab:mdp_stats}
\begin{tabular}{lccc}
\hline
\textbf{Statistic} & \textbf{Min} & \textbf{Mean} & \textbf{Max} \\
\hline
\multicolumn{4}{l}{\textit{Structural parameters (controlled)}} \\[2pt]
Number of levels & 3 & 4.53 & 6 \\
Number of communicating classes & 3 & 11.25 & 23 \\
Number of transient states & 1 & 9.36 & 20 \\
Number of states & 20 & 71.67 & 147 \\
Number of actions & 29 & 149.8 & 359 \\
States per class (mean) & 3.88 & 5.53 & 6.57 \\
Actions per state (mean) & 1.33 & 2.04 & 3.42 \\
Per-MDP reward span & 14.93 & 20.35 & 27.57 \\
Number of preferred (high-gain) classes & 3 & 6.49 & 12 \\
Number of classes made multichain & 0 & 6.52 & 16 \\[4pt]
\midrule
\multicolumn{4}{l}{\textit{Emergent gain and class structure}} \\[2pt]
$\Cr$ classes per MDP & 1 & 2.97 & 7 \\
$\Ce$ classes per MDP & 0 & 3.90 & 14 \\
$\Cb$ classes per MDP & 0 & 4.38 & 12 \\
Optimal gain span & 0 & 7.51 & 16.78 \\
MDPs with nonconstant optimal gain & \multicolumn{3}{c}{77 / 100} \\
\hline
\end{tabular}
\end{table}
MDPs are generated using a structured procedure that directly leverages Bather's UHD framework, ensuring that every generated instance has a known, valid hierarchical decomposition by construction. For each MDP, the number of levels $L$ is sampled uniformly in the range $[3,6]$, the number of classes per level in the range $[1,4]$, the number of transient states per level in the range $[0,4]$, and the number of states per class in the range $[3,8]$.

Each communicating class $C_{\ell, k}$ is constructed by first laying a random directed cycle over its states (guaranteeing communication), then adding additional random internal transitions (possibly stochastic) with probability sampled uniformly in the range $[0.2,0.6]$ for each MDP. For each class at a level $\ell \ge 1$, a single exit action transitioning deterministically to a randomly selected class at level $\ell-1$ is added to ensure the desired hierarchical structure. Additional exit actions are added with probability sampled from $[0.4, 0.8]$, these exit actions can be stochastic, with each stochastic transition pointing to a random subset of between $1$ and $3$ states, at least one of which is at a strictly lower level. Transient states receive random actions pointing to classes or earlier transients at the same or lower levels, with probability sampled from $[0.4, 0.8]$. Critically, to preserve the desired structure, we ensure no cycles are created between transient states.

To induce genuine multichain structure within classes, each class is independently made multichain with probability sampled per-MDP from $[0.2, 0.9]$. Two methods are used with equal probability: the first selects a random subset of $2-4$ states and equips each with a self-absorbing action; the second partitions the class into $2-3$ random subgroups, each of which is given its own directed subcycle, so multiple distinct recurrent behaviors are possible within the class.

Rewards are assigned through a three layered mechanism applied in sequence. First, a base random reward drawn uniformly from $[-3, 3]$ is assigned to every action. Second, a subset of classes at randomly selected preferred levels are designated as high-gain targets and given an additional reward bonus sampled from $[2, 6]$; non-preferred classes receive a penalty sampled from $[-6, -2]$. For multichain classes at preferred levels, subcycle and absorbing actions additionally receive rewards sampled from $[2, 6]$, while ordinary internal cycling actions are penalized by amounts sampled from $[-6, -2]$. This creates diverse gain and bias structure within classes, where some recurrent behaviors within the same class are more desirable than others.

To explicitly construct $\Cb$-type classes (those for which staying and exiting are gain-tied, and optimal decisions come down to optimizing bias), a dedicated construction is applied. Between 2 and 10 classes per-MDP are selected for this construction. For each such class, a base gain level is sampled from $[1, 10]$. With equal probability, we assign each such class as \emph{exit-good} or \emph{stay-good}, where we construct rewards in order induce optimal behavior which depends on the bias. For exit-good classes, exit actions are designed to offer a reward bonus sampled from $[5, 15]$, over the base gain. For stay-good classes, a value $x \in [5, 10]$ is sampled, $-x$ is applied to exit actions and $+x$ is applied to an internal action leading to a cycle, creating a better bias from remaining within the class. This aims to ensure that a nontrivial fraction of classes present a genuine gain-bias tradeoff with hopes of distinguishing behavior between our algorithms.

Table~\ref{tab:mdp_stats} summarizes the structural and emergent properties of the 100 generated MDPs. The controlled parameters reflect a deliberately diverse benchmark: MDPs range from small instances with 20 states and 29 actions to larger ones with 147 states and 359 actions, with varying numbers of levels, classes, and transient states. The per-MDP reward span, which averages 20.35 and ranges from 14.93 to 27.57, confirms that the layered reward construction — combining base random rewards with level-preference bonuses and multichain subcycle penalties — creates substantial reward variation within each MDP, giving algorithms meaningful signals to differentiate between good and bad behavior.

The emergent class and gain structure is more revealing. As defined in equations \eqref{eq-def-C1}-\eqref{eq-def-C3}, $\Cr$ is the set of classes where it is gain-optimal to stay, $\Ce$ is the set where it is gain-optimal to exit, and $\Cb$ is the set of gain-tied regions, where staying and exiting is equal in terms of gain. The presence of $\Ce$ and $\Cb$ classes across the benchmark confirms that the construction has largely succeeded in inducing the gain-tied and gain-dominated structure that motivates the augmented algorithms. Notably, 98 of the 100 MDPs contain at least one $\Cb$ class providing a strong basis for evaluating the performance differences between Algorithms \ref{alg:2} and \ref{alg:3}. A moderate limitation of the construction is that 24 MDPs ended up with no $\Ce$ classes, meaning in terms of gain there is no class where it is strictly better to exit. The absence of $\Ce$ classes primarily affects the base algorithm's transient performance: in classes where exiting is optimal, the base algorithm uses $\pi_{\ell, k}^{\mathrm{reach}}$ to navigate to an exit state $s^*$ via an SSP formulation, which minimizes hitting time without accounting for intermediate rewards, leading to poor transient behavior. Algorithms \ref{alg:2} and \ref{alg:3} bypass this issue through their action constraint construction, so this limitation does not materially affect the comparison between them. Since the base algorithm's transient performance was already demonstrated to be poor in the main results and in Subsection \ref{subsec:app-extra-exp-results}, we believe these 24 MDPs remain representative of the general behavior of all algorithms. Additionally, 76 of the MDPs still had a variety of classes which fall into the $\Ce$ classification. Nevertheless, future benchmark construction should enforce a minimum number of $\Ce$ classes per MDP to ensure fuller coverage of this case.

Turning to the optimal gain structure, the 77 MDPs with non-constant optimal gain confirm that the majority of the benchmark consists of genuinely multichain problems where the long-run performance depends on the initial state. The remaining 23 MDPs, despite being structurally multichain, happen to have a constant optimal gain vector, reducing to a more weakly communicating problem and representing the easier end of the spectrum. Among the 77 truly multichain instances, the span of the optimal gain varies considerably: 39 MDPs had gain span exceeding $10.0$ --- cases where the optimal gain varies substantially across states and where correctly identifying the optimal classes to stop at are critical to performance. The remaining 38 MDPs fall in between, exhibiting moderate gain variation. Together, the benchmark spans
a wide difficulty spectrum, providing a comprehensive testbed for evaluating average-reward RL algorithms in the multichain setting.

\subsection{Experimental implementation details}
\label{subsec:app-alg-implementation-details}

All experiments were run on a 2021 Apple MacBook Pro with an Apple M1 chip and 16\,GB of memory. The algorithms were implemented in Python 3.13.2. Each of the 100 MDP experiments was run for $750{,}000$ iterations, requiring approximately $1$ hour per MDP and $100$ hours in total. Experiments were run sequentially across MDPs. At the beginning of each run, the parameters needed to recreate the MDP were saved, along with the bias-optimal benchmark policy. Every $1{,}000$ iterations a light history of the run was saved consisting of the current policies of the algorithms. Every $50{,}000$ iterations a full checkpoint was created saving all of the learner objects. Each MDP experiment took up approximately 5\,MB of storage for a total of approximately 500\,MB.

 The implementation uses a round-robin update scheme for reproducibility and ease of comparison. In each outer iteration, every phase advances by exactly one full sweep over its state-action pairs in a fixed deterministic order (randomized only at initialization). Phases are not actually run in parallel; instead, they are frozen and advanced sequentially within each outer iteration, which approximates the concurrent operation described in the paper. Since the iterative phases of our algorithms are specified to operate concurrently, a distributed or parallel implementation is a natural next step, which we leave for future work. The augmented algorithms begin advancing only after $\bar{t} = 30{,}000$ outer iterations, allowing the base algorithm sufficient time to produce stable gain and Q-function estimates before the augmented phases begin using them.

All Q-learning components use a polynomial stepsize schedule $\alpha_t(s,a) = c / (\text{offset} + n_t(s,a))^\omega$ where $n_t(s,a)$ is the number of updates to pair $(s,a)$ up to time $t$. The parameters used are $c = 1.0$, $\omega = 0.7$ and $\text{offset} = 10.0$ for all components. A warm-start phase is applied for the base algorithm in Phase 2 RVI Q-learning: for the first $5{,}000$ outer iterations, the more aggressive polynomial stepsize schedule above is used, after which the schedule switches to $\alpha_t(s,a) = 1 / ( An_t(s,a))$ with $A=8.0$. The $f$ function for RVI Q-learning is taken to be the maximum of the $Q$-function over all state-action pairs (and over all admissible state-action pairs for Algorithm \ref{alg:3} Phase 2).

The thresholds for the set constructions in Algorithms \ref{alg:2} and \ref{alg:3} are set to $\ud = \ud_1 = 0.03$, $\bd = \bd_1 = 0.06$. The additional thresholds for Algorithm \ref{alg:3} are set to $\ud_2 = 0.01$ and $\bd_2 = 0.03$. The penalty parameter for Algorithm 2 is $\epsilon = 0.05$, held constant throughout for simplicity. The discount factor for Phase 1 of Algorithm \ref{alg:3} is $\beta = 0.99$, and Algorithm \ref{alg:3}'s Phase 2 begins at iteration $\bar{t}' = \bar{t} + 10{,}000$. The discounted Q-learning baseline uses $\beta = 0.9995$.

Phase 4 of the base algorithm requires constructing a policy that reaches a target state $s^*$ from all states within a class $C_{\ell,k}$. This is implemented by solving an SSP problem on $C_{\ell,k}$ using $1{,}000$ round-robin Q-learning sweeps, with results cached across outer iterations to avoid redundant recomputation.

The 0-discount policy iteration benchmark is computed once for each MDP using the true MDP model and the exact $0$-discount policy iteration algorithm \cite{Vei66} (see also \cite[Sec.\ 10.3]{Put94}).

\subsection{Additional experimental results and discussion}
\label{subsec:app-extra-exp-results}

\begin{figure}[t]
    \centering
    \begin{subfigure}[t]{0.48\linewidth}
        \centering
        \includegraphics[width=\linewidth]{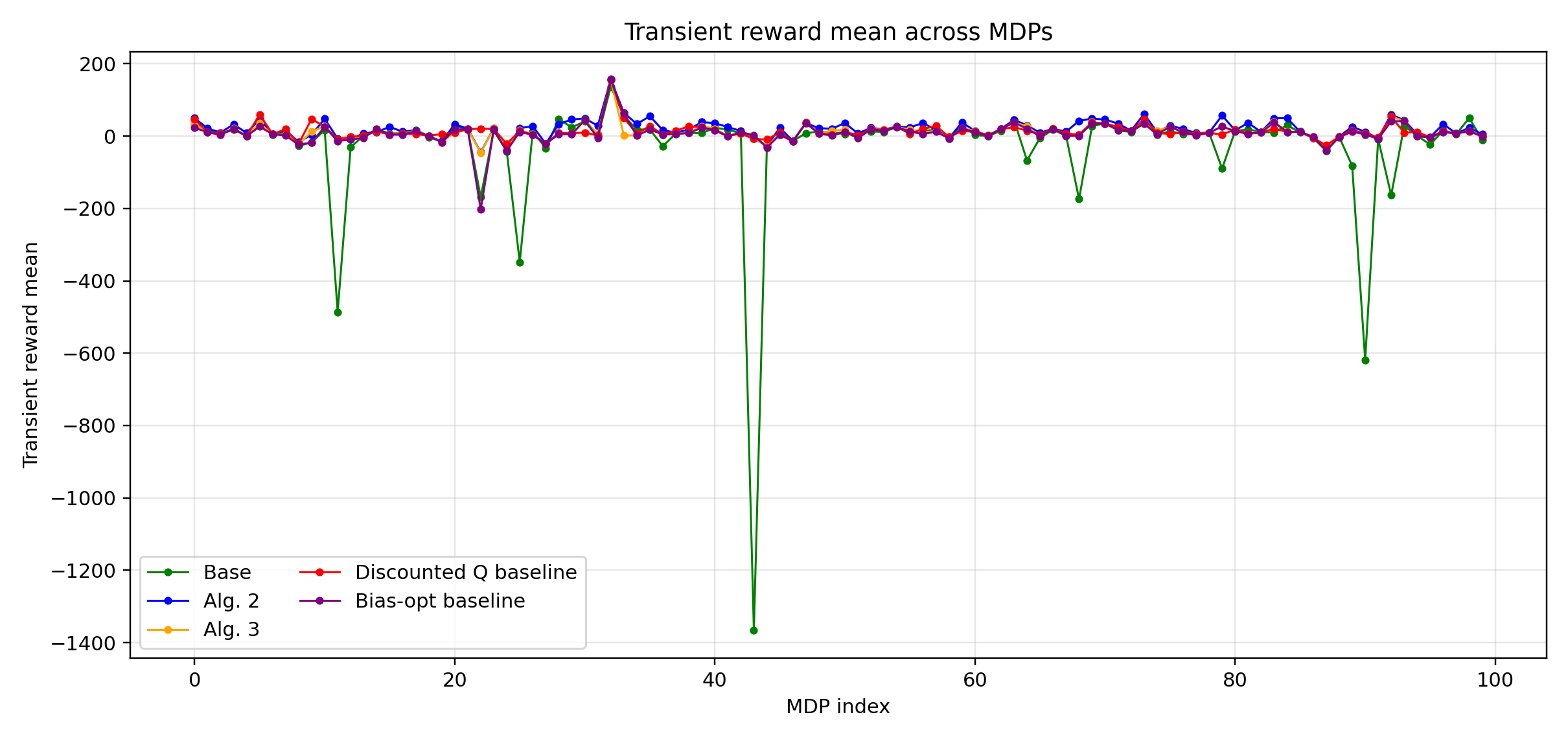}
        \caption{Mean transient reward prior to entering a recurrent class, averaged over
        all initial states, per-MDP. Higher values indicate better transient performance.}
        \label{fig:per_mdp_transient}
    \end{subfigure}
    \hfill
    \begin{subfigure}[t]{0.48\linewidth}
        \centering
        \includegraphics[width=\linewidth]{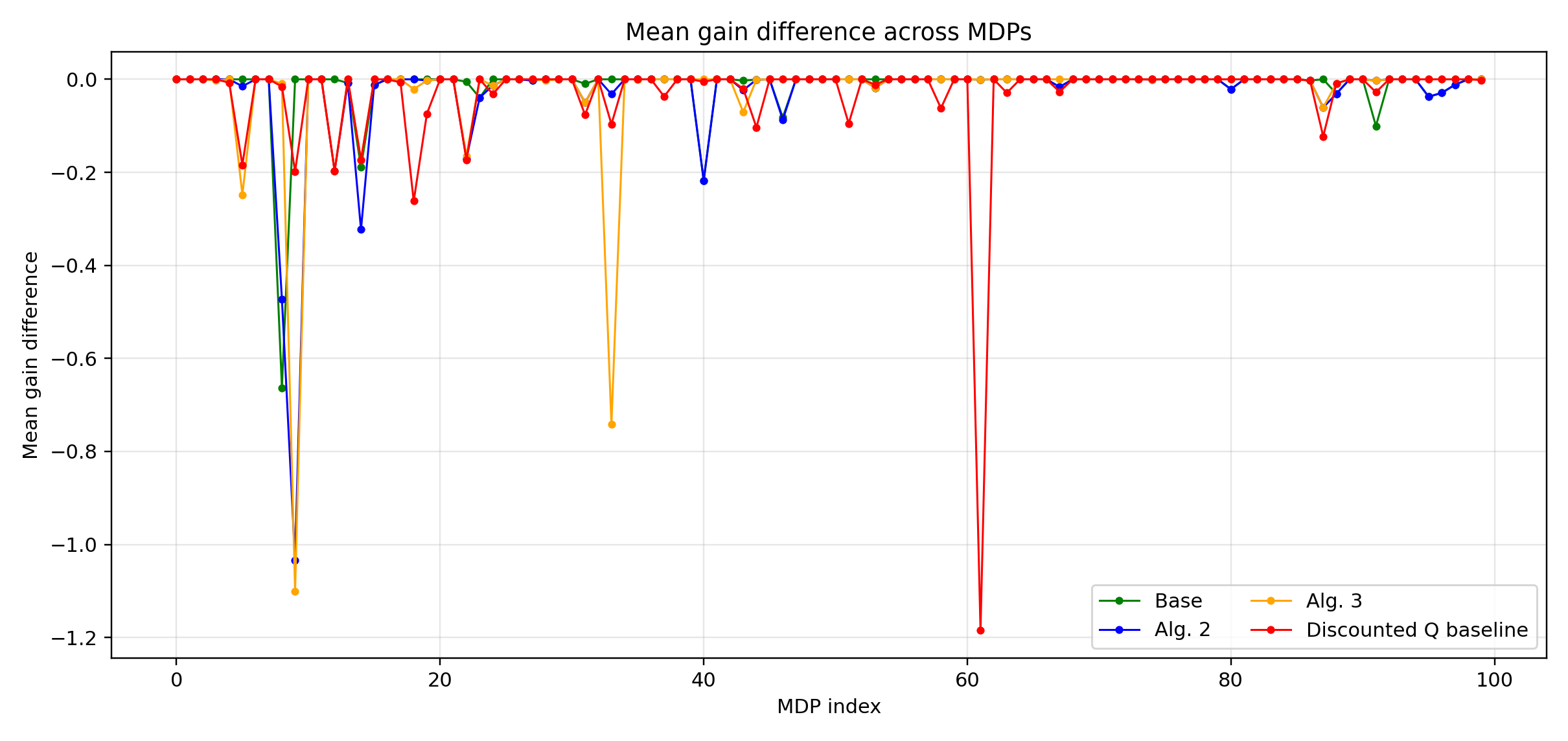}
        \caption{Negative mean gain difference from bias-optimal baseline, averaged over all initial states,
        per-MDP. Higher values indicate better gain performance.}
        \label{fig:per_mdp_gain}
    \end{subfigure}
    \caption{Per-MDP performance across all 100 MDPs.}
    \label{fig:per_mdp}
\end{figure}

Figures~\ref{fig:per_mdp_transient} and~\ref{fig:per_mdp_gain} display per-MDP performance
across all 100 MDPs. Because the aggregated plots in Figure~\ref{fig:main_results}
do not include error bars --- which proved too visually cluttered to be legible --- these
per-MDP plots serve as evidence further supporting the general claims of Section~\ref{sec:main-results}.

Figure~\ref{fig:per_mdp_transient} shows the mean transient reward prior to entering a recurrent class, averaged over all initial states per-MDP. Algorithms~\ref{alg:2} and~\ref{alg:3} both cluster near the bias-optimal benchmark across virtually all 100 MDPs. The base algorithm, by contrast, exhibits large transient failures on a subset of MDPs. The most striking is the base algorithm's
catastrophic transient failure on MDP~43, with mean transient reward of approximately $-1{,}380$. To partially understand why this failure occurs, recall that in Phase~4 of the base algorithm, for any class $C_{\ell,k}$ where Phase~3 decides
to exit via some $(s^*, a^*)$, the within-class policy is set to $\pi^{\mathrm{reach}}_{\ell,k}$--- the solution to an SSP problem that navigates to $s^*$ in minimum expected time, with
no regard for rewards accumulated along the way.

The failures arise in MDPs where the shortest path to $s^*$ within one or more exited classes passes through states with large negative per-step rewards. Since $\pi^{\mathrm{reach}}_{\ell,k}$ minimizes hitting time rather than cumulative reward, it selects these paths without penalty, potentially
accumulating large negative transient costs. When several such classes appear across multiple levels of the hierarchy, these costs compound: the system traverses each exited class via a
reward-blind shortest path before eventually reaching its recurrent destination, and the cumulative transient reward can become extremely poor. In MDP~43, for instance, inspection of the base algorithm's policy reveals that $\pi^{\mathrm{reach}}_{\ell,k}$ routes through several states with per-step rewards in the range $[-8, -5]$, whereas the bias-optimal policy takes alternative paths through the same classes
that achieve per-step rewards as high as $+10$, precisely because it accounts for transient reward through its bias function.

Figure~\ref{fig:per_mdp_gain} shows the mean gain sub-optimality per-MDP. The base algorithm is largely gain-optimal across all MDPs. Extension algorithms \ref{alg:2} and \ref{alg:3} remain near gain-optimal across the majority of MDPs, though both exhibit occasional failures: Algorithm \ref{alg:2} shows more frequent smaller deviations, while Algorithm \ref{alg:3} has fewer, but occasionally larger ones. These results are consistent with Figure~\ref{fig:main_results} in Section~\ref{sec:main-results}. As augmented algorithms are approximation-based, we expect these failures to stem from approximation error. Future work will verify why these large failures occasionally occur and will investigate the sensitivity of the algorithms to parameter choices.

Together, Figures~\ref{fig:per_mdp_transient}  and~\ref{fig:per_mdp_gain} validate the aggregated claims of Section~\ref{sec:main-results}: the base algorithm achieves the strongest gain performance but exhibits large transient failures on a subset of MDPs,
while the extensions substantially improve transient performance while remaining near gain-optimal on the vast majority of instances.

\addcontentsline{toc}{section}{References} 
\bibliographystyle{apa} 
\let\oldbibliography\thebibliography
\renewcommand{\thebibliography}[1]{%
  \oldbibliography{#1}%
  \setlength{\itemsep}{0pt}%
}
{\fontsize{9}{11} \selectfont
\bibliography{rl-multichain.bib}}

\begin{thebibliography}{}

\bibitem[\protect\astroncite{Abounadi et~al.}{2001}]{ABB01}
Abounadi, J., Bertsekas, D.~P., and Borkar, V.~S. (2001).
\newblock Learning algorithms for {Markov} decision processes with average
  cost.
\newblock {\em SIAM J. Control Optim.}, 40(3):681--698.

\bibitem[\protect\astroncite{Bather}{1973a}]{Bather2}
Bather, J. (1973a).
\newblock Optimal decision procedures for finite {Markov} chains. {Part II}:
  Communicating systems.
\newblock {\em Adv. Appl. Prob.}, 5:521--540.

\bibitem[\protect\astroncite{Bather}{1973b}]{Bather3}
Bather, J. (1973b).
\newblock Optimal decision procedures for finite {Markov} chains. {Part III}:
  General convex systems.
\newblock {\em Adv. Appl. Prob.}, 5:541--553.

\bibitem[\protect\astroncite{Bertsekas and Tsitsiklis}{1991}]{BeT91}
Bertsekas, D.~P. and Tsitsiklis, J.~N. (1991).
\newblock An analysis of stochastic shortest path problems.
\newblock {\em Math. Oper. Res.}, 16:185--202.

\bibitem[\protect\astroncite{Bertsekas and Tsitsiklis}{1996}]{BeT96}
Bertsekas, D.~P. and Tsitsiklis, J.~N. (1996).
\newblock {\em Neuro-Dynamic Programming}.
\newblock Athena Scientific, Belmont.

\bibitem[\protect\astroncite{Bertsekas and Yu}{2012}]{BeY12}
Bertsekas, D.~P. and Yu, H. (2012).
\newblock Q-learning and enhanced policy iteration in discounted dynamic
  programming.
\newblock {\em Math. Oper. Res.}, 37:66--94.

\bibitem[\protect\astroncite{Blackwell}{1962}]{Bla62}
Blackwell, D. (1962).
\newblock Discrete dynamic programming.
\newblock {\em Ann. Math. Stat.}, 33:719--726.

\bibitem[\protect\astroncite{Boone and Tuynman}{2025}]{BoT25}
Boone, V. and Tuynman, A. (2025).
\newblock Towards {Blackwell} optimality: {Bellman} optimality is all you can
  get.
\newblock arXiv:2510.13476.

\bibitem[\protect\astroncite{Borkar}{1998}]{Bor98}
Borkar, V.~S. (1998).
\newblock Asynchronous stochastic approximations.
\newblock {\em SIAM J. Control Optim.}, 36(3):840--851.
\newblock Erratum, ibid. 38 (2000), 662--663.

\bibitem[\protect\astroncite{Borkar}{2023}]{Bor23}
Borkar, V.~S. (2023).
\newblock {\em Stochastic Approximations: A Dynamical Systems Viewpoint}.
\newblock Springer and Hindustan Book Agency, Singapore and New Delhi, 2nd
  edition.

\bibitem[\protect\astroncite{Borkar and Meyn}{2000}]{BoM00}
Borkar, V.~S. and Meyn, S. (2000).
\newblock The o.d.e. method for convergence of stochastic approximation and
  reinforcement learning.
\newblock {\em SIAM J. Control Optim.}, 38(2):447--469.

\bibitem[\protect\astroncite{Borkar and Soumyanath}{1997}]{BoS97}
Borkar, V.~S. and Soumyanath, K. (1997).
\newblock A new analog parallel scheme for fixed point computation, {Part I}:
  Theory.
\newblock {\em IEEE Trans. Circuits Systems---I Fund. Theory Appl.},
  44(4):351--355.

\bibitem[\protect\astroncite{Federgruen and Schweitzer}{1981}]{FeS81}
Federgruen, A. and Schweitzer, P.~J. (1981).
\newblock Nonstationary {Markov} decision problems with converging parameters.
\newblock {\em J. Optim. Theory Appl.}, 34:207--241.

\bibitem[\protect\astroncite{Federgruen and Schweitzer}{1984}]{FeS84}
Federgruen, A. and Schweitzer, P.~J. (1984).
\newblock Successive approximation methods for solving nested functional
  equations in {Markov} decision problems.
\newblock {\em Math. Oper. Res.}, 9:319--344.

\bibitem[\protect\astroncite{Fox and Landi}{1968}]{FoxLandi}
Fox, B.~L. and Landi, D.~M. (1968).
\newblock An algorithm for identifying the ergodic subchains and transient
  states of a stochastic matrix.
\newblock {\em Communications of the ACM}, 11(9):619--621.

\bibitem[\protect\astroncite{Fruit et~al.}{2018}]{FPL18}
Fruit, R., Pirotta, M., and Lazaric, A. (2018).
\newblock Near optimal exploration-exploitation in non-communicating {Markov}
  decision processes.
\newblock In {\em Advances in Neural Information Processing Systems}.

\bibitem[\protect\astroncite{Grand-Cl\'{e}ment and Petrik}{2023}]{GrP23}
Grand-Cl\'{e}ment, J. and Petrik, M. (2023).
\newblock Reducing {Blackwell} and average optimality to discounted {MDPs} via
  the {Blackwell} discount factor.
\newblock In {\em Advances in Neural Information Processing Systems}.

\bibitem[\protect\astroncite{Jin et~al.}{2024}]{JBGZ24}
Jin, Y., Blanchet, J., Gummadi, R., and Zhou, Z. (2024).
\newblock Feasible {Q}-learning for average reward reinforcement learning.
\newblock In {\em International Conference on Artificial Intelligence and
  Statistics}.

\bibitem[\protect\astroncite{Kara and Y\"{u}ksel}{2026}]{KaY26}
Kara, A.~D. and Y\"{u}ksel, S. (2026).
\newblock Approximations and learning for continuous state and action mdps
  under average cost criteria.
\newblock {\em J. Mach. Learn. Res.}, 27(68):1--50.

\bibitem[\protect\astroncite{Kushner and Dupuis}{2001}]{KuD01}
Kushner, H.~J. and Dupuis, P.~G. (2001).
\newblock {\em Numerical Methods for Stochastic Control Problems in Continuous
  Time}.
\newblock Springer, New York.

\bibitem[\protect\astroncite{Ohno}{1988}]{Ohno}
Ohno, K. (1988).
\newblock A value iteration method for undiscounted multichain {Markov}
  decision processes.
\newblock {\em Z. Oper. Res}, 32:71--93.

\bibitem[\protect\astroncite{Puterman}{1994}]{Put94}
Puterman, M.~L. (1994).
\newblock {\em Markov Decision Processes: Discrete Stochastic Dynamic
  Programming}.
\newblock John Wiley \& Sons.

\bibitem[\protect\astroncite{Saldi et~al.}{2018}]{SLY18}
Saldi, N., Linder, T., and Y\"{u}ksel, S. (2018).
\newblock {\em Finite Approximations in Discrete-Time Stochastic Control:
  Quantized Models and Asymptotic Optimality}.
\newblock Springer Nature, Cham.

\bibitem[\protect\astroncite{Schweitzer}{1984}]{Schweitzer1984}
Schweitzer, P.~J. (1984).
\newblock A value-iteration scheme for undiscounted multichain {Markov} renewal
  programs.
\newblock {\em Z. Oper. Res.}, 28(4):143--152.

\bibitem[\protect\astroncite{Schweitzer and Federgruen}{1977}]{ScF77}
Schweitzer, P.~J. and Federgruen, A. (1977).
\newblock The asymptotic behavior of undiscounted value iteration in {Markov}
  decision problems.
\newblock {\em Math. Oper. Res.}, 2(4):360--381.

\bibitem[\protect\astroncite{Schweitzer and Federgruen}{1978}]{ScF78}
Schweitzer, P.~J. and Federgruen, A. (1978).
\newblock The functional equations of undiscounted {Markov} renewal
  programming.
\newblock {\em Math. Oper. Res.}, 3(4):308--321.

\bibitem[\protect\astroncite{Sutton and Barto}{2018}]{SuB18}
Sutton, R. and Barto, A. (2018).
\newblock {\em Reinforcement Learning: An Introduction}.
\newblock MIT Press, Cambridge, MA, 2nd edition.

\bibitem[\protect\astroncite{Sutton et~al.}{1999}]{SPS99}
Sutton, R.~S., Precup, D., and Singh, S. (1999).
\newblock Between {MDPs} and {semi-MDPs}: A framework for temporal abstraction
  in reinforcement learning.
\newblock {\em Artificial Intelligence}, 112:181--211.

\bibitem[\protect\astroncite{Tsitsiklis}{1994}]{Tsi94}
Tsitsiklis, J. (1994).
\newblock Asynchronous stochastic approximation and {Q-learning}.
\newblock {\em Mach. Learn.}, 16:185--202.

\bibitem[\protect\astroncite{Veinott}{1966}]{Vei66}
Veinott, A.~F. (1966).
\newblock On finding optimal policies in discrete dynamic programming with no
  discounting.
\newblock {\em Ann. Math. Stat.}, 37:1284--1294.

\bibitem[\protect\astroncite{Wan et~al.}{2021a}]{WNS21b}
Wan, Y., Naik, A., and Sutton, R.~S. (2021a).
\newblock Average-reward learning and planning with options.
\newblock In {\em Advances in Neural Information Processing Systems}.

\bibitem[\protect\astroncite{Wan et~al.}{2021b}]{WNS21a}
Wan, Y., Naik, A., and Sutton, R.~S. (2021b).
\newblock Learning and planning in average-reward {Markov} decision processes.
\newblock In {\em International Conference on Machine Learning}.

\bibitem[\protect\astroncite{Wan et~al.}{2024}]{WYS24}
Wan, Y., Yu, H., and Sutton, R.~S. (2024).
\newblock On convergence of average-reward {Q-learning} in weakly communicating
  {Markov} decision processes.
\newblock arXiv:2408.16262.

\bibitem[\protect\astroncite{Watkins}{1989}]{WatkinsQ}
Watkins, C. J. C.~H. (1989).
\newblock {\em Learning from Delayed Rewards}.
\newblock PhD thesis, Cambridge University.

\bibitem[\protect\astroncite{White}{1963}]{Whi63}
White, D.~J. (1963).
\newblock Dynamic programming, {Markov} chains, and the method of successive
  approximations.
\newblock {\em J. Math. Anal. Appl.}, 6(3):373--376.

\bibitem[\protect\astroncite{Yu and Bertsekas}{2013a}]{YuB13}
Yu, H. and Bertsekas, D.~P. (2013a).
\newblock On boundedness of {Q-learning} iterates for stochastic shortest path
  problems.
\newblock {\em Math. Oper. Res.}, 38:209--227.

\bibitem[\protect\astroncite{Yu and Bertsekas}{2013b}]{YuB12}
Yu, H. and Bertsekas, D.~P. (2013b).
\newblock Q-learning and policy iteration algorithms for stochastic shortest
  path problems.
\newblock {\em Ann. Oper. Res.}, 208:95--132.

\bibitem[\protect\astroncite{Yu et~al.}{2025}]{YWS25b}
Yu, H., Wan, Y., and Sutton, R.~S. (2025).
\newblock Average-reward reinforcement learning in semi-{Markov} decision
  processes via relative value iteration.
\newblock arXiv:2512.06218.

\bibitem[\protect\astroncite{Yu et~al.}{2026}]{YWS25a}
Yu, H., Wan, Y., and Sutton, R.~S. (2026).
\newblock Asynchronous stochastic approximation with applications to
  average-reward reinforcement learning.
\newblock {\em SIAM J. Control Optim.}, 64:1456--1481.
\newblock (Extended arXiv version: arXiv:2409.03915).

\bibitem[\protect\astroncite{Zurek and Chen}{2024}]{ZuC24}
Zurek, M. and Chen, Y. (2024).
\newblock Span-based optimal sample complexity for weakly communicating and
  general average reward {MDPs}.
\newblock In {\em Advances in Neural Information Processing Systems}.

\bibitem[\protect\astroncite{Zurek and Chen}{2025}]{ZuC25}
Zurek, M. and Chen, Y. (2025).
\newblock The plug-in approach for average-reward and discounted {MDPs}:
  Optimal sample complexity analysis.
\newblock In {\em International Conference on Algorithmic Learning Theory}.
\newblock arXiv:2410.07616.

\end{thebibliography}

\end{document}